\documentclass[11pt,letterpaper]{article}

\usepackage[margin=1in]{geometry}

\usepackage{setspace}
\usepackage{parskip}
\usepackage{microtype}

\usepackage{natbib}

\usepackage{algorithm, algorithmicx, algpseudocode}
\usepackage{amssymb}
\usepackage{amsmath}
\usepackage{graphicx}
\usepackage{mathtools}
\usepackage{needspace} 
\usepackage{multicol}
\usepackage{listings}
\usepackage{amsthm}
\usepackage[shortlabels]{enumitem}
\usepackage[italicdiff]{physics}
\usepackage{cancel}
\usepackage[dvipsnames]{xcolor}
\usepackage{verbatim}
\usepackage{comment}
\usepackage{array}
\usepackage{multirow}
\usepackage{tikz-cd}
\usepackage{import}
\usepackage{booktabs}
\usepackage{comment}
\usepackage{tabularx}
\usepackage[colorlinks, linktoc=page, linkcolor=MidnightBlue, citecolor=blue]{hyperref}
\usepackage{forest}
\usepackage{float}
\usepackage{xspace}
\usepackage{aliascnt}
\usepackage{wrapfig}
\usepackage{thm-restate}
\usepackage{adjustbox}
\usepackage{bm}
\usepackage{bbm}
\usepackage{complexity}
\usepackage{makecell}
\usepackage{nicefrac}
\usepackage{mathrsfs}
\usepackage[most]{tcolorbox}
\usetikzlibrary{calc, positioning, external, arrows.meta}

\usepackage{pgfplots}
\usepgfplotslibrary{groupplots}
\pgfplotsset{compat=1.18}

\allowdisplaybreaks

\let\E\relax

\DeclareMathOperator*{\argmax}{argmax}
\DeclareMathOperator*{\E}{\mathbb E}

\let\cite\citep

\newfloat{protocol}{tbp}{lop}
\floatname{protocol}{Protocol}

\newtheorem{theorem}{Theorem}
\numberwithin{theorem}{section}
\newtheorem*{theorem*}{Theorem}

\newtheorem{lemma}[theorem]{Lemma}
\newtheorem{proposition}[theorem]{Proposition}
\newtheorem*{proposition*}{Proposition}

\theoremstyle{definition}
\newtheorem{definition}[theorem]{Definition}
\newtheorem{notation}[theorem]{Notation}

\newcommand{\ie}{{\em i.e.}\xspace}

\newcommand{\Prob}{\mathbb P}

\newcommand{\indi}{\perp \!\!\! \perp}
\newcommand{\As}{\mathcal A}
\newcommand{\Bs}{\mathcal B}
\newcommand{\Cs}{\mathcal C}

\newcommand{\Fs}{\mathcal F}
\newcommand{\Gs}{\mathcal G}

\newcommand{\Os}{\mathcal O}

\newcommand{\Rs}{\mathcal R}

\newcommand{\Vs}{\mathcal V}

\newcommand{\Vbuy}{V^{\text{buy}}}
\newcommand{\Vsell}{V^{\text{sell}}}
\newcommand{\goodevent}{\mathcal{E}_\text{good}}

\usepackage{caption}
\usepackage{tcolorbox}
\usepackage{enumitem}
\definecolor{lightblue}{rgb}{0.8, 0.9, 1}
\definecolor{lightred}{rgb}{1, 0.8, 0.8}
\definecolor{lightgreen}{rgb}{0.8, 1, 0.8}

\tcbset{
	myblockstyle/.style={
		boxrule=0.4pt,
		colback=#1!10,
		colframe=#1!80!black,
		boxsep=2pt,
		left=-2.5pt,
		right=2pt,
		top=0pt,
		bottom=0pt,
		before skip=2pt,
		after skip=2pt,
	}
}

\newtcbox{\tcbinlinebox}[1][]{nobeforeafter,
	on line,
	colback=#1!10,
	colframe=#1!80!black,
	boxrule=0.8pt,
	arc=2pt,
	boxsep=0pt,
	left=2pt, right=2pt, top=2pt, bottom=2pt,
	valign=center
}

\usepackage{multirow}
\usepackage{colortbl}
\usepackage{makecell}

\usepackage{pifont}

\definecolor{rqBlueBack}{HTML}{EAF2FF} 
\definecolor{rqBlueFrame}{HTML}{3B82F6} 

\makeatletter
\renewcommand\thanks[1]{%
  \footnotemark%
  \protected@xdef\@thanks{\@thanks
    \protect\footnotetext[\the\c@footnote]{#1}}%
}
\makeatother

\title{No-Regret Revenue Maximization in Sequential Market Making\\ Under Inventory Constraints}
\author{
Davide Maran\thanks{The authors are listed in alphabetical order.} \\
    \texttt{davide.maran@polimi.it} \\
    Politecnico di Milano 
\and
    Marcello Restelli\footnotemark[1] \\
    \texttt{marcello.restelli@polimi.it} \\
    Politecnico di Milano 
\and
    Francesco Emanuele Stradi\footnotemark[1] \\
    \texttt{francescoemanuele.stradi@polimi.it} \\
    Politecnico di Milano 
}
\date{} %

\begin{document}

\maketitle

\begin{abstract}
Market makers are fundamental actors in today’s economy, as they provide \emph{liquidity} by accelerating the matching between demand and supply. Accordingly, \emph{sequential market making} models the problem of an agent who repeatedly interacts with a population of buyers and sellers. In each round, the market maker simultaneously posts a bid price and an ask price to a buyer-seller pair. A trade occurs whenever the posted prices meet the agents’ valuations: the buyer buys if their valuation is at least the ask price, and the seller sells if their valuation is at most the bid price. The two trades occur independently, so the market maker may sell without buying (short selling), buy without selling, or observe no trade at all. Crucially, the market maker must satisfy \emph{inventory constraints}: over the entire horizon, the total number of bought and sold items should balance out, while during the process the inventory must remain sufficiently small, thus avoiding excessive accumulation of items and risky short-selling exposure.

We study this problem through the lens of \emph{regret minimization}, when buyers’ and sellers’ valuations are sampled from unknown distributions. In this setting, the market maker (that is, the learner) sequentially posts bids and asks so as to maximize its cumulative revenue over $T \in \mathbb{N}_{> 0}$ interactions while satisfying the inventory constraints. Specifically, the goal is to minimize the difference between the learner’s performance and that of the optimal inventory-aware bid-ask pair that knows the valuation distributions (\ie, the \emph{cumulative regret}).
We first study the hardest scenario, where the learner only observes whether each trade (buy and/or sell) occurs. In this case, we provide a computationally efficient algorithm attaining $\widetilde{\mathcal{O}}(T^{{\nicefrac{2}{3}}})$ regret, thus matching the well-known lower bound for the easier dynamic pricing setting of order $\Omega(T^{\nicefrac{2}{3}})$~\citep{10.5555/946243.946352}, while satisfying the global inventory constraints and keeping the inventory at every round below $\widetilde{\mathcal{O}}(T^{\nicefrac{2}{3}})$ items.
We then study the more interesting, and arguably more realistic, setting where the market maker has access to a \emph{Limit Order Book} (LOB). Accordingly, the learner observes the realized valuation whenever the associated trade does not occur; otherwise, it only observes that the trade has occurred. In this case, we provide a computationally efficient algorithm that attains $\widetilde{\mathcal{O}}(\sqrt{T})$ regret, satisfies the global inventory constraints, and keeps the inventory below $\widetilde{\mathcal{O}}(\sqrt{T})$ items at every round. We complement these results by providing instance-dependent bounds that improve according to the c.d.f. value associated with the discretized optimum.
We conclude by establishing a \emph{lower bound} showing that any algorithm that guarantees per-round inventory of at most $\mathcal{O}(T^{\beta})$ must incur regret of at least $\Omega(T^{1-\beta})$ for every $\beta < \nicefrac{1}{2}$, even in the significantly easier setting where the agent knows the valuation distributions.
\end{abstract}

\newpage
\tableofcontents
\newpage

\section{Introduction}
Market making is a foundational part of modern financial markets~\citep{GLOSTEN198571,avellaneda2008high}. At a basic level, market makers sustain trading by continuously offering to buy and sell, thereby providing the \emph{immediacy} and \emph{liquidity} that make markets usable in the first place. In the absence of such intermediation, even modest imbalances between buyers and sellers could translate into large delays, wide spreads, and severe price dislocations. This role is not minor and in today’s electronic markets, liquidity provision is carried out at a massive scale by highly sophisticated firms, whose business is precisely to keep markets tradable under a broad range of conditions. Market making is therefore not merely a financial activity, but a central problem in \emph{market design}, \emph{algorithms}, and \emph{online decision-making}~\citep{milionis2024myersonian,cesa2025market}: a market maker must repeatedly quote prices in an uncertain environment, balancing the short-term opportunity to earn spread revenue with the longer-term need to remain resilient as order flow evolves. What makes this problem especially subtle is that liquidity provision is inseparable from risk-taking~\citep{gueant2013dealing}. A market maker does not simply post prices; it absorbs temporary supply--demand imbalances by taking them on directly. For this reason, \emph{inventory} is not a secondary accounting quantity, but one of the key forces governing the entire process: without controlling inventory, apparent revenue gains may reflect unsustainable exposure rather than economic performance.

More broadly, market making is closely connected to other well-studied problems in \emph{revenue maximization}, such as \emph{dynamic pricing}~\citep{10.5555/946243.946352} and \emph{bilateral trade}~\citep{castiglioni2026sample}. As in dynamic pricing, the seller repeatedly sets prices in the presence of uncertain and evolving demand, with the goal of learning how prices affect trade and revenue over time. At the same time, market making also has a two-sided nature that brings it closer to bilateral trade settings, since the intermediary interacts with both buyers and sellers and must jointly manage the terms at which it is willing to trade on each side of the market. There is, however, an important difference. In standard bilateral trade models, trade is often viewed as a more localized interaction between the two sides. In market making, by contrast, trades may occur in only one direction for extended periods, as long as the intermediary is able to absorb the resulting imbalance. This makes \emph{inventory} central to the problem: the market maker can keep trading on one side only as long as its inventory remains within acceptable bounds. For this reason, market making is not simply a two-sided pricing problem, but a sequential revenue maximization problem in which pricing, trade, and inventory are tightly linked across time. Thus, market making shares important features with classical revenue maximization problems, but differs in that current pricing decisions directly shape future feasible actions through their effect on inventory.

These connections also point naturally to the framework of \emph{regret minimization} and \emph{online learning}~\citep{cesa2006prediction, Orabona, slivkins2019introduction}. In market making, as in other sequential revenue maximization problems, the distributions governing buyers and sellers are typically unknown to the maker. The market maker must therefore learn how prices affect trade while continuing to make sequential decisions, without knowing in advance which pricing policy maximizes revenue. Regret minimization captures this objective by comparing the cumulative revenue of the algorithm to that of the best benchmark policy computed knowing the valuation distributions of both buyers and sellers (we refer to this difference as the \emph{cumulative regret}). This viewpoint is standard in related problems such as \emph{dynamic pricing} and \emph{bilateral trade}, where the central challenge is to balance exploration and revenue over time (or in welfare-maximization settings, to balance exploration and social welfare)~\citep{besbes2009dynamic, babaioff15pricing, bachoc2024fair, 11369105}. It is thus a natural framework for studying market making in uncertain environments. 

\subsection{The Model}

We model market making as a discrete-time online learning problem between an agent (the ``market maker'', or simply the ``maker'') and a sequence of ``takers''. Specifically, at each round $t\in[T]$,\footnote{In this work, for any $n\in\mathbb{N}_{>0}$, we write $[n]$ to denote the set $\{1,\dots,n\}$.} a \emph{new} buyer--seller pair arrives, each interested in trading one unit of the underlying asset. Their private valuations are denoted by $\Vbuy_t,\Vsell_t\in[0,1]$, where $\Vsell_t$ is the minimum amount the seller is willing to accept to sell one unit, and $\Vbuy_t$ is the maximum amount the buyer is willing to pay to buy one unit. We assume that the valuations are sampled i.i.d. across rounds from an unknown distribution; however, within any given round, the buyer's and seller's valuations may be \emph{arbitrarily correlated}.
Consequently, the maker selects a bid--ask pair $(B_t,A_t)$ such that $B_t<A_t$ without any knowledge of $\Vbuy_t,\Vsell_t$. On the buy side, if $\Vsell_t\leq B_t$, the seller accepts the bid and the maker buys at price $B_t$, obtaining a reward $-B_t$; otherwise, no trade occurs, and the reward is $0$. Symmetrically, on the sell side, if $\Vbuy_t\geq A_t$, the buyer accepts the ask and the maker sells at price $A_t$, obtaining reward $A_t$; otherwise, no trade occurs, and the reward is $0$. Hence, the maker's per-round reward is:
\[
A_t\cdot\mathbbm{I}(\Vbuy_t\geq A_t)-B_t\cdot\mathbbm{I}(\Vsell_t\leq B_t).
\]
A specific feature of market making is that the maker can sell an asset before buying it, thereby taking a \emph{short} position. Nevertheless, market makers typically aim to keep their \emph{inventory} small, so as to avoid a large imbalance between the number of units bought and sold over time. Indeed, a large positive inventory exposes the maker to downward price movements, while a large negative inventory exposes it to upward ones. In both cases, profits become increasingly dependent on the market value of the asset, rather than solely on the bid--ask spread.
To capture this aspect, we impose two inventory constraints. First, we require \emph{global inventory balance} over the entire horizon, namely,
$\sum_{t=1}^T \mathbbm{I}(\Vbuy_t\ge A_t)=\sum_{t=1}^T \mathbbm{I}(\Vsell_t\le B_t)$.
This ensures that the maker ends with zero inventory, while still allowing temporary short or long positions during learning. Second, we require the \emph{inventory} level to remain \emph{uniformly bounded} at every round by a parameter $\alpha\in\mathbb{N}_{>0}$, that is,
$| I_t |\leq \alpha$, for all $t\in[T]$, where $I_t\coloneqq
\sum_{\tau=1}^t\mathbbm{I}(\Vsell_\tau\le B_\tau)
-
\sum_{\tau=1}^t\mathbbm{I}(\Vbuy_\tau\ge A_\tau).$
This second constraint limits the maker's exposure to the market throughout the learning process. To summarize, the goal of the agent is not only to maximize the rewards attained during learning, but also to simultaneously satisfy the aforementioned constraints. This makes the problem particularly challenging, since the actions taken in the early rounds directly affect both the feasibility and the profitability of the decisions taken later on.

We consider two feedback models, depending on the information revealed to the market maker at the end of each round. 
\begin{enumerate}
    \item Under \emph{two-bit feedback}, the maker only observes whether a trade occurred on the buy side and on the sell side, namely $\mathbbm{I}(\Vbuy_t\geq A_t)$ and $\mathbbm{I}(\Vsell_t\leq B_t)$.

    \item Differently, when access to the \emph{limit order book} (LOB) is available, the maker receives additional information whenever no trade occurs. More precisely, if $B_t < \Vsell_t$, then the maker observes $\Vsell_t$, and if $A_t > \Vbuy_t$, then it observes $\Vbuy_t$. Thus, information is revealed exactly on those sides where no reward (loss) is obtained. This captures a natural feature of market making: when a trade occurs, takers have no incentive to reveal their true valuation, whereas when no trade occurs, they may expose it by posting an order in the book. In this sense, LOB feedback can be viewed as a realistic and intermediate regime between the two-bit model and full-feedback, where the entire valuations are observed at the end of each round.
\end{enumerate}
As is standard in online learning, we compare the performance of the learner against an offline benchmark. In our setting, this benchmark is the best fixed bid--ask pair $(b,a)$ that knows the underlying valuation distribution in advance and satisfies the global inventory-balance constraint in expectation. We denote by $\mathrm{OPT}$ the expected revenue achieved by such a benchmark over the whole horizon.
Given this benchmark, we define the regret of a learning algorithm as:
\[R_T\coloneqq \mathrm{OPT}-\sum_{t=1}^T \left[A_t\cdot \mathbbm{I}(\Vbuy_t\geq A_t) - B_t \cdot \mathbbm{I}(\Vsell_t\leq B_t)\right].\]
Informally, the goal of the paper is the following.
\begin{tcolorbox}[
    enhanced,
    breakable,
    sharp corners,
    boxrule=0pt,
    frame hidden,
    borderline west={2pt}{0pt}{gray!80}, 
    colback=gray!5
]
\noindent \textbf{Research Question:} Is it possible to design efficient learning algorithms for the sequential market-making problem that achieve optimal regret bounds while satisfying the inventory constraints throughout the learning process?
\end{tcolorbox}
We answer this question affirmatively, showing that the problem faced by market makers is learnable. 

\subsection{Related Work}
\label{sec:related}
In this section, we discuss the work most closely related to ours.
\paragraph{Market making} Market making has also been studied from a mechanism-design perspective (\emph{e.g.,} see~\citep{milionis2024myersonian}), where the mechanism must provide an optimal allocation and an incentive-compatible payment rule, assuming access to both buyers' and sellers' valuation distributions. Our work is mainly related to sequential market making~\cite{NIPS2013_995e1fda, Xue_Du_Xu_2025}, which studies the performance of a regret-minimization algorithm against a \emph{finite} class of constant-spread dynamic strategies. Our results are not directly comparable to theirs, since we operate in a continuous bid--ask decision space and do not assume full feedback. The work most closely related to ours is~\citep{cesa2025market}. In this work, the authors study an online market-making setting with a \emph{single} per-round taker in which the reward attained by the market maker (the agent) is given by the difference between the bid (or ask) and the unknown market value of the item. They provide optimal regret bounds under various assumptions on the nature of the environment, assuming that the market maker observes the market value and a one-bit feedback signal on the trade at the end of each round.
While~\citep{cesa2025market} provides a foundational and technically very strong regret-minimization treatment of market making, our setting captures two features that are central in practice and that substantially change the nature of the problem. First, unlike~\citep{cesa2025market}, which assumes the existence of a market value relating the rewards associated with sold items to those associated with purchased items, we explicitly model inventory constraints. As a result, in our setting, actions taken in early rounds affect the feasibility and profitability of the decisions in later rounds, posing additional and unique technical challenges. 
Second, we study different feedback models. We consider two-bit feedback---which is the natural minimal model in our setting, since buyers and sellers are treated separately, as in real-world market making---and we also study a more realistic partial-information model motivated by the \emph{limit order book}.
In this sense, we believe that our framework is arguably closer to the informational and operational constraints faced by actual market makers. 
The \emph{limit order book} feedback was originally introduced in~\citep{pmlr-v336-maran26a} and then further studied by~\cite{vischi2026optimalregretonlinemarket}, who consider the market-value-based setting of~\citep{cesa2025market} while allowing full feedback when no trade occurs. Despite the shared motivation for the feedback model, the underlying setting differs substantially from ours since our explicit inventory constraints make the feasibility and profitability of future decisions depend on past trades.

\paragraph{Online learning in economic settings and bilateral trade} Regret minimization and online learning methods have become a standard toolkit for studying digital markets and economic design under uncertainty, replacing strong prior assumptions with robust online guarantees. These techniques have been applied to one-sided and dynamic pricing~\citep{10.5555/946243.946352, babaioff15pricing, feldman2016online}, repeated auctions~\citep{Cesa-BianchiGM15, Cesa-BianchiCRFL24, morgenstern2015pseudo, weed2016online, LykourisST16, DaskalakisS22, nedelec2022learning, balseiro2019learning}, 
contract design~\citep{HoSV16, DuettinGSW23,ZhuBYWYJ22}, brokerage~\citep{bolic2023online}, and Bayesian persuasion~\citep{castiglioni2020online,castiglioni2023regret,zu2021learning}.
A particularly relevant related line of work is repeated bilateral trade, where a learner acts as a broker and posts prices to agents with unknown private valuations, with the objective of maximizing the social welfare. Foundational results established optimal regret guarantees under i.i.d. assumptions, smoothness conditions, and budget-balance constraints~\citep{CesaBianchiCCF21, cesa23mor}. Subsequent work extended these guarantees to adversarial environments, that is, beyond i.i.d. assumptions, weaker budget-balance notions, and more general dependence structures between buyer and seller valuations~\citep{AzarFF22, CesaBianchiCCFL23, longVersion}. Recently,~\citep{11369105} studied the problem of revenue maximization in bilateral trade. 
Our work naturally fits within the online learning in the economic category; nonetheless, we remark the following key distinctive features between our setting and the aforementioned ones. On the one hand, our setting involves a two-sided interaction with both a buyer and a seller, which distinguishes it from standard dynamic pricing or auction settings. On the other hand, in market making, trades may occur in only one direction for extended periods, until some inventory constraints are satisfied, while in bilateral trade, a trade may happen in both directions only.

\paragraph{Online learning with general and budget constraints} A strongly related research direction is online learning under unknown constraints, where the objective is to design online learning algorithms that attain sublinear regret together with sublinear cumulative violation of some unknown constraints~\citep{Mannor, mahdavi2013stochastic, jenatton2016adaptive, yu2017online}. 
A particularly influential framework in this area is that of bandits with knapsacks (BwK), introduced 
by~\citep{badanidiyuru2018bandits}. In BwK, the learner observes bandit feedback and faces stochastic rewards 
and resource consumptions, with the goal of maximizing cumulative utility while keeping the total 
consumption of each of the $m$ resources within a budget $B$. For the stochastic setting, improved and 
optimal algorithms were later developed by~\citep{agrawal2019bandits}. The adversarial BwK version, that is, allowing the rewards and consumption to change arbitrarily during the learning dynamic, was first investigated by~\citep{immorlica2022adversarial}, who adopt as benchmark the 
best fixed distribution over arms in hindsight. Under this benchmark, standard no-regret guarantees are 
in general unattainable, and one must instead resort to weaker notions such as $\alpha$-regret. 
If $\rho$ denotes the learner's per-round budget, the best currently known competitive-ratio guarantees are 
$1/\rho$ when $B=\Omega(T)$ \citep{castiglioni2022online}, and $\mathcal{O}(\log m \log T)$ in the general case~\citep{kesselheim2020online}. More recently,~\citep{kumar2022non, bernasconi2024bandits} generalized the BwK model 
to allow non-monotone resource dynamics, so that resources may also be replenished over time. 
Related in spirit, but closer to reinforcement learning~\citep{sutton2018reinforcement} formulations, is the literature on online learning 
in constrained Markov decision processes (CMDPs). In CMDPs~\citep{Altman1999ConstrainedMD}, the learner must optimize cumulative reward 
while enforcing trajectory-level constraints in an environment with unknown state dynamics. 
A foundational contribution in this direction is due to~\citep{Exploration_Exploitation}, who studied the 
exploration--exploitation trade-off in unknown CMDPs and provided sublinear guarantees for both regret and 
constraint violation. More recent works extended the model to settings with adversarial losses \citep{Upper_Confidence_Primal_Dual} and adversarial constraints~\citep{stradi24a}, while others 
investigated stronger safety notions that do not allow past constraint satisfaction to compensate for 
future violations \citep{muller2024truly,stradi2024primal}.
Our setting shares similarities with the online learning with general constraints (or even online CMDPs) framework and the replenishable budget ones. Nonetheless, the aforementioned techniques cannot be transferred directly 
to our problem, due to the presence of a continuous and structurally richer action space, as well as a much 
weaker feedback model than the one usually assumed in bandit settings. Additionally, we remark that our setting does not present any Slater-like condition for the offline optimum.

\subsection{Technical Challenges}

In this section, we discuss the main technical challenges of our setting. In particular, online market making under inventory constraints presents several unique difficulties that make the learning problem substantially challenging.

\paragraph{On the action space and the reward function} The action space of bid-ask pairs available to the agent during learning is continuous and two-dimensional, namely
$
\{(b,a)\in[0,1]^2 : b<a\}.
$
This already rules out standard multi-armed bandit techniques~\citep{lattimore2020bandit}, which are inherently tailored to discrete action spaces. At the same time, existing methods for continuum-armed bandits, such as the Zooming algorithm~\citep{kleinberg2008multi, kleinberg2019bandits}, do not apply directly to our setting for several reasons. First, in two-dimensional problems, these methods typically yield regret bounds of order $T^{3/4}$, which we show to be suboptimal in our setting. Second, our reward function is neither Lipschitz nor smooth, since we impose no regularity assumptions on the c.d.f.'s of buyers and sellers. As a result, standard discretization arguments or smoothing-based reductions cannot be applied in any straightforward way to obtain sublinear regret. We acknowledge that related, but substantially different, challenges also arise in no-regret algorithms for bilateral trade~\citep{Cesa_Bianchi_2021}. In such settings, however, these challenges are typically overcome either by assuming regularity of the buyers’ and sellers’ c.d.f.'s~\citep{Cesa_Bianchi_2021,CesaBianchiCCFL23}, or by leveraging the special structure of the bilateral trade objective~\citep{longVersion}. Intuitively, that structure makes discretization considerably less challenging since trades can occur only when both sides simultaneously accept the proposed prices.

\paragraph{On the inventory constraints} The inventory constraints play a crucial role in our setting. Indeed, they clearly relate the bid and ask sides of the problem, preventing us from optimizing them independently. More precisely, the bid-ask pair selected in earlier rounds directly affects the inventory level carried into later ones, and therefore affects both the feasibility and the profitability of future bid-ask pairs.
More generally, while the literature on online learning under unknown budget and general long-term constraints is huge~\citep{badanidiyuru2018bandits,Unifying_Framework}, our inventory constraints have a structure that does not allow for any direct application of existing techniques. First of all, the overall structure of the problem (\emph{e.g.,} action and rewards space) is already substantially more involved than in standard constrained online learning models, thus preventing simple adaptation of existing techniques from the literature. Second, our constraints do not fit the canonical templates studied in the literature. On the one hand, the \emph{overall} inventory requirement is closely related to replenishable-budget models~\citep{kumar2022non,bernasconi2024bandits}, but with some additional challenges: indeed, $i)$ \emph{both} the initial budget and the \emph{final} one must be $0$ and $ii)$ more importantly, \emph{no Slater's condition is available}---that is, there is not a bid-ask pair which \emph{strictly} satisfies the offline expected constraint, since it is an equality constraint, leaving no slack around the optimum. On the other hand, the \emph{per-round} inventory constraints can be
seen as \emph{any-time} general constraint~\cite{Unifying_Framework}, where $\alpha$ is the violation rate and Slater's condition is \emph{not available}.
Overall, these features make our setting significantly different from the ones considered in the existing literature, and show that handling the inventory constraints in our problem requires substantially different ideas.

\paragraph{On the benchmark}
It is standard in the literature on online learning under unknown constraints to compare the performance of the algorithm against the expected reward attained by the optimal action, or strategy of actions, satisfying the constraints in expectation---which we denote by $\text{OPT}$. This is the benchmark typically considered in the existing literature, and the one for which most of the available techniques are designed~\citep{Exploration_Exploitation, liu2021learning, castiglioni2022online, bernasconi2024bandits, Unifying_Framework, stradi2024learning}. In our setting, in addition to the regret against this standard benchmark, we also study the regret $R_T^J$ with respect to a stronger one, namely the benchmark maximizing the objective:
\[
J^* \coloneqq \sup_{b,a\in[0,1]^2} J(b,a)\coloneq\sup_{b,a\in[0,1]^2} (a-b) \min\left\{\E\left[\mathbbm{I}\left(\Vsell_t\leq b\right)\right],\E\left[\mathbbm{I}\left(\Vbuy_t\geq a\right)\right]\right\}.
\]
Proving an upper bound for the following notion of \emph{surrogate regret}, namely, $R_T^J\coloneqq T\cdot J^*-\sum_{t=1}^T [A_t\cdot \mathbbm{I}(\Vbuy_t\geq A_t) - B_t \cdot \mathbbm{I}(\Vsell_t\leq B_t)]$,
leads to a substantially stronger guarantee for two reasons. First, it is easy to show that $T\cdot J^* \ge \text{OPT}$, and thus our benchmark is indeed stronger than the standard one. Second, $ J^*$ may be greater than $0$ even when the program defining $\text{OPT}$ does not admit solutions with strictly positive utility, that is, when the optimal safe solution\footnote{We refer as a safe solution to a bid-ask pair which satisfies the expected inventory constraint.} is simply to never buy or sell. Accordingly, the guarantees of our algorithm remain non-trivial even when no safe solution with strictly positive utility exists. 
To the best of our knowledge, this is a novel feature in the constrained online learning literature and, more importantly, it requires substantially different ideas, since one cannot simply aim to approach $\text{OPT}$ directly, thus guaranteeing both safety and optimality. Instead, it is necessary to satisfy the constraints while learning through a surrogate objective based on $J$, which may induce optimal bid-ask pairs that are not safe.

\paragraph{On the feedback}  
The feedback structure is another major source of difficulty in our problem. Under two-bit feedback, after posting a bid-ask pair, the learner only observes whether the seller accepts the bid and whether the buyer accepts the ask, while the agents' valuations remain unobserved. Equivalently, on each side of the market, the learner only learns whether the corresponding valuation lies above or below a posted threshold, rather than observing its actual value. This provides extremely coarse information relative to the complexity of the decision space. 
The LOB model is also non-trivial: whenever trade occurs on one side of the market, the learner only receives a one-bit signal for that specific direction of trade. Thus, although this hybrid observation model is more informative than purely two-bit feedback, it still reveals only partial information. Notably, despite this limitation, we show that in the LOB setting, one can still attain the same optimal bounds as in the full-feedback case. We also note that even the full-feedback setting, which is not the focus of this work, is far from trivial and has been studied in several related settings~\citep{Cesa_Bianchi_2021, cesa23mor}. We believe that this makes our results for the LOB setting particularly interesting, as it shows that the same optimal bounds can be achieved under weaker feedback.


\subsection{Overview of the Results}

In this section, we highlight the main contributions of our work. Specifically, we discuss the key results presented in each section of the paper.

\subsection*{Main Results of Section~\ref{sec:prel_result}}

In Section~\ref{sec:prel_result}, we provide some preliminary results on the relationship between $\text{OPT}$ and $J^*$. Specifically, the optimization problem underlying the definition of $\text{OPT}$ cannot be directly exploited to derive meaningful discretizations, since we do not impose any assumption on the shape of the buyers’ and sellers’ valuation distributions.
We therefore shift our focus to the surrogate objective function $J(b,a)$, establishing its two fundamental properties: \emph{i}) it is well defined for every bid--ask pair and satisfies $T J^*\ge\mathrm{OPT}$; \emph{ii}) it allows for meaningful discretizations of the decision space, as shown by the following result.
\begin{proposition*}[Proposition~\ref{prop:unigrid}]
    Let $N\in \mathbb N_{>0}$ and $\Gs_N$ be the uniform $1/N$-grid on $[0,1]$, namely:
    $$\Gs_N:=\{n/N:\ n=0,\dots N\}.$$
    Then, it holds:
    $$\max_{b,a\in \Gs_N^2} J(b,a)\ge \sup_{b,a\in [0,1]^2}J(b,a)-2/N,$$
    where $\mathcal{G}^2_N\coloneqq \mathcal{G}_N\times \mathcal{G}_N.$
\end{proposition*}
The result provides a precise approximation guarantee for the discretization of the bid--ask space. Indeed, optimizing $J$ over the uniform grid $\Gs_N^2$ yields a value that differs from the continuous optimum by at most $2/N$, showing that the discretized problem is not too far from the original one.

From this section onward, we will focus on bounding $R_T^J$, since this directly yields a bound on $R_T$.

\subsection*{Main Results of Section~\ref{sec:2B}}

In Section~\ref{sec:2B}, we focus on the 2-bit feedback scenario. In such a setting, we establish worst-case bounds on both the regret and the per-round inventory. 
\begin{theorem*}[Informal version of Theorem~\ref{thm:regret_two_bit}]
     Let $\delta\in(0,\nicefrac{1}{T})$. There exists a polynomial time algorithm such that, when two-bit feedback is available, it attains, with probability at least $1-\delta$ simultaneously for all the $t$ time-steps:
    $$|I_t|\lesssim T^{2/3} \sqrt{\log(T/\delta)} ,\qquad R^J_T\lesssim T^{2/3} \sqrt{\log(T/\delta)}.$$
\end{theorem*}
We remark that the upper bound on the regret provided by the previous result cannot be improved even in the easier unconstrained setting. This follows immediately from the $\Omega(T^{2/3})$ lower bound proved in~\citep{10.5555/946243.946352} for arbitrary stochastic demand curves. Indeed, our setting can be viewed as a generalization of dynamic pricing when the sellers' valuation distribution puts all its mass on $0$ and is known to the agent. In such a case, the optimum of the surrogate function $J$ collapses to the one of dynamic pricing, while our agent is still forced to satisfy additional constraints on the inventory, which makes our learning problem at least as hard as the dynamic pricing one.

\subsection*{Main Results of Section~\ref{sec:LOB}}

In Section~\ref{sec:LOB}, we switch to the arguably more interesting and realistic setting in which the agent has access to the \emph{limit order book}. Under this feedback model, we establish three main results.

We begin by establishing worst-case bounds on both the regret and the per-round inventory. These guarantees are informally stated in the following theorem.
\begin{theorem*}[Informal version of Theorem~\ref{thm:mainDKW}]
   Let $\delta\in(0,1)$. There exists a polynomial time algorithm such that, when LOB feedback is available, it attains, with probability at least $1-\delta$ simultaneously for all the $t$ time-steps:
    $$|I_t|\lesssim \sqrt{T\log(T/\delta)},\qquad R^J_T\lesssim \sqrt{T\log(T/\delta)}.$$ 
\end{theorem*}

We complement the worst-case analysis by establishing an instance-dependent guarantee, showing that the performance of the algorithm can be significantly better on favorable instances. In particular, both the regret and the per-round inventory imbalance scale as the maximum between the c.d.f. $F$ and the survival function $S$ at the discretized optimal bid--ask pair $(b^\circ,a^\circ)$, rather than only on $T$. 

    \begin{theorem*}[Informal version of Theorem~\ref{thm:mainBER}]
    Let $\delta\in(0,1)$. There exists a polynomial time algorithm such that, when LOB feedback is available, it attains, with probability at least $1-\delta$ simultaneously for all the $t$ time-steps:
    \[
        |I_t|\lesssim \sqrt{\max\{F(b^\circ),S(a^\circ)\}T\log(T/\delta)}+\log(T)\log(T/\delta),
    \]
    and
    \[R^J_T \lesssim \sqrt{\max\{F(b^\circ),S(a^\circ)\}T\log(T/\delta)}+\log(T)\log(T/\delta).\]
\end{theorem*}

In Figure~\ref{fig:small-overlap-valuations}, we provide a graphical representation of a favorable regime in which $\max\{F(b^\circ),S(a^\circ)\}$ is small. This happens when buyer and seller valuations have little overlap: few sellers have valuations below $b^\circ$, and few buyers have valuations above $a^\circ$. Equivalently, few sellers are willing to sell at a low price, and few buyers are willing to buy at a high price. This is a natural regime in financial applications, where the two sides of the market are often well separated.

The careful reader will notice that the instance-dependent bound immediately implies the worst-case bound, since $\max\{F(b^\circ),S(a^\circ)\}\leq 1$. Nevertheless, we state both results separately for the following reason. The instance-dependent guarantee requires a finer grid construction, which leads to a worse---though still polynomial---computational complexity than the one needed to obtain the worst-case bound, as we will discuss in the related section.

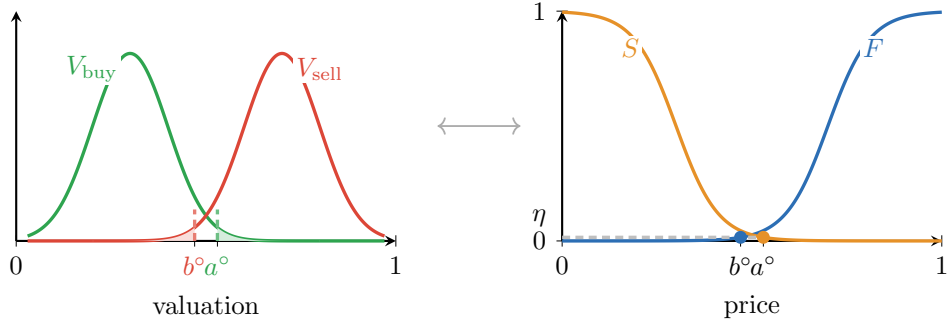
\begin{figure}[t]
\centering
\begin{tikzpicture}

\definecolor{sellc}{RGB}{46,164,79}
\definecolor{buyc}{RGB}{220,70,55}
\definecolor{cdfc}{RGB}{44,110,180}
\definecolor{survc}{RGB}{230,145,40}

\def\bstar{0.47}
\def\astar{0.53}
\pgfmathsetmacro{\etaval}{1/(1+exp(-18*(\bstar-0.70)))}

\begin{groupplot}[
    group style={group size=2 by 1, horizontal sep=2.2cm},
    width=0.40\textwidth,
    height=0.28\textwidth,
    xmin=0, xmax=1,
    axis line style={thick},
    tick style={black},
    tick label style={font=\small},
    label style={font=\small},
    every axis plot/.append style={line width=1.25pt},
    clip=false,
]

\nextgroupplot[
    ymin=0, ymax=1.9,
    axis lines=left,
    xlabel={valuation},
    ylabel={},
    xtick={0,\bstar,\astar,1},
    xticklabels={$0$,{\color{buyc}$b^\circ$},{\color{sellc}$a^\circ$},$1$},
    ytick=\empty,
]

\addplot[
    domain=0.03:0.97,
    samples=250,
    draw=sellc
]
{1.55*exp(-((x-0.30)^2)/(2*0.10^2))};

\addplot[
    domain=0.03:0.97,
    samples=250,
    draw=buyc
]
{1.55*exp(-((x-0.70)^2)/(2*0.10^2))};

\addplot[
    domain=\astar:0.97,
    samples=250,
    draw=none,
    fill=sellc!20
]
{1.55*exp(-((x-0.30)^2)/(2*0.10^2))} \closedcycle;

\addplot[
    domain=0.03:\bstar,
    samples=250,
    draw=none,
    fill=buyc!18
]
{1.55*exp(-((x-0.70)^2)/(2*0.10^2))} \closedcycle;

\addplot[densely dashed, buyc!70] coordinates {(\bstar,0) (\bstar,0.26)};
\addplot[densely dashed, sellc!70] coordinates {(\astar,0) (\astar,0.26)};

\node[text=sellc, font=\small, fill=white, inner sep=1pt] at (axis cs:0.20,1.42) {$V_{\mathrm{buy}}$};
\node[text=buyc, font=\small, fill=white, inner sep=1pt] at (axis cs:0.80,1.42) {$V_{\mathrm{sell}}$};


\nextgroupplot[
    ymin=0, ymax=1,
    axis lines=left,
    xlabel={price},
    ylabel={},
    xtick={0,\bstar,\astar,1},
    xticklabels={$0$,$b^\circ$,$a^\circ$,$1$},
    ytick={0,1},
]

\addplot[
    domain=0:1,
    samples=250,
    draw=cdfc
]
{1/(1+exp(-18*(x-0.70)))};

\addplot[
    domain=0:1,
    samples=250,
    draw=survc
]
{1/(1+exp(18*(x-0.30)))};

\addplot[densely dashed, gray!50] coordinates {(0,\etaval) (\astar,\etaval)};
\addplot[densely dashed, gray!50] coordinates {(\bstar,0) (\bstar,\etaval)};
\addplot[densely dashed, gray!50] coordinates {(\astar,0) (\astar,\etaval)};

\addplot[only marks, mark=*, mark size=1.8pt, cdfc] coordinates {(\bstar,\etaval)};
\addplot[only marks, mark=*, mark size=1.8pt, survc] coordinates {(\astar,\etaval)};

\node[text=survc, font=\small, fill=white, inner sep=1pt] at (axis cs:0.18,0.84) {$S$};
\node[text=cdfc, font=\small, fill=white, inner sep=1pt] at (axis cs:0.82,0.84) {$F$};
\node[font=\small, fill=white, inner sep=0.5pt] at (axis cs:-0.06,\etaval+0.085) {$\eta$};

\end{groupplot}

\draw[<->, line width=0.7pt, gray!65]
    ($(group c1r1.east)+(0.55cm,0)$) -- ($(group c2r1.west)+(-0.55cm,0)$);

\end{tikzpicture}
\caption{A favorable and arguably realistic regime. Left: the valuation distributions have little overlap, yielding a narrow optimal spread. Right: the corresponding quantities $F(b^\circ)$ and $S(a^\circ)$ are both small, namely of order $\eta\ll 1$.}
\label{fig:small-overlap-valuations}
\end{figure}

We conclude the section by establishing a lower bound showing that the $\widetilde{\mathcal{O}}(\sqrt{T})$\footnote{In this work, we use the $\widetilde{\mathcal{O}}$ to hide logarithmic terms.} regret and per-round imbalance guarantees are tight up to logarithmic factors. More precisely, we prove that the trade-off achieved by our upper bounds cannot be substantially improved: any algorithm enforcing a per-round inventory constraint of order $T^{\beta}$ for $\beta<\nicefrac{1}{2}$ must necessarily incur a regret of order at least $T^{1-\beta}$. 
\begin{theorem*}[Informal version of Theorem~\ref{prop:lowerbound}]
    Let $0\le\beta<\nicefrac{1}{2}$. Then, every algorithm that satisfies the per-round inventory constraint for $\alpha\le T^{\beta}$ suffers expected regret of order $\E[R_T]=\Omega(T^{1-\beta})$.
\end{theorem*}
We remark that this lower bound holds even in a significantly more informative setting than the one considered in our LOB model. In particular, it applies when the agent has \emph{complete prior knowledge} of the valuation distributions of both buyers and sellers. As a consequence, the impossibility result is not an artifact of the limited feedback available in the LOB setting; rather, it reflects an intrinsic difficulty of the problem under per-round inventory constraints.

\subsection*{Main Results of Section~\ref{sec:global}}

In Section~\ref{sec:global}, we provide a metaprocedure which, given as input any algorithm with similar guarantees to the ones of Sections~\ref{sec:2B}--\ref{sec:LOB}, obtains a global inventory constraint satisfaction, that is, at the end of the process, the number of sold items is equal to the number of purchased ones.
\begin{theorem*}[Informal version of Theorem~\ref{thm:global}] 
    Let $\delta\in(0,1)$ and let $\mathscr{A}$ be an algorithm for the online market making problem whose surrogate regret and per-round imbalance are upper bounded by $R^J_{\mathscr{A}}$ and $I_{\mathscr{A}}$, respectively, with probability at least $1-\delta$. Then, $\mathscr{A}$ can be modified so as to achieve, with probability at least $1-\delta$,
\[
R^J_T \leq R^J_{\mathscr{A}}+2I_{\mathscr{A}},\qquad |I_T|=0,\qquad |I_t|\leq I_{\mathscr{A}}\quad \forall t\in[T].
\]
\end{theorem*}
Instantiating the above result with the guarantees attained by our algorithms immediately yields the main results of this work. 

\section{Techniques and Proofs at a Glance}

In this section, we describe the main ideas employed by our algorithms to address the several technical challenges posed by our setting. The common difficulty is to learn profitable bid--ask pairs while keeping purchases and sales balanced. Indeed, exploration itself generates trades, and a pair with a large spread may lead to many more purchases than sales, or vice versa. 

Our approach relies on the \emph{surrogate} objective introduced above, which separates the value of a pair from the imbalance it may generate. Writing $F(b):=\Prob(\Vsell_t\le b)$ and $S(a):=\Prob(\Vbuy_t\ge a)$ for the two trade probabilities, recall that:
\[
J(b,a)=(a-b)\min\{F(b),S(a)\},\qquad T J^*\ge\mathrm{OPT}.
\]
The minimum captures the trading rate that both sides can sustain after reducing activity on the more frequent side. Thus, our algorithms aim to learn a nearly optimal pair for $J$ and to control inventory without losing the revenue associated with this common trading rate.

A first obstacle is the continuous action space. Since we impose no regularity assumptions on the valuation distributions, rounding prices can substantially change the trade probabilities. Nevertheless, the structure of $J$ allows us to control the loss through a one-sided discretization. Specifically, round a bid $b$ upward to $b'$ and an ask $a$ downward to $a'$ on a uniform grid of mesh $1/N$. Whenever $a-b>2/N$, the rounded pair is feasible and monotonicity gives:
\[
F(b')\ge F(b),\qquad S(a')\ge S(a),\qquad
J(b',a')\ge J(b,a)-\frac{2}{N}.
\]
Both trade probabilities can only improve, while the spread loses at most $2/N$. Therefore, letting $(b^\circ,a^\circ)$ maximize $J$ over feasible grid pairs, and denoting by $R_T^\circ$ the virtual regret obtained by replacing $J^*$ with $J(b^\circ,a^\circ)$ in the surrogate regret, we have:
\[
R_T^J=R_T^\circ+T\bigl(J^*-J(b^\circ,a^\circ)\bigr)
\le R_T^\circ+\frac{2T}{N}.
\]
This reduces the continuous problem to learning over a finite grid. The remaining challenge is to use the available feedback to identify good pairs while controlling the inventory generated.

\subsection*{Two-Bit Feedback}

Under two-bit feedback, each posted pair only reveals whether a purchase and a sale occurred. Estimating $F$ and $S$ therefore requires repeatedly posting prices, which may itself create a large imbalance. Our algorithm (Algorithm~\ref{alg:twobit}) addresses this through a \emph{phased elimination} procedure over bid and ask grids of size $\Theta(T^{1/3})$. At phase $\ell$, it maintains upper and lower confidence bounds for $F$ and $S$. These induce corresponding bounds for $J$; for instance, the optimistic value is:
\[
J_\ell^{\mathrm{up}}(b,a)
=(a-b)\min\{F_\ell^{\mathrm{up}}(b),S_\ell^{\mathrm{up}}(a)\}.
\]
The pessimistic value $J_\ell^{\mathrm{low}}$ is defined using the lower bounds. A feasible pair is called \emph{compatible} if its optimistic value reaches the largest pessimistic value among the active pairs. Prices with no compatible partner are eliminated. 

The key difficulty is to refine these bounds without continuing to trade on a side that is already too active. To estimate a bid $b$, the algorithm pairs it with the compatible ask $a'$ having the largest upper bound on $S$. It then runs the pair in a block, using a common target for the number of purchases and sales, where the target is the optimistic estimate of the number of trades sustainable on both sides, plus a confidence term. Once one side reaches this target, it is deactivated by bidding $0$ or asking $1$.

At first sight, halting seems to prevent accurate estimation. However, in a completed block, reaching the common target guarantees that the halted side is not the one determining the minimum in $J$. In particular, if the bid is halted, its true trade probability exceeds even the upper bound at $a'$. Since $a'$ has the largest such bound among compatible asks, on the concentration event we obtain, for every ask $a$ still compatible with $b$:
\[
F(b)>S(a),\qquad J(b,a)=(a-b)S(a).
\]
Further accuracy on $F(b)$ is therefore unnecessary for evaluating any relevant pair. If the bid is not halted, it is sampled throughout the block and its estimate is refined in the usual way. The ask side is symmetric. This is what allows the confidence bounds on $J$ to tighten even when one of the two trade probabilities is not estimated accurately. 

The same halting rule controls inventory and revenue. In a completed block of $n$ rounds, both trade counts are close to $n\min\{F(b),S(a')\}$, up to $\widetilde{\mathcal{O}}(\sqrt n)$. Their difference controls the inventory change, while the revenue is close to $nJ(b,a')$. Since compatible pairs are nearly optimal at the current confidence, each completed block contributes at most $\widetilde{\mathcal{O}}(\sqrt n)$ to both virtual regret and imbalance.

It remains to control inventory \emph{within} a block, since purchases and sales may occur at different times even if their final counts are close. For this reason, block lengths double across phases only up to order $T^{2/3}$; estimation and elimination then continue with blocks of that length. The growing blocks cost at most $\widetilde{\mathcal{O}}(T^{1/3})$ per grid price, while the later blocks cost $\widetilde{\mathcal{O}}(T^{-1/3})$ per round. Summing these two contributions gives order $\widetilde{\mathcal{O}}(T^{2/3})$. The unfinished-block excursion and the discretization error have the same order. Hence, with high probability, the algorithm guarantees:
\[
R_T^J\leq\widetilde{\mathcal{O}}(T^{2/3}),\quad \max_{t\in[T]}|I_t|
\le\widetilde{\mathcal{O}}(T^{2/3}).
\]

\subsection*{LOB Access}

With access to the \emph{limit order book}, the learner receives richer feedback, but only in truncated form. Specifically, a failed trade reveals the valuation, whereas a successful trade only reveals that the posted threshold was crossed. The first challenge is therefore to identify which prices can reuse these observations. If $b\ge B_\tau$, a purchase at $B_\tau$ already certifies that $\Vsell_\tau\le b$; if no purchase occurs, the revealed valuation determines whether this inequality holds. Formally, we have:
\[
\mathbbm{I}(\Vsell_\tau\le b)
=\mathbbm{I}(\max\{\Vsell_\tau,B_\tau\}\le b)
\qquad\text{for every }b\ge B_\tau.
\]
The right-hand side is observable, and the symmetric identity holds for asks $a\le A_\tau$. Thus, if all past bids are below an active bid and all past asks are above an active ask, averaging these indicators recovers exactly the empirical estimates $\widehat F_t$ and $\widehat S_t$ available under full feedback. Our algorithm (Algorithm~\ref{alg:ODPM}) preserves this property by maintaining nested active sets, with an increasing lower boundary for bids and a decreasing upper boundary for asks. Elimination is therefore also what makes it possible to reuse all past observations.

The difficulty is to move these boundaries without discarding the optimum. The \textsc{Double Double Pointer} subroutine (Algorithm~\ref{alg:DDP}) does this using two scans of the active grid. Starting from the smallest bid and largest ask, each scan advances the pointer on the side with the smaller trade-probability bound. Indeed, moving the other pointer would only reduce the spread without increasing the minimum. The first scan computes the largest pessimistic value of $J$; the second returns the first pair $(B_{t,0},A_{t,0})$ whose optimistic value reaches that benchmark. Whenever the confidence bounds hold, the second scan must stop before passing an optimal price since at the first such crossing, the current pair would already have an optimistic value at least as large as the optimal value. Consequently, the optimum remains active, and the returned candidate is nearly optimal up to the confidence error. 

The returned bid-ask pair may still generate unequal purchase and sale rates. To see how to correct this, suppose these rates were known and $F(b)>S(a)$. Keeping the bid with probability $S(a)/F(b)$ and otherwise bidding $0$ would make both trade rates equal to $S(a)$, while preserving expected revenue $(a-b)S(a)=J(b,a)$. Our algorithm implements this idea with confidence bounds. When $F^{\mathrm{low}}_{t-1}(B_{t,0})>S^{\mathrm{up}}_{t-1}(A_{t,0})$, it keeps the candidate bid with probability:
\[
\frac{S^{\mathrm{up}}_{t-1}(A_{t,0})}
     {F^{\mathrm{low}}_{t-1}(B_{t,0})},
\]
and otherwise resets it to $0$. The ask correction is again symmetric. If the intervals overlap, no correction is needed, since their widths already bound the mismatch. Thus, conditional on the past, both the inventory drift and the expected revenue loss relative to $J(B_{t,0},A_{t,0})$ are controlled by the confidence radii at the selected prices. Moreover, these corrections only move the posted prices outward, so they preserve the observation identities on the active sets.

Summing these confidence errors controls the cumulative inventory drift and conditional reward gaps; a martingale concentration then controls the fluctuations of realized trades around their conditional expectations. With DKW confidence radii, the cumulative error and these fluctuations are both of order $\widetilde{\mathcal{O}}(\sqrt T)$. A grid of mesh $1/\lceil\sqrt T\rceil$ gives the same discretization cost, yielding, with high probability:
\[
R_T^J\le\widetilde{\mathcal{O}}(\sqrt T),\quad \max_{t\in[T]}|I_t|
\le\widetilde{\mathcal{O}}(\sqrt T).
\]
Finally, the optimum's survival also explains the instance-dependent improvement. Since $B_{t,0}\le b^\circ$ and $A_{t,0}\ge a^\circ$, the selected trade probabilities are at most $F(b^\circ)$ and $S(a^\circ)$. These quantities control both the empirical-Bernstein confidence radii and the conditional variances of trades and revenues. When trades at the optimum are rare, both sources of error are therefore smaller. Using these radii and a finer grid of mesh $1/T$ gives, with high probability:
\[
R_T^J\le\widetilde{\mathcal{O}}\left(
\sqrt{T\max\{F(b^\circ),S(a^\circ)\}}\right),\quad \max_{t\in[T]}|I_t|
\le\widetilde{\mathcal{O}}\left(
\sqrt{T\max\{F(b^\circ),S(a^\circ)\}}\right).
\]

\section{Preliminaries} 
In the following, we better describe the maker-takers interaction. At each round $t\in[T]$, a \emph{new} pair of takers enters the market: a buyer and a seller, each interested in trading a single unit of the underlying asset. These takers are drawn from an underlying population, and are characterized by private valuations $\Vbuy_t,\Vsell_t$. Here, $\Vsell_t$ denotes the \textit{minimum} amount the seller is willing to accept to sell one unit, while $\Vbuy_t$ denotes the \textit{maximum} amount the buyer is willing to pay to buy one unit.
More formally, the valuation pair observed at each round is drawn i.i.d. from an unknown distribution. Specifically, there exists a probability distribution $\mathcal V$ supported on $[0,1]^2$ such that:
$$\left(\Vsell_t,\Vbuy_t\right) \overset{\indi}\sim \mathcal V\qquad \text{supp}(\mathcal V)\subset [0,1]^2.$$
Therefore, each round involves a freshly sampled buyer--seller pair from the same population. Denoting by $\mathcal V^{\text{buy}}$ and $\mathcal V^{\text{sell}}$ the two marginals of $\mathcal V$, we have $\Vbuy_t\sim \mathcal V^{\text{buy}}$ and $\Vsell_t\sim \mathcal V^{\text{sell}}$; however, within any given round, the two valuations may be \emph{arbitrarily correlated}.
At each round, the market maker selects a pair of bid and ask prices $(B_t,A_t)$ such that $0\leq B_t < A_t \leq 1$, without observing either $\Vsell_t$ or $\Vbuy_t$ in advance. The following interactions then take place simultaneously.
On the buy side, two scenarios may occur:
\begin{itemize}
\item \textbf{Maker buys.} If $\Vsell_t \leq B_t$, the seller accepts the bid and sells the asset to the maker at price $B_t$, resulting in a reward of $-B_t$ for the agent.
\item \textbf{No trade.} If $\Vsell_t > B_t$, no purchase occurs, and the agent receives reward $0$.
\end{itemize}
Symmetrically, on the sell side:
\begin{itemize}
\item \textbf{Maker sells.} If $\Vbuy_t \geq A_t$, the buyer accepts the ask and buys the asset from the maker at price $A_t$, resulting in a reward of $A_t$ for the agent.
\item \textbf{No trade.} If $\Vbuy_t < A_t$, no sale occurs, and the agent receives reward $0$.
\end{itemize}
Throughout the paper, we assume that $\Vbuy_t<1$ and $\Vsell_t>0$, for all $t\in[T]$. Equivalently, $\mathcal{V}$ is such that $\mathbb{P}(\Vbuy_t=1)=\mathbb{P}(\Vsell_t=0)=0$. This assumption is natural, since the boundary prices correspond to the case in which no buyer is willing to buy and no seller is willing to sell. Therefore, bidding $0$ and asking $1$ correspond, from the agent's perspective, to doing nothing. In this sense, the assumption is always satisfied in practice.

\subsection{Feedback Type}

Throughout the paper, we will focus on different feedback types denoting the information the market maker receives at the end of each round. Specifically, we can have one of the following feedback:
\begin{itemize}
    \item \textbf{Two-bit feedback \textcolor{green!60!black}{(2-Bit)}.} When only two-bit feedback is available, the agent observes solely whether a trade has occurred. Specifically, at the end of round $t\in[T]$, the agent observes $\mathbbm{I}(\Vbuy_t\geq A_t)$ and $\mathbbm{I}(\Vsell_t\leq B_t)$. This two-bit feedback model (and, similarly, its one-bit counterpart) is predominant in revenue maximization settings~\citep{10.5555/946243.946352, castiglioni2026sample} and is generally regarded as the hardest form of feedback, in contrast to the easier full-feedback setting, where the valuations are fully observed at the end of each round.
    
    \item \textbf{Limit order book access \textcolor{cyan!60!black}{(LOB)}.} When access to the limit order book is available, at the end of each round the agent observes additional information whenever no trade occurs. Specifically, if $B_t < \Vsell_t$, the agent observes $\Vsell_t$. Similarly, if $A_t > \Vbuy_t$, it observes $\Vbuy_t$.
In other words, \textit{information arrives exactly when reward does not}. This assumption captures a key feature of the market-making problem. Indeed, whenever a transaction occurs (either on the buy side or on the sell side), the agent never observes the actual valuation, since takers have no incentive to reveal it. By contrast, when no transaction takes place, takers are incentivized to post their valuation in the order book, hoping that their order will be executed in the near future. Notice that LOB feedback can be viewed as a meaningful hybrid feedback model, interpolating between the two-bit and the full-feedback settings.
\end{itemize}

In Protocol~\ref{alg: Learner-Environment Interaction}, we provide the online interaction between the agent and sequence of takers.

\begin{protocol}[!htp]
		\caption{Maker-Takers Interaction}
		\label{alg: Learner-Environment Interaction}
        \begin{algorithmic}[1]
		\For{$t\in[T]$}
			
			\State A new pair of buyer and seller arrives with private valuations $\Vbuy_t\sim \mathcal{V}^{\text{buy}}$ and $\Vsell_t\sim \mathcal{V}^{\text{sell}}$
            
             \State The agent selects bid and ask prices $(B_t,A_t)$
			
			\State The agent gets reward $A_t\cdot \mathbbm{I}(\Vbuy_t\geq A_t) - B_t \cdot \mathbbm{I}(\Vsell_t\leq B_t)$ 
					
             \begin{tcolorbox} [myblockstyle=green] \State The agent observes $\mathbbm{I}(\Vbuy_t\geq A_t)$ and $\mathbbm{I}(\Vsell_t\leq B_t)$ \Comment{{2-Bit}}
             \end{tcolorbox}
            
            \begin{tcolorbox} [myblockstyle=cyan]
            \State The agent observes $\Vbuy_t$ whenever $\Vbuy_t<A_t$ and $\Vsell_t$ whenever $\Vsell_t>B_t$.
           \Comment{{LOB}}
           \end{tcolorbox}
			
		\EndFor
        \end{algorithmic}
	\end{protocol}

\subsection{The Inventory Constraints}

Market making differs from standard buy--sell interactions between agents. In particular, market makers can sell an asset even if they have not bought it yet (this is the so-called \emph{short selling}).
However, market makers usually try to keep their \emph{inventory} small, meaning they try to avoid building a large imbalance between how many units they have sold and how many they have bought. The reason is simple:
\begin{itemize}
    \item if the maker ends up holding too much of the asset (positive inventory), they are exposed to the risk that the market price goes down;
    \item if the maker ends up too short (negative inventory), they are exposed to the risk that the market price goes up.
\end{itemize}
In both cases, the market maker becomes more sensitive to price movements, because part of their profit (or loss) starts depending on the market valuation of the asset, not only on the trading activity itself. This is typically against the idea of market making, which (in the simplest view) aims to earn mainly from the \emph{bid--ask spread}, while keeping directional price risk as low as possible.

We encode this aspect by imposing the following two constraints on the learning dynamics.  
On the one hand, we require a \emph{global inventory balance} over the whole horizon, namely:
\begin{equation}\label{eq:consistency_constraints}
\sum_{t=1}^T \mathbbm{I}(\Vbuy_t\ge A_t)=\sum_{t=1}^T\mathbbm{I}(\Vsell_t\le B_t).  
\end{equation}
Equation~\eqref{eq:consistency_constraints} simply states that, at the end of the learning dynamics, the inventory must be zero, while still allowing the agent to take short positions (i.e., to sell before buying) during learning.

On the other hand, we enforce that the \emph{inventory level at each round} stays below a prescribed threshold $\alpha\in\mathbb{N}_{>0}$:
\begin{equation}\label{eq:per-round_consistency constraints}
\left|\sum_{\tau=1}^t \mathbbm{I}(\Vbuy_\tau\ge A_\tau) -\sum_{\tau=1}^t\mathbbm{I}(\Vsell_\tau\le B_\tau)\right| \leq \alpha \quad \forall t\in[T].  
\end{equation}
For the sake of simplicity we will introduce the following definition.
\begin{definition}[Imbalance]\label{def:imbalance}
    We define the current imbalance at time $t\in[T]$ as:
    $$I_t:=\sum_{\tau=1}^t \mathbbm{I}(\Vsell_\tau\le B_\tau)-\sum_{\tau=1}^t \mathbbm{I}(\Vbuy_\tau\ge A_\tau).$$
\end{definition}
Thus, Equation~\eqref{eq:per-round_consistency constraints} reduces to $|I_t|\leq \alpha$, for all $t\in[T]$.

\subsection{The Offline Baseline and The Performance Metrics}

As is standard in online learning, we compare the performance of our learning algorithms against an offline benchmark that has full knowledge of the underlying valuation distributions and commits to a fixed bid-ask pair throughout the entire horizon. This benchmark serves as the baseline for our notion of regret.
Specifically, we define the optimal offline value as:
\begin{equation}\label{progr:offline_opt}
	\mathrm{OPT} \coloneqq \begin{cases}
		\max_{b,a\in [0,1]^2}  & \sum_{t=1}^T \left[a\, \E \left[\mathbbm{I}(\Vbuy_t\geq a)\right] - b \, \E \left[\mathbbm{I}(\Vsell_t\leq b) \right]\right] \\
        \text{s.t.} &
		\sum_{t=1}^T \E \left[\mathbbm{I}(\Vbuy_t\ge a)\right]
        =
        \sum_{t=1}^T\E \left[\mathbbm{I}(\Vsell_t\le b)\right] \\
		& b<a,
	\end{cases}
\end{equation}
where all expectations are taken with respect to the unknown valuation distribution. Observe that the benchmark corresponds to the best fixed bid-ask $(b,a)$ computed knowing the valuations distribution, subject to the overall inventory constraint.

For ease of notation, we let $F_{\mathrm{buy}}:[0,1]\to[0,1]$ and $F_{\mathrm{sell}}:[0,1]\to[0,1]$ denote the cumulative density functions of $\mathcal V^{\mathrm{buy}}$ and $\mathcal V^{\mathrm{sell}}$, respectively. Then, Program~\eqref{progr:offline_opt} can be equivalently rewritten by replacing:
\[
\E\left[\mathbbm{I}(\Vbuy_t\ge a)\right] = 1-F_{\mathrm{buy}}(a)\eqqcolon S(a)
\qquad\text{and}\qquad
\E\left[\mathbbm{I}(\Vsell_t\le b)\right] = F_{\mathrm{sell}}(b)\eqqcolon F(b).
\]
Given this benchmark, we define the \emph{cumulative regret} of a learning algorithm as:
\[
R_T\coloneqq \mathrm{OPT}-\sum_{t=1}^T \left[A_t\cdot \mathbbm{I}(\Vbuy_t\geq A_t) - B_t \cdot \mathbbm{I}(\Vsell_t\leq B_t)\right].
\]
In the rest of the work, we will refer to $A_t\mathbbm{I}(\Vbuy_t\ge A_t)-B_t\mathbbm{I}(\Vsell_t\le B_t)$ as $r_t$.

Additionally, we introduce the following notion of surrogate objective function. As we show next, it is closely related to Program~\eqref{progr:offline_opt} and will serve as a key tool in our analysis.
\begin{definition}[Surrogate objective function]\label{def:surr_obj}
    The surrogate objective is defined as:
    \begin{equation} \label{eq:surrogate}
    J(b,a):=(a-b)\min\{F(b),S(a)\}.
    \end{equation}
    Furthermore, its optimal value is defined as:
    \begin{equation*} 
    J^*:=\sup_{b,a\in[0,1]^2}(a-b)\min\{F(b),S(a)\}.
    \end{equation*}  
\end{definition}
As we will show in Section~\ref{sec:prel_result}, it holds that $T\cdot J^*\geq \text{OPT}$. Accordingly, we can define the following notion of \emph{surrogate cumulative regret}:
\begin{equation}
R^J_T\coloneqq J^*\cdot T-\sum_{t=1}^T \left[A_t\cdot \mathbbm{I}(\Vbuy_t\geq A_t) - B_t \cdot \mathbbm{I}(\Vsell_t\leq B_t)\right].\label{eq:surr_regret}
\end{equation}
The goal of the maker is to ensure that $R_T=o(T)$, while at the same time guaranteeing that Equation~\eqref{eq:consistency_constraints} and Equation~\eqref{eq:per-round_consistency constraints}, for suitable values of $\alpha$, are satisfied by the sequence of bid-ask pairs selected during learning. In this work, we go even further: we show that the same results can be attained against the stronger benchmark $J^*$, namely, ensuring that $R^J_T=o(T)$, while at the same time guaranteeing that the constraints are satisfied.

\section{Preliminary Results: Program~\eqref{progr:offline_opt} and The Surrogate Objective}
\label{sec:prel_result}
In this section, we provide some preliminary results that will be employed for both the setting characterized by the different feedback. Specifically, we focus on the fundamental relation between Program~\eqref{progr:offline_opt} and the surrogate objective function introduced in Definition~\ref{def:surr_obj}.

We start by the following simple characterization of Program~\eqref{progr:offline_opt}. 
\begin{proposition}\label{prop:threefive}
    Let $\mathcal R$ be the feasible region associated to Program~\eqref{progr:offline_opt}. Then, the following holds:
    \begin{enumerate}
        \item $\mathcal R$ is non-empty.
        \item The solution to Program~\eqref{progr:offline_opt} corresponds to the one of the following problem:      \begin{equation}\label{progr:stoch_opt}
        	\textnormal{OPT} = \begin{cases}
        		\max_{b,a\in [0,1]^2}  & T\left[aS(a)-bF(b)\right]\quad 
                \\\text{s.t.}
        		&F(b)=S(a).
        	\end{cases}
        \end{equation}
    \end{enumerate}
\end{proposition}
\begin{proof}
    For any $t\in[T]$, by definition of $F,S$ it follows that: \[a\E \left[\mathbbm{I}(\Vbuy_t\geq a)\right]=aS(a) \ \  \text{ and }\ \  b\E \left[\mathbbm{I}(\Vsell_t\leq b) \right]=bF(b).\] Therefore, Program~\eqref{progr:offline_opt} writes as follows:
    \begin{equation*}
        \text{OPT} = \begin{cases}
            \max_{b,a\in [0,1]^2}  & T\left[aS(a)-bF(b)\right]\quad 
            \\ \text{s.t.}
            &F(b)=S(a)\\
            &b<a.
        \end{cases}
    \end{equation*}
    Under the assumption that $\Vbuy_t<1$ and $\Vsell_t>0$ for all $t\in[T]$, the point $b=0,a=1$ is always feasible and attains an objective value equal to $0$. Therefore, $(0,1)\in \Rs$ and, for any $(a,b)$ such that $a\le b$ the objective function is equivalent to:
    $$T\left[aS(a)-bF(b)\right]\le 0 = T\left[1\cdot S(1)-0\cdot F(0)\right].$$
    This allows to remove the constraint $b<a$, since any point that falls outside of this set is automatically dominated by $(0,1)$, concluding the proof.
\end{proof}
The constrained maximization of the former objective function is strictly linked to the one of the nonlinear and \emph{unconstrained} objective provided in Definition~\ref{def:surr_obj}.
The relation between this objective and the one of Problem~\eqref{progr:offline_opt} is provided in the following proposition.
\begin{proposition}
    Let $\mathcal{R}$ be the feasible region of Program~\eqref{progr:offline_opt}. Given $(b,a)\in \Rs$, it holds: 
    \[J(b,a)=aS(a)-bF(b).\]
\end{proposition}
\begin{proof}
    Given $(b,a)\in \Rs$, one has, by Proposition \ref{prop:threefive}, $F(b)=S(a)$. Therefore, it holds:
    \begin{align*}
        J(b,a)&=(a-b)\min\{F(b),S(a)\}\\
        &=a\min\{F(b),S(a)\}-b\min\{F(b),S(a)\}\\
        &=aS(a)-bF(b).
    \end{align*}
    This concludes the proof.
\end{proof}
We remark that the opposite is \emph{not} true, in general. Indeed, not only is $J(b,a)$ defined on points that are infeasible in Program~\eqref{progr:offline_opt}, but the optimum of $J$ (specifically, $J^*\cdot T$) may be \textit{strictly} larger than $\text{OPT}$. Specifically, consider the following instance, defined for some small $\varepsilon>0$:
$$F(b):=
\begin{cases}
    0 &\qquad b< \varepsilon\\
    1/3 &\qquad \varepsilon\le b<1-\varepsilon\\
    1 & \qquad 1-\varepsilon\le b
\end{cases}
\qquad
S(a):=
\begin{cases}
    1 &\qquad a\le \varepsilon\\
    2/3 &\qquad \varepsilon<a\le 1-\varepsilon\\
    0 & \qquad 1-\varepsilon<a
\end{cases}.
$$
In such a case, every feasible solution to Program~\eqref{progr:offline_opt} attains an objective value of zero. On the contrary, it holds:
$$J^* \coloneqq\sup_{b,a\in [0,1]^2}J(b,a)\ge J(2\varepsilon, 1-\varepsilon)=\frac{(1-3\varepsilon)}{3}\approx 1/3.$$
This shows that the optimal value $J^*$ can outperform the one of Program~\eqref{progr:offline_opt} (divided by $T$) of a positive constant. Therefore, if an algorithm enjoys low regret compared to $J^*$, the same regret guarantees hold \textit{a fortiori} with respect to $\text{OPT}$.

The surrogate objective $J$ offers several advantages. First, it is well defined for every pair $(b,a)$. More importantly, it satisfies a useful smoothness-like property close to the maximum, independently of any smoothness of $F$ and $S$. This is shown in the following proposition.
\begin{proposition}\label{prop:unigrid}
    Let $N\in \mathbb N_{>0}$ and $\Gs_N$ be the uniform $1/N$-grid on $[0,1]$, namely:
    $$\Gs_N:=\{n/N:\ n=0,\dots N\}.$$
    Then, it holds:
    $$\max_{b,a\in \Gs_N^2} J(b,a)\ge \sup_{b,a\in [0,1]^2}J(b,a)-2/N,$$
    where $\mathcal{G}^2_N\coloneqq \mathcal{G}_N\times \mathcal{G}_N.$
\end{proposition}
\begin{proof}
    Fix $\varepsilon>0$. By definition of supremum, there is $(b_0,a_0)\in [0,1]^2$ such that: $$J(b_0,a_0)\ge \sup_{b,a\in [0,1]^2}J(b,a)-\varepsilon.$$
    By definition of the grid, there is $(b_0',a_0')\in \Gs_N^2$ such that $a_0-N^{-1}\le a_0'\le a_0$ and $b_0 \le b_0'\le b_0+N^{-1}$. Therefore, it holds:
    \begin{align*}
        J(b_0',a_0')&=(a_0'-b_0')\min\{F(b_0'),S(a_0')\}\\
        &\ge (a_0-b_0-2N^{-1})\min\{F(b_0'),S(a_0')\}\\
        &\ge (a_0-b_0-2N^{-1})\min\{F(b_0),S(a_0)\}\\
        &\ge (a_0-b_0)\min\{F(b_0),S(a_0)\}-2N^{-1}.
    \end{align*}
    The second inequality is the \emph{crucial} step of the former derivation. This follows from the fact that $F$ is non-decreasing, while $S$ is non-increasing.
    By definition of $b_0,a_0$, we have shown that:
    $$\max_{b,a\in \Gs_N^2} J(b,a)\ge \sup_{b,a\in [0,1]^2}J(b,a)-\varepsilon-2/N.$$
    The proof is concluded letting $\varepsilon \to 0$.
\end{proof}
\begin{figure}[t]
\centering
\begin{tikzpicture}[scale=5.2]

    \def\N{5}          
    \def\xopt{0.34}    
    \def\yopt{0.67}    
    \def\xgrid{0.4}    
    \def\ygrid{0.6}    

    \draw[->, thick] (0,0) -- (1.06,0) node[below] {$b$};
    \draw[->, thick] (0,0) -- (0,1.06) node[left] {$a$};

    \fill[gray!20, opacity=0.25] (0,0) -- (1,1) -- (0,1) -- cycle;

    \draw[thick] (0,0) rectangle (1,1);
    \draw[step=0.2, gray!45, very thin] (0,0) grid (1,1);

    \draw[dashed, gray!70] (0,0) -- (1,1);
    \node[gray!70] at (0.8,0.9) {$a>b$};

    \fill[blue!12] (0.2,0.6) rectangle (0.4,0.8);
    \draw[blue!50, thin] (0.2,0.6) rectangle (0.4,0.8);

    \foreach \i in {0,1,...,5}{
        \foreach \j in {0,1,...,5}{
            \fill[black] ({\i/5},{\j/5}) circle (0.32pt);
        }
    }

    \fill[red] (\xopt,\yopt) circle (0.8pt);
    \node[red, above left] at (\xopt,\yopt) {$(b_0,a_0)$};

    \fill[blue] (\xgrid,\ygrid) circle (0.9pt);
    \node[blue, below right] at (\xgrid,\ygrid) {$(b_0',a_0')$};

    \draw[densely dashed, red!70] (\xopt,\yopt) -- (\xopt,0);
    \draw[densely dashed, red!70] (\xopt,\yopt) -- (0,\yopt);
    \draw[densely dashed, blue!70] (\xgrid,\ygrid) -- (\xgrid,0);
    \draw[densely dashed, blue!70] (\xgrid,\ygrid) -- (0,\ygrid);

    \draw[->, thick, blue!70] (\xopt,\yopt) -- (\xgrid,\ygrid);

    \foreach \x in {0,0.2,0.4,0.6,0.8,1.0}{
        \draw (\x,0.008) -- (\x,-0.008);
    }
    \foreach \y in {0,0.2,0.4,0.6,0.8,1.0}{
        \draw (0.008,\y) -- (-0.008,\y);
    }

    \node[below] at (0,0) {$0$};
    \node[below] at (1,0) {$1$};
    \node[left] at (0,1) {$1$};

\end{tikzpicture}
\caption{Illustration of the discretization argument in Proposition~\ref{prop:unigrid}. The point $(b_0,a_0)$ is an optimizer (or an $\varepsilon$-optimizer) of $J$ over $[0,1]^2$, while $(b_0',a_0')\in\mathcal{G}_N^2$ is the grid point selected in the proof, with $b_0\le b_0'\le 1/N+b_0$ and $a_0-1/N\le a_0'\le a_0$.}
\label{fig:unigrid}
\end{figure}
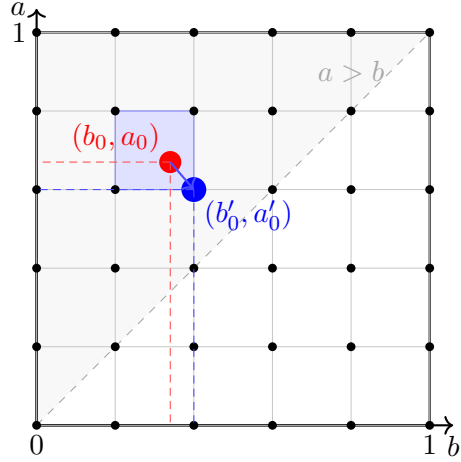
In Figure~\ref{fig:unigrid}, we provide a graphical representation of the proposition.

For notational convenience, we define the (virtual) surrogate optimal solution over an arbitrary grid as follows. Notice that, at this stage, we keep the definition as general as possible. The actual grid employed by our algorithms will instead depend on the particular setting under consideration and will be specified once needed.

\begin{definition}[Virtual optimal bid-ask pair]\label{def:virtual}
    Let $\Bs^\circ \subset [0,1], \As^\circ\subset [0,1]$ be finite subsets of the bid space and the ask space, respectively. We define the virtual optimal solution on the grid as:
    \begin{equation}
    (b^\circ,a^\circ)\in \argmax_{b\in \Bs^\circ,\; a\in \As^\circ} J(b,a).
    \label{eq:abcirctwo}
    \end{equation}
\end{definition}
If $J$ admits multiple optimal solutions, the pair $(b^\circ,a^\circ)$ need not be unique. All of our results will hold regardless of how ties are broken.
Similarly, we define the following notion of virtual regret.
\begin{definition}[Virtual regret]\label{def:virtualregret}
    The virtual regret is defined as:   
    \[R_T^\circ:=T\cdot J(b^\circ,a^\circ)-\sum_{t=1}^T \left[A_t\cdot \mathbbm{I}(\Vbuy_t\geq A_t) - B_t \cdot \mathbbm{I}(\Vsell_t\leq B_t)\right].\]
\end{definition}
The virtual regret is defined analogously to the surrogate regret in Equation~\eqref{eq:surr_regret}, with only one difference: it is evaluated with respect to $b^\circ$ and $a^\circ$, which belong to $\Bs^\circ$ and $\As^\circ$, rather than with respect to $J^*$. 

To conclude, we summarize the main insights obtained so far. As already discussed, the structure of Program~\eqref{progr:offline_opt} cannot be directly exploited to derive meaningful discretizations, since we do not impose any assumption on the shape of the valuation distributions of buyers and sellers. For this reason, we shift our focus to the surrogate objective function introduced in Definition~\ref{def:surr_obj}. In this section, we have established two fundamental properties of this surrogate objective: \emph{i}) it is well-defined for every bid-ask pair and satisfies $T J^*\ge\text{OPT}$; \emph{ii}) it allows for meaningful discretizations of the decision space, as shown in Proposition~\ref{prop:unigrid}. Both properties will play a crucial role in the analysis developed in the following sections.

From this section onward, we will focus on bounding $R_T^J$, since this directly yields a bound on $R_T$. Specifically, our strategy will be to first bound the virtual regret $R_T^\circ$ of our algorithms, and then convert this into a bound on the surrogate regret $R_T^J$ using the results of this section.

\section{Two-Bit Feedback}
\label{sec:2B}

In this section, we focus on the case the market maker receives \emph{two-bit} feedback. The
results of this section are provided in the following order:
\begin{itemize}
    \item In Section~\ref{sec:alg_2bit}, we describe our algorithm for the 2-bit feedback case and we provide some related preliminary results. Then, we study the associated computational complexity.
    \item In Section~\ref{sec:theo_2b}, we provide the theoretical guarantees attained by our algorithm in terms of surrogate regret and per-round inventory constraints. We first define the good event associated to the performance of our algorithm and the associated results. Then, we provide the final regret and per-round inventory bounds.
\end{itemize}

\subsection{Our Algorithm}
\label{sec:alg_2bit}
In this section, we describe our algorithm for the two-bit feedback case. Specifically, the pseudocode of our algorithm is provided in Algorithm~\ref{alg:twobit}. 

\begin{algorithm}[t]
\caption{\textsc{Two-Bit Market Maker}}
\label{alg:twobit}
\begin{algorithmic}[1]
\Require $T\in\mathbb{N}_{>0}$, $\delta\in(0,1)$
\State $\ell\gets 0$ 
\State Set $\Bs_0\gets \Bs^\circ$
\State Set $\As_0\gets \As^\circ$
\State $F_0^{\mathrm{up}}(b)\gets 1$, $F_0^{\mathrm{low}}(b)\gets 0$, for all $b\in \Bs_0$ and $S_0^{\mathrm{up}}(a)\gets 1$, and $S_0^{\mathrm{low}}(a)\gets 0$ for all $a\in \As_0$
\State Compute $\Cs_0$ as in Equation~\eqref{eq:coset}
\While{Total number of steps $\le T$}
    \For{$b \in \Bs_\ell$}
        \State $F_{\ell+1}^\text{up}(b),F_{\ell+1}^\text{low}(b)\gets$ Algorithm~\ref{alg:SAE} $(b, \ell, F_\ell^{\mathrm{up}}, F_\ell^{\mathrm{low}}, S_\ell^{\mathrm{up}}, S_\ell^{\mathrm{low}})$     
    \EndFor
    \State $\As_{\ell+1} \gets \As_\ell$
    \For{$a \in \As_\ell$}
        \State $S_{\ell+1}^\text{up}(a),S_{\ell+1}^\text{low}(a)\gets$ Algorithm~\ref{alg:SAE_ask} $(a, \ell, F_\ell^{\mathrm{up}}, F_\ell^{\mathrm{low}}, S_\ell^{\mathrm{up}}, S_\ell^{\mathrm{low}})$
    \EndFor
    \State Compute all $\Cs_{\ell+1}$ as in Equation~\eqref{eq:coset}, using the new bounds on $\Bs_\ell\times\As_\ell$
    \State $\Bs_{\ell+1} \gets \{b\in\Bs_\ell:\Cs_{\ell+1}(b)\ne\emptyset\}$
    \State $\As_{\ell+1} \gets \{a\in\As_\ell:\Cs_{\ell+1}(a)\ne\emptyset\}$
    \State $\ell \gets \ell+1$
\EndWhile
\end{algorithmic}
\end{algorithm}

Algorithm~\ref{alg:twobit} works on discretized subsets of the bid and ask spaces, denoted by
$\Bs^\circ \subset [0,1]$ and $\As^\circ \subset[0,1]$, respectively, which are defined as:
\[
\Bs^\circ := \left\{ \frac{n}{\lfloor T^{1/3}\rfloor} : n=0,1,\ldots,\lfloor T^{1/3}\rfloor-1 \right\},
\qquad
\As^\circ := \left\{ \frac{n}{\lfloor T^{1/3}\rfloor} : n=1,2,\ldots,\lfloor T^{1/3}\rfloor \right\}.
\]
Both sets have cardinality exactly $\lfloor T^{1/3}\rfloor$. Thus, Algorithm~\ref{alg:twobit} follows a \emph{phased elimination} strategy: it maintains a set of active bid-ask pairs and, at the end of each phase, eliminates some of them to improve the performance. In phase $\ell$, each active price is assigned a block of $2^{\min\{\ell,\lfloor\log_2(T^{2/3})\rfloor\}}$ rounds, and the confidence radius is:
\[
\psi_\ell:=\sqrt{\frac{\log(4T/\delta)}{2^{\min\{\ell,\lfloor\log_2(T^{2/3})\rfloor\}}}}.
\]
The block length doubles across phases until it reaches $2^{\lfloor\log_2(T^{2/3})\rfloor}$, and then remains fixed. The phase index keeps increasing, and estimation and elimination continue in every phase. Limiting the block length controls the inventory changes within a block, as explained below.

For each phase $\ell$, let:
\[
\Bs_\ell \subseteq \Bs^\circ,
\qquad
\As_\ell \subseteq \As^\circ
\]
denote the sets of active bids and active asks after phase $\ell$, kept fixed throughout phase $\ell+1$.

Algorithm~\ref{alg:twobit} employs Algorithms~\ref{alg:SAE}--\ref{alg:SAE_ask} as subroutines to compute, for each phase $\ell$ and each $b\in\Bs_{\ell-1}$, $a\in\As_{\ell-1}$, estimators of the cumulative distribution function $F$ and the survival function $S$, denoted by $\widehat F_\ell$ and $\widehat S_\ell$, respectively. In addition, the subroutines compute confidence bounds for these quantities, namely $F_\ell^{\mathrm{up}}(b)$, $F_\ell^{\mathrm{low}}(b)$, $S_\ell^{\mathrm{up}}(a)$, and $S_\ell^{\mathrm{low}}(a)$. Given that, for every pair $b\in\Bs_\ell$ and $a\in\As_\ell$, it is possible to define pessimistic and optimistic surrogate function estimates as:
\begin{equation}
J_\ell^\text{low}(b,a)=(a-b)\min\{F_\ell^{\mathrm{low}}(b),S_\ell^{\mathrm{low}}(a)\},\qquad
J_\ell^\text{up}(b,a)=(a-b)\min\{F_\ell^{\mathrm{up}}(b),S_\ell^{\mathrm{up}}(a)\}.
\label{eq:twobitJJ}
\end{equation}
Finally, letting $J_\ell^\text{low*}:=\max_{b\in \Bs_\ell,\ a\in \As_\ell}J_\ell^\text{low}(b,a)$, we define the \emph{compatibility sets}:
\begin{equation}
\Cs_\ell(a)=\left\{b\in \Bs_\ell :J_\ell^\text{up}(b,a)\ge J_\ell^\text{low*}, a>b\right\},\quad
\Cs_\ell(b)=\left\{a\in \As_\ell :J_\ell^\text{up}(b,a)\ge J_\ell^\text{low*}, a>b\right\}.
\label{eq:coset}
\end{equation}
The sets provided in Equation~\eqref{eq:coset} are the ones that drive the \emph{elimination} procedure. Specifically, whenever a bid (resp., an ask) corresponds to an empty $\Cs_\ell(b)$ (resp., $\Cs_\ell(a)$), it is eliminated from the set of active bid-ask pairs.
At phase $\ell+1$, all bid and ask estimates are first computed using the bounds and compatibility sets from phase $\ell$. The new bounds are then used to compute $\Cs_{\ell+1}$ and eliminate prices with no compatible partner. Every compatible pair keeps both its prices, so no further update of the compatibility sets is needed.\footnote{For completeness, if a subroutine returns an empty confidence interval, the algorithm selects $(0,1)$ in all subsequent rounds. This never occurs on the good event.}

\subsubsection{Subroutines Description}

\begin{algorithm}[t]
\caption{\textsc{Single-action estimator - Bid}}
\label{alg:SAE}
\begin{algorithmic}[1]
\Require Bid $b$ to estimate, current phase $\ell+1$, all the estimators computed at phase $\ell$
\State Set $\psi_{\ell+1}\gets\sqrt{\frac{\log(4T/\delta)}{2^{\min\{\ell+1,\lfloor\log_2(T^{2/3})\rfloor\}}}}$
\State Set
        $a'\gets \argmax \{S_\ell^\text{up}(a), a\in \Cs_\ell(b)\}$\label{algline:sao}
\State Set $u\gets \min\{F_\ell^\text{up} (b),S_\ell^\text{up}(a')\}+\psi_{\ell+1}$\label{algline:minim}
\State $X_0\gets 0$
\State $Y_0\gets 0$
\State $B\gets b$
\State $A\gets a'$
\For {$t=1,\dots 2^{\min\{\ell+1,\lfloor\log_2(T^{2/3})\rfloor\}}$}
    \State Select $(B,A)$
    \State Update $X_t=X_{t-1}+\mathbbm{I}(\Vsell_t\le B)$ \label{algline:x}
    \State Update $Y_t=Y_{t-1}+\mathbbm{I}(\Vbuy_t\ge A)$ \label{algline:y}
    
    \If{$X_t\ge 2^{\min\{\ell+1,\lfloor\log_2(T^{2/3})\rfloor\}} u$}\label{algline:halting} \Comment{\textsc{Halting condition $\mathrm{I}$}}
        \State $B\gets0$
        \State $F^\text{up}(b)\gets 1$\label{algline:place2}
        \State $F^\text{low}(b)\gets \min\{F_\ell^\text{up}(b),S_\ell^\text{up}(a')\}$\label{algline:place1}
        \Else{}
        \State $\widehat F_{\ell+1}(b)\gets \frac{X_t}{t}$\label{algline:key2}
        \State $F^\text{up}(b)\gets \widehat F_{\ell+1}(b)+\psi_{\ell+1}$
                
        \State $F^\text{low}(b)\gets \widehat F_{\ell+1}(b)-\psi_{\ell+1}$
    \EndIf
    \If{$Y_t\ge 2^{\min\{\ell+1,\lfloor\log_2(T^{2/3})\rfloor\}} u$}\label{algline:halting_two} \Comment{\textsc{Halting condition $\mathrm{II}$}}
        \State $A\gets1$
    \EndIf
\EndFor
\State \textbf{Return} $\min\{F^\text{up}(b),F_{\ell}^\text{up}(b)\},\max\{F^\text{low}(b),F_{\ell}^\text{low}(b)\}$\label{algline:return}
\end{algorithmic}
\end{algorithm}

\begin{algorithm}[t]
\caption{\textsc{Single-action estimator - Ask}}
\label{alg:SAE_ask}
\begin{algorithmic}[1]
\Require Ask $a$ to estimate, current phase $\ell+1$, all the estimators computed at phase $\ell$
\State Set $\psi_{\ell+1}\gets\sqrt{\frac{\log(4T/\delta)}{2^{\min\{\ell+1,\lfloor\log_2(T^{2/3})\rfloor\}}}}$
\State Set
        $b'\gets \argmax \{F_\ell^\text{up}(b), b\in \Cs_\ell(a)\}$\label{algline:sao_ask}
\State Set $u\gets \min\{F_\ell^\text{up} (b'),S_\ell^\text{up}(a)\}+\psi_{\ell+1}$\label{algline:minim_ask}
\State $X_0\gets 0$
\State $Y_0\gets 0$
\State $B\gets b'$
\State $A\gets a$
\For {$t=1,\dots 2^{\min\{\ell+1,\lfloor\log_2(T^{2/3})\rfloor\}}$}
    \State Select $(B,A)$
    \State Update $X_t=X_{t-1}+\mathbbm{I}(\Vsell_t\le B)$ \label{algline:x_ask}
    \State Update $Y_t=Y_{t-1}+\mathbbm{I}(\Vbuy_t\ge A)$ \label{algline:y_ask}
    
    \If{$Y_t\ge 2^{\min\{\ell+1,\lfloor\log_2(T^{2/3})\rfloor\}} u$}\label{algline:halting_ask} \Comment{\textsc{Halting condition $\mathrm{I}$}}
        \State $A\gets1$
        \State $S^\text{up}(a)\gets 1$\label{algline:place2_ask}
        \State $S^\text{low}(a)\gets \min\{F_\ell^\text{up}(b'),S_\ell^\text{up}(a)\}$\label{algline:place1_ask}
        \Else{}
        \State $\widehat S_{\ell+1}(a)\gets \frac{Y_t}{t}$\label{algline:key2_ask}
        \State $S^\text{up}(a)\gets \widehat S_{\ell+1}(a)+\psi_{\ell+1}$
                
        \State $S^\text{low}(a)\gets \widehat S_{\ell+1}(a)-\psi_{\ell+1}$
    \EndIf
    \If{$X_t\ge 2^{\min\{\ell+1,\lfloor\log_2(T^{2/3})\rfloor\}} u$}\label{algline:halting_two_ask} \Comment{\textsc{Halting condition $\mathrm{II}$}}
        \State $B\gets0$
    \EndIf
\EndFor
\State \textbf{Return} $\min\{S^\text{up}(a),S_{\ell}^\text{up}(a)\},\max\{S^\text{low}(a),S_{\ell}^\text{low}(a)\}$\label{algline:return_ask}
\end{algorithmic}
\end{algorithm}

In this section, we describe the subroutines employed by Algorithm~\ref{alg:twobit}, whose pseudocode is provided in Algorithms~\ref{alg:SAE}--\ref{alg:SAE_ask}. They are used to estimate upper and lower bounds on the c.d.f. $F$ and on the survival function $S$. For simplicity, we focus on Algorithm~\ref{alg:SAE}, namely the subroutine that estimates the c.d.f. $F$. The functioning of Algorithm~\ref{alg:SAE_ask} is entirely analogous.

Specifically, Algorithm~\ref{alg:SAE} is employed at every phase $\ell+1$ to refine the confidence interval on the value of the c.d.f. at a fixed bid $b$, starting from the estimates available at phase $\ell$. First, the algorithm defines the confidence radius
$
\psi_{\ell+1}=\sqrt{\nicefrac{\log(4T/\delta)}{2^{\min\{\ell+1,\lfloor\log_2(T^{2/3})\rfloor\}}}},
$
which is the accuracy at which the algorithm aims to estimate $F(b)$ during the current phase.

The algorithm then chooses an associated ask defined as:
\[
a'\gets \argmax\{S_\ell^\text{up}(a):a\in \Cs_\ell(b)\}.
\]
The reason for this choice is that the ask is not the quantity to be estimated, but it determines how informative the pulled pair is. Among the asks compatible with $b$, the algorithm selects the one with largest upper confidence bound on the survival function, since this is the one that can still provide the largest amount of information. Once $a'$ is fixed, the algorithm sets:
\[
u\gets \min\{F_\ell^\text{up}(b),S_\ell^\text{up}(a')\}+\psi_{\ell+1},
\]
which should be interpreted as the target scale of the estimation problem in the current phase: it is an upper estimate, up to the confidence radius, of the rate at which informative observations can be collected from the pair $(b,a')$.

The algorithm is initialized with $(B,A)=(b,a')$ and with counters $X_0=Y_0=0$. It then selects $(B,A)$ for $2^{\min\{\ell+1,\lfloor\log_2(T^{2/3})\rfloor\}}$ rounds. At round $t$, the counter $X_t$ records how many times the event $\{\Vsell_t\le B\}$ has been observed, while $Y_t$ records how many times the event $\{\Vbuy_t\ge A\}$ has been observed. As long as the bid has not been halted, the empirical frequency
$
\widehat F_{\ell+1}(b)=\nicefrac{X_t}{t}
$
is used as estimate of $F(b)$, and the confidence bounds are updated as:
\[
F^\text{up}(b)=\widehat F_{\ell+1}(b)+\psi_{\ell+1},
\qquad
F^\text{low}(b)=\widehat F_{\ell+1}(b)-\psi_{\ell+1}.
\]
Thus, while the pair remains informative, the algorithm simply uses the empirical frequency together with the phase-dependent confidence radius.
The role of the two halting conditions is to detect when one of the two actions has already produced as much information as can be usefully exploited at the current scale $u$. If
$
X_t\ge 2^{\min\{\ell+1,\lfloor\log_2(T^{2/3})\rfloor\}}u,
$
then the bid is deactivated by setting $B=0$. This prevents further purchases once the counter has reached its target. In this case, the algorithm stops updating the estimate through the empirical mean, sets the upper bound to the trivial value $1$, and keeps as lower bound the quantity
$
\min\{F_\ell^\text{up}(b),S_\ell^\text{up}(a')\},
$
which is the natural conservative value available from the previous phase. Similarly, if
$
Y_t\ge 2^{\min\{\ell+1,\lfloor\log_2(T^{2/3})\rfloor\}}u,
$
the ask is deactivated by setting $A=1$. This prevents further sales once the buyer-side counter has reached the same target.
Finally, the subroutine returns:
\[
\min\{F^\text{up}(b),F_\ell^\text{up}(b)\},
\qquad
\max\{F^\text{low}(b),F_\ell^\text{low}(b)\}.
\]
That is, the newly computed bounds are intersected with those inherited from phase $\ell$. This guarantees that the confidence interval can only improve from one phase to the next, and never becomes looser.

The block length is limited to control inventory also between block endpoints. The halting rules bound the net inventory change at the end of a block, but purchases and sales may accumulate at different rates within it. Since each round changes inventory by at most one, limiting the block length to $2^{\lfloor\log_2(T^{2/3})\rfloor}\le T^{2/3}$ bounds any such excursion by $T^{2/3}$. The same subroutines continue to update the estimates in subsequent phases.

To conclude the section, we notice that, by Equation~\eqref{eq:coset} and the design of Algorithm~\ref{alg:twobit}, the following theorem follows immediately.
\begin{theorem}\label{cor:zerotwo}
    At any $t\in[T]$, Algorithm~\ref{alg:twobit} selects $(B_t,A_t)$ such that $B_t<A_t$. 
\end{theorem}
\begin{proof}
    At any time-step one has either $B_t=0$, $A_t=1$ (in these cases the thesis follows automatically) or:
    $$B_t\in \Cs_\ell(A_t)\subseteq[0,A_t),$$
    which implies $B_t<A_t$. This concludes the proof.
\end{proof}

\subsubsection{Computational Complexity}

In this section, we provide the computational complexity of Algorithm~\ref{alg:twobit}. 

In the following, we refer as $\gamma_t$ to the per-round complexity of choosing $B_t, A_t$ at each time-step $t$. Thus, our goal reduces to bound:
$$\Gamma:=\max_{t=1,\dots T}\gamma_t.$$
While another common objective is the average total running time $T^{-1}\sum_{t=1}^T \gamma_t$, we prefer to quantify the performance with respect to $\Gamma$ for the following reasons:
\begin{itemize}
\item In real-life market making, the agent must react quickly to market changes at every time step, not only on average. In an environment where every microsecond matters, even a single spike in computation time may cause the agent to miss the best opportunities.
\item From a more technical point of view, the average time complexity can be artificially reduced by using the well-known doubling trick \cite{besson2018doubling}. This method consists in repeating the same actions over batches of length $2^n$. $\Gamma$ is indeed  a stronger metric since it does not allow to employ this rather artificial trick. 
\end{itemize}
This is done by means of the following theorem.
\begin{theorem}
\label{thm:per-round-complexity}
The (per-round) maximum computational complexity of
Algorithm~\ref{alg:twobit}, including the calls to
Algorithms~\ref{alg:SAE}--\ref{alg:SAE_ask}, is bounded as:
\[
\Gamma \le \Os\left(|\Bs^\circ||\As^\circ|\right).
\]
In particular, under the discretization
$
|\Bs^\circ|=|\As^\circ|=\lfloor T^{1/3}\rfloor
$, it holds
$
\Gamma \le \Os(T^{2/3}).
$
\end{theorem}

\begin{proof}
We inspect the pseudocode of Algorithms~\ref{alg:twobit}--\ref{alg:SAE_ask}.

First, we observe that the initialization of Algorithm~\ref{alg:twobit} requires setting
$
\Bs_0=\Bs^\circ,\As_0=\As^\circ,
$
and initializing the confidence bounds
$
F_0^{\mathrm{up}}(b),F_0^{\mathrm{low}}(b),
S_0^{\mathrm{up}}(a),S_0^{\mathrm{low}}(a)
$
for all $b\in\Bs^\circ$ and $a\in\As^\circ$. This requires at most linear time in the size of the discretization, namely
$\Os(|\Bs^\circ|+|\As^\circ|).$ Computing the initial compatibility sets costs $\Os(|\Bs^\circ||\As^\circ|)$. 
Hence, the first round satisfies $\gamma_1 \leq \Os(|\Bs^\circ||\As^\circ|)$.

We now consider an arbitrary round $t\ge 2$. There are two possible cases.

\begin{itemize}
 \item \textbf{Round $t$ is an internal round of Algorithm~\ref{alg:SAE} or Algorithm~\ref{alg:SAE_ask}.}
In this case, the pair $(B,A)$ has already been selected, and all quantities needed by the subroutine have already been computed. Therefore, during round $t$, the algorithm only:
\begin{enumerate}
    \item selects the current pair $(B,A)$,
    \item updates the two counters $X_t$ and $Y_t$,
    \item checks the two halting conditions,
    \item possibly updates a constant number of variables.
\end{enumerate}
All these operations take constant time. Hence, in this case,
$\gamma_t\le \Os(1).$

\item \textbf{Round $t$ is a boundary round, i.e., a round in which the Algorithm~\ref{alg:twobit} starts or finalizes a call to one of the subroutines.}
Suppose first that Algorithm~\ref{alg:twobit} is calling Algorithm~\ref{alg:SAE} on some bid $b\in\Bs_\ell$. Before entering the loop of Algorithm~\ref{alg:SAE}, the algorithm has to compute:
\[
a'\in \argmax\{S_\ell^\textnormal{up}(a):a\in \Cs_\ell(b)\}.
\]
In the worst case, and assuming that $\Cs_\ell(b)$ has already been computed, this requires scanning all the asks in the active set $\As_\ell$, and therefore costs at most $\Os(|\As_\ell|)$. The computation of
$
u=\min\{F_\ell^\textnormal{up}(b),S_\ell^\textnormal{up}(a')\}+\psi_{\ell+1}
$
is then constant-time.
After all calls in the phase are completed, the metaprocedure computes the new compatibility sets and eliminates prices with no compatible partner. This requires a constant number of scans of the active bid-ask pairs, and therefore costs at most $\Os(|\Bs_\ell||\As_\ell|)$.
Thus, any round associated with the start or the completion of a bid subroutine call has complexity at most
$
\Os(|\Bs_\ell||\As_\ell|).
$
The same argument applies symmetrically to Algorithm~\ref{alg:SAE_ask}. Indeed, starting or finalizing a call to Algorithm~\ref{alg:SAE_ask} requires computing
$
b'\in \argmax\{F_\ell^\textnormal{up}(b):b\in \Cs_\ell(a)\}
$
or completing the phase, which in the worst case costs $\Os(|\Bs_\ell||\As_\ell|)$.
Therefore, in every boundary round,
$
\gamma_t\le \Os(|\Bs_\ell||\As_\ell|).
$
Finally, since the active sets only decrease across phases, for every $\ell$, $|\Bs_\ell|\le |\Bs^\circ|,
|\As_\ell|\le |\As^\circ|.
$
\end{itemize}
Combining the two cases above yields, for every $t\in[T]$,
$
\gamma_t\le \Os(|\Bs^\circ||\As^\circ|),
$ which in turn implies $\Gamma\leq \Os(|\Bs^\circ||\As^\circ|)$.
Under the discretization used in this section, both sets have cardinality $\lfloor T^{1/3}\rfloor$. Hence,
$
\Gamma\le \Os(T^{2/3})
$, as claimed.
\end{proof}

\subsection{Theoretical Guarantees of Algorithm~\ref{alg:twobit}}
\label{sec:theo_2b}

In this section, we provide the theoretical guarantees attained by Algorithm~\ref{alg:twobit} in terms of cumulative surrogate regret and per-round inventory.

\subsubsection{The Good Event}
\label{sec:clean_2bit}

In this section, we introduce the \emph{good event} that will be used throughout the following theorems.

\begin{definition}[\emph{Good event}]\label{def:bounds_two}
For $\ell=0,\ldots,\lfloor\log_2(T^{2/3})\rfloor$ and $t=1,\ldots,T-2^\ell+1$, let $\psi_\ell=\sqrt{\nicefrac{\log(4T/\delta)}{2^{\ell}}}$ and let:
\[\Delta_F(b,t,\ell)=\left|F(b)-2^{-\ell}\sum_{\tau=t}^{t+2^\ell-1}\mathbbm{I}(\Vsell_\tau\le b)\right|\qquad \Delta_S(a,t,\ell)=\left|S(a)-2^{-\ell}\sum_{\tau=t}^{t+2^\ell-1}\mathbbm{I}(\Vbuy_\tau\ge a)\right|.\]
The \emph{good event} is defined as:
\begin{equation}
    \goodevent :=
\bigcap_{\ell=0}^{\lfloor\log_2(T^{2/3})\rfloor}\ \bigcap_{t=1}^{T-2^\ell+1}
\left[\bigcap_{b\in\Bs^\circ}\{\Delta_F(b,t,\ell)<\psi_\ell\}
\cap\bigcap_{a\in\As^\circ}\{\Delta_S(a,t,\ell)<\psi_\ell\}\right].
\label{eq:good_two}
\end{equation}
\end{definition}
Intuitively, the good event states that, for every bid--ask pair $(b,a)$, any contiguous block of length $2^\ell$ contained in the horizon yields an empirical mean of $\mathbbm{I}(\Vsell_\tau\le b)$ and $\mathbbm{I}(\Vbuy_\tau\ge a)$ whose discrepancy from the corresponding true quantities $F(b)$ and $S(a)$ is strictly less than $\psi_\ell$. Since all possible starting times are included, the event also covers every subsequent block of the capped length. The following proposition shows that this good event holds with high probability.
\begin{restatable}{proposition}{gooddeltatwo}\label{thm:gooddeltatwo}
    Let $\delta\in(0,1)$ and $\goodevent$ be the good event defined in Equation~\eqref{eq:good_two}, it holds $\Prob(\goodevent)\ge 1-\delta$.
\end{restatable}

Proposition~\ref{thm:gooddeltatwo} is a key ingredient of the analysis, since from this point on we will work under the assumption that the good event holds. Its proof is deferred to the appendix. We next show that, on this event, the upper and lower bounds returned by Algorithms~\ref{alg:SAE}--\ref{alg:SAE_ask} are always valid.

\begin{proposition}\label{thm:boundstwo}
On $\goodevent$, at every completed phase $\ell$, the active prices satisfy:
\[
F_\ell^{\mathrm{low}}(b)\le F(b)\le F_\ell^{\mathrm{up}}(b),
\qquad
S_\ell^{\mathrm{low}}(a)\le S(a)\le S_\ell^{\mathrm{up}}(a).
\]
\end{proposition}
\begin{proof}
    We focus on the bid case, as the ask case is entirely analogous. The proof proceeds by induction on $\ell$; the initial intervals $[0,1]$ are valid. Assuming that the claim holds at a completed phase $\ell$, we show that it also holds at phase $\ell+1$. Write $n=2^{\min\{\ell+1,\lfloor\log_2(T^{2/3})\rfloor\}}$ for the block length.
    
    In case that the halting condition at phase $\ell+1$ is not triggered at Line~\ref{algline:halting}, the bid is pulled $n$ times.  Thus,
    $$F^\text{low}(b)= \widehat F_{\ell+1}(b)-\psi_{\ell+1}\le F(b)\le \widehat F_{\ell+1}(b)+\psi_{\ell+1}=F^\text{up}(b)$$
    holds by definition of $\goodevent$. Since, by Line~\ref{algline:return}, it holds:
    \begin{align*}
        F_{\ell+1}^\text{up}(b)\gets &\min\{F^\text{up}(b),F_{\ell}^\text{up}(b)\},\\
        F_{\ell+1}^\text{low}(b)\gets&\max\{F^\text{low}(b),F_{\ell}^\text{low}(b)\},
    \end{align*}
    the induction ensures that the lower and upper bounds are still valid at $\ell+1$.
    
    Consider the opposite case, that is, the halting condition is met. In this case, trivially,
    $$F(b)\le 1=F^\text{up}(b).$$
    On the opposite side, numbering the completed block's rounds from $1$ to $n$ and counting acceptances of the fixed bid $b$ also after deactivation, the halting condition and the good event give:
    \begin{align*}
        F^\text{low}(b)&= \min\{F_\ell^\text{up}(b),S_\ell^\text{up}(a')\}\\
        &=u-\psi_{\ell+1}\\
        &\le \frac1n\sum_{t=1}^{n}\mathbbm{I}(\Vsell_t\le b)-\psi_{\ell+1}\\
        &< F(b).
    \end{align*}
    Similarly to the previous case, noticing that:
    \begin{align*}
        F_{\ell+1}^\text{up}(b)\gets &\min\{F^\text{up}(b),F_{\ell}^\text{up}(b)\}\\
        F_{\ell+1}^\text{low}(b)\gets&\max\{F^\text{low}(b),F_{\ell}^\text{low}(b)\},
    \end{align*}
    concludes the inductive step and the proof.
\end{proof}

We notice that the previous result easily implies that a grid-optimal pair $(b^\circ,a^\circ)$ with $b^\circ<a^\circ$ survives every completed phase. In particular, the compatibility sets employed by the algorithm are nonempty.
We now provide two fundamental results concerning the monotonicity of the (optimistic and pessimistic) estimates of the surrogate objective function.
\begin{proposition}\label{prop:monothonicity}
    Let $b\in \Bs_\ell$ and $a\in \As_\ell$. Then, under the good event~\eqref{eq:good_two}, for every completed phase $r<\ell$, it holds:
    $$J_r^\textnormal{low}(b,a)\le J_\ell^\textnormal{low}(b,a),\qquad J_r^\textnormal{up}(b,a)\ge J_\ell^\textnormal{up}(b,a).$$
    Thus, $J_r^\textnormal{low*}\le J_\ell^\textnormal{low*}$.
\end{proposition}
\begin{proof}
    By Line~\ref{algline:return} of Algorithm~\ref{alg:SAE}, it holds:
    \begin{align*}
        J_r^\text{low}(b,a)&=(a-b)\min\{F_r^\text{low}(b),S_r^\text{low}(a)\}\\
        &\le (a-b)\min\{F_\ell^\text{low}(b),S_\ell^\text{low}(a)\}\\
        &=J_\ell^\text{low}(b,a).
    \end{align*}
    The second point on $J_r^\text{up}(b,a)$ follows analogously.
\end{proof}
Proposition~\ref{prop:monothonicity} shows that the confidence interval associated with the surrogate objective becomes tighter across phases: lower bounds can only increase, while upper bounds can only decrease.

\begin{proposition}\label{prop:Cprop}
    Let $b\in \Bs_\ell$. Under the good event \eqref{eq:good_two}, the following statements hold:
    \begin{enumerate}
        \item For every $r\le \ell$, $\Cs_\ell(b)\subseteq \Cs_r(b)$,
        \item For every $a\in \Cs_\ell(b)$, $b\in \Cs_\ell(a)$.
    \end{enumerate}
\end{proposition}
\begin{proof}
    By Proposition~\ref{prop:monothonicity}, it holds:
    \begin{align*}
        \Cs_\ell(b)&=\left\{a\in \As_\ell:J_\ell^\text{up}(b,a)\ge J_\ell^\text{low*}\right\}\\
        &\subseteq \left\{a\in \As_\ell:J_\ell^\text{up}(b,a)\ge J_r^\text{low*}\right\}\\
        &\subseteq \left\{a\in \As_\ell:J_r^\text{up}(b,a)\ge J_r^\text{low*}\right\}\\
        &\subseteq \Cs_r(b).
    \end{align*}
    The second point follows from the symmetry between $b,a$ in the definition of $\Cs_\ell(b)$.
\end{proof}

Proposition~\ref{prop:Cprop} shows that compatibility sets shrink monotonically across phases, and that compatibility is symmetric among active bid-asks.
The previous results are crucial for proving the next theorem, which establishes the main structural property of the algorithm.
\begin{proposition}\label{thm:mutual}
    Under the good event \eqref{eq:good_two}, if $b$ triggers the halting condition (Line~\ref{algline:halting}) in phase $\ell+1$, then:
    $$\forall a\in \Cs_\ell(b),\qquad S(a)< F(b).$$
    Symmetrically, if the target ask $a$ is halted in a completed call, then $F(b)<S(a)$ for every $b\in\Cs_\ell(a)$.
\end{proposition}
\begin{proof}
    Let $a\in\Cs_\ell(b)$ and let $n$ be the length of the bid call. The following inequalities hold:
    \begin{align*}
        F(b)&>\frac{X_n}{n}-\psi_{\ell+1}\\
        &\ge u-\psi_{\ell+1}\\
        &=S_\ell^\text{up}(a')\\
        &\ge S_\ell^\text{up}(a)\\
        &\ge S(a).
    \end{align*}
    The first inequality follows from the good event applied to $b$, since $X_n$ is at most the number of acceptances of the fixed bid over the full block. The second inequality follows from the halting condition at Line~\ref{algline:halting}. For the equality, recall the definition of $u$ at Line~\ref{algline:minim}. Since $F_\ell^\text{up}(b)\ge F(b)>u-\psi_{\ell+1}$, the minimum in the definition of $u$ must be attained at $S_\ell^\text{up}(a')$. The fourth inequality follows from the definition of $a'$ at Line~\ref{algline:sao}. The last inequality follows once again from the good event. The ask case follows by reversing the roles of the two sides. This concludes the proof.
\end{proof}

Proposition~\ref{thm:mutual} is particularly important, since it guarantees that, in a completed call, the halting condition can only be triggered for bids $b$ such that $F(b)$ is larger than the $S(a)$ of any ask with which they can be paired. Therefore, even if $F(b)$ is not estimated accurately after halting, this is not problematic for the agent, since:
\[
J(b,a)=(a-b)\min\{F(b),S(a)\}=(a-b)S(a).
\]
In other words, once the halting condition is triggered, the surrogate reward is already determined by the ask side. In the next section, we combine these results to prove the regret and the per-round imbalance.

\subsubsection{Final Regret and Per-Round Inventory Bounds}
\label{sec:guarantee_2bit}

In this section, we establish the regret and per-round inventory bounds of Algorithm~\ref{alg:twobit}. 

We start by the following preliminary result.
\begin{proposition}\label{thm:diff_J}
    Fix a phase $\ell$ and let $b\in \Bs_\ell$. Under the good event, it holds:
    $$\forall a\in \Cs_\ell(b)\qquad J_\ell^\text{up}(b,a)-2\psi_\ell\le J(b,a)\le J_\ell^\text{low}(b,a)+2\psi_\ell.$$
\end{proposition}
\begin{proof}
    Let $a\in \Cs_\ell(b)$.
    We split the proof in three parts, noticing that, by Proposition~\ref{thm:mutual}, \emph{at most one} between $b,a$ has triggered the halting condition at the previous step.
    \begin{enumerate}
        \item If neither $b$ nor $a$ triggered the halting condition, by the good event, it holds:
        \begin{align*}
            J_\ell^\text{up}(b,a)-2\psi_\ell &= (a-b)\min\{F_{\ell}^\text{up}(b),S_{\ell}^\text{up}(a)\}-2\psi_\ell\\
            &\le (a-b)\min\{F(b)+2\psi_\ell,S(a)+2\psi_\ell\}-2\psi_\ell\\
            &\le J(b,a)\\
            &\le (a-b)\min\{F(b)-2\psi_\ell,S(a)-2\psi_\ell\}+2\psi_\ell\\
            &\le J_\ell^\text{low}(b,a)+2\psi_\ell.
        \end{align*}
        \item If $b$ triggered the halting condition, by Proposition~\ref{thm:mutual}, $S(a)< F(b)$. Thus, $a$ has not triggered the halting condition. Therefore, under the good event, it holds:
        \begin{align*}
            J_\ell^\text{up}(b,a)-2\psi_\ell &= (a-b)\min\{F_{\ell}^\text{up}(b),S_{\ell}^\text{up}(a)\}-2\psi_\ell\\
            &\le (a-b)S_{\ell}^\text{up}(a)-2\psi_\ell\\
            &\le J(b,a)\\
            &\le (a-b)\min\{F(b)-2\psi_\ell,S(a)-2\psi_\ell\}+2\psi_\ell\\
            &= (a-b)(S(a)-2\psi_\ell)+2\psi_\ell\\
            &\le (a-b)S_\ell^\text{low}(a)+2\psi_\ell\\
            &\le J_\ell^\text{low}(b,a)+2\psi_\ell.
        \end{align*}
        Here, the first inequality follows from the good event, and the final inequality follows from the fact that, by Line~\ref{algline:place1}, it holds $S_\ell^\text{low}(a)\le F_\ell^\text{low}(b)$, under the good event.
        \item If $a$ triggered the halting condition, the same argument as in the previous point applies, with the roles of bids and asks reversed.
    \end{enumerate}
    The same cases also give $\min\{F_\ell^\text{up}(b),S_\ell^\text{up}(a)\}\le\min\{F(b),S(a)\}+2\psi_\ell$, since a side attaining the latter minimum has not been halted. This also holds at $\ell=0$. 
    This concludes the proof.
\end{proof}

Proposition~\ref{thm:diff_J} shows that, at any time step, the agent has accurate upper and lower bounds for every active bid-ask pair. Crucially, this result relies on Proposition~\ref{thm:mutual}. Indeed, by that result, whichever of $F(b)$ and $S(a)$ attains the minimum cannot have triggered the halting condition in that phase. 

We are now ready to prove an upper bound for the imbalance generated at phase $\ell$. In the following, $n=2^{\min\{\ell,\lfloor\log_2(T^{2/3})\rfloor\}}$ denotes the number of rounds assigned to each active price in phase $\ell\ge1$.
\begin{proposition}\label{thm:twobitimbalance}
    Let $N_\ell=|\Bs_{\ell-1}|+|\As_{\ell-1}|$ be the total number of bids and asks active at the start of phase $\ell$.
    Under $\goodevent$, each completed call generates an absolute inventory change of at most $5n\psi_{\ell-1}$. In particular, the absolute change over a completed phase does not exceed $5N_\ell n\psi_{\ell-1}$.
\end{proposition}
\begin{proof}
    In order to bound the total imbalance, we bound the imbalance generated by any given bid/ask, and then we multiply it by $N_\ell$.
    Let $b\in\Bs_{\ell-1}$ and let $a\in\Cs_{\ell-1}(b)$ be the ask such that bid $b$ is paired with in Line~\ref{algline:sao} of Algorithm~\ref{alg:SAE}. By validity of the confidence bounds and the argument in the proof of Proposition~\ref{thm:diff_J}, it holds:
    \[
    \min\{F(b),S(a)\}\le u\le\min\{F(b),S(a)\}+2\psi_{\ell-1}+\psi_\ell.
    \]
    Each side either reaches $\lceil nu\rceil$ trades and is halted, or its final counter equals its empirical count over all $n$ rounds. Thus, by the halting conditions and the good event, it holds:
    \[
    n\bigl(\min\{F(b),S(a)\}-\psi_\ell\bigr)
    \le X_n,Y_n
    \le n\bigl(\min\{F(b),S(a)\}+3\psi_{\ell-1}\bigr)+1.
    \]
    Consequently, since $\psi_\ell\le\psi_{\ell-1}$ and $n\psi_{\ell-1}\ge1$,
    \begin{align*}
        |Y_n-X_n|&\le n(3\psi_{\ell-1}+\psi_\ell)+1\\
        &\le4n\psi_{\ell-1}+1\\
        &\le5n\psi_{\ell-1}.
    \end{align*}
    The ask calls are analogous. Summing over the $N_\ell$ calls concludes the proof.
\end{proof}

We provide the upper bound to the virtual regret associated to phase $\ell$, which is defined as follows.
\begin{definition}[Virtual regret at phase $\ell$]\label{def:virtualregret_ask}
    Let $r_t$ be the reward associated to round $t\in[N_\ell n]$ of a completed phase $\ell$. The virtual regret at phase $\ell$ is defined as:
    \[
    R^{\circ,\ell}:=N_\ell n\cdot J(b^\circ,a^\circ)-\sum_{t=1}^{N_\ell n}r_t.
    \]
    For an individual completed call, the same definition uses its $n$ rounds.
\end{definition}
This is done in the following proposition.
\begin{proposition}\label{thm:tworegret}
    Under $\goodevent$, the total virtual regret paid in a completed phase $\ell$ is bounded as:
    \[
    R^{\circ,\ell}\le9N_\ell n\psi_{\ell-1}.
    \]
    Each completed call has virtual regret at most $9n\psi_{\ell-1}$.
\end{proposition}
\begin{proof}
    In order to bound the total virtual regret at phase $\ell$, we bound the virtual regret at phase $\ell$ generated by any given bid/ask, and then we multiply it by $N_\ell$.
    Let $b\in\Bs_{\ell-1}$ and let $a\in\Cs_{\ell-1}(b)$ be the ask such that $b$ is paired with in Line~\ref{algline:sao} of Algorithm~\ref{alg:SAE}. By definition of $\Cs_{\ell-1}$ (see Equation~\eqref{eq:coset}) it holds:
    \[
    J_{\ell-1}^\text{up}(b,a)\ge J_{\ell-1}^\text{low*}.
    \]
    Moreover, by Proposition~\ref{thm:diff_J}, it holds:
    \begin{align}
        J(b,a)&\ge J_{\ell-1}^\text{up}(b,a)-2\psi_{\ell-1}\nonumber\\*
        &\ge J_{\ell-1}^\text{low*}-2\psi_{\ell-1}\nonumber\\*
        &\ge J_{\ell-1}^\text{low}(b^\circ,a^\circ)-2\psi_{\ell-1}\nonumber\\*
        &\ge J(b^\circ,a^\circ)-4\psi_{\ell-1}.\label{eq:Jlo}
    \end{align}
    Now, we need to show that the total realized reward of Algorithm~\ref{alg:SAE} applied to $b,a$ is close to $nJ(b,a)$. By definition, this reward corresponds to $aY_n-bX_n$, for the counters defined in Lines~\ref{algline:x} and~\ref{algline:y}. Using the count bounds from the proof of Proposition~\ref{thm:twobitimbalance}, it holds:
    \begin{align*}
        aY_n-bX_n
        &\ge an\bigl(\min\{F(b),S(a)\}-\psi_\ell\bigr)\\
        &\qquad-b\bigl(n(\min\{F(b),S(a)\}+3\psi_{\ell-1})+1\bigr)\\
        &=nJ(b,a)-na\psi_\ell-3nb\psi_{\ell-1}-b\\
        &\ge nJ(b,a)-4n\psi_{\ell-1}-1.
    \end{align*}
    Together with Equation~\eqref{eq:Jlo}, this shows that the corresponding regret does not exceed $8n\psi_{\ell-1}+1\le9n\psi_{\ell-1}$. The ask calls are analogous. Summing over the $N_\ell$ calls concludes the proof.
\end{proof}

We are now ready to provide the final theoretical guarantees of Algorithm~\ref{alg:twobit}.
\begin{theorem}\label{thm:regret_two_bit}
    Let $\delta\in(0,1)$. Then, Algorithm~\ref{alg:twobit} enjoys, with probability at least $1-\delta$, simultaneously over all the $t$ time-steps, the following guarantees:
    \[
    |I_t|\le59T^{2/3}\sqrt{\log(4T/\delta)},\qquad R_t^J\le109T^{2/3}\sqrt{\log(4T/\delta)}.
    \]
\end{theorem}
\begin{proof}
    Work on $\goodevent$ and fix $t\le T$. By Equation~\eqref{eq:abcirctwo} and Proposition~\ref{prop:unigrid}, the regret satisfies:
    \[
    R_t^J\le R_t^\circ+\frac{2T}{\lfloor T^{1/3}\rfloor}.
    \]
    Indeed, removing bid $1$ and ask $0$ from the full grid does not change its maximum, since such pairs have nonpositive value and $(0,1)$ remains available.

    By Proposition~\ref{thm:tworegret}, the virtual regret accumulated in each completed call of phase $\ell$ does not exceed $9n\psi_{\ell-1}$. Moreover, it holds:
    \begin{enumerate}
        \item $N_\ell$ does not exceed $|\Bs^\circ|+|\As^\circ|=2\lfloor T^{1/3}\rfloor$;
        \item the total length of the calls completed by time $t$ does not exceed $T$.
    \end{enumerate}
    Before the block length reaches $2^{\lfloor\log_2(T^{2/3})\rfloor}$, we have $n\psi_{\ell-1}=\sqrt{2^{\ell+1}\log(4T/\delta)}$. From the first phase reaching this length onwards, $\psi_{\ell-1}\le\sqrt{2\log(4T/\delta)/2^{\lfloor\log_2(T^{2/3})\rfloor}}$. Thus, summing $n\psi_{\ell-1}$ over all completed calls, including those in an unfinished phase, gives at most:
    \begin{align*}
        &2\lfloor T^{1/3}\rfloor\sum_{\ell=1}^{\lfloor\log_2(T^{2/3})\rfloor-1}\sqrt{2^{\ell+1}\log(4T/\delta)}
        +T\sqrt{\frac{2\log(4T/\delta)}{2^{\lfloor\log_2(T^{2/3})\rfloor}}}\\
        &\qquad\le(6+2\sqrt2)T^{2/3}\sqrt{\log(4T/\delta)}.
    \end{align*}
    \needspace{3\baselineskip}
    Here the sum is empty for $T\le7$; the inequality follows from the geometric sum and $T^{2/3}/2<2^{\lfloor\log_2(T^{2/3})\rfloor}\le T^{2/3}$.

    The possible unfinished call has length at most $T^{2/3}$ and contributes at most $2T^{2/3}$ to the virtual regret, since $J(b^\circ,a^\circ)\le1$ and $r_t\ge-1$. Combining everything, it holds:
    \begin{align*}
        R_t^J&\le9(6+2\sqrt2)T^{2/3}\sqrt{\log(4T/\delta)}
        +2T^{2/3}+\frac{2T}{\lfloor T^{1/3}\rfloor}\\
        &\le109T^{2/3}\sqrt{\log(4T/\delta)},
    \end{align*}
    where we used $\lfloor T^{1/3}\rfloor\ge T^{1/3}/2$ and $\sqrt{\log(4T/\delta)}>1$.

    For what concerns the imbalance, Proposition~\ref{thm:twobitimbalance} gives the same form of bound, except for the constant $5$ in place of $9$. Each round changes inventory by at most one, so the unfinished call contributes at most $T^{2/3}$. Therefore, applying the previous reasoning, it holds:
    \begin{align*}
        |I_t|&\le5(6+2\sqrt2)T^{2/3}\sqrt{\log(4T/\delta)}+T^{2/3}\\
        &\le59T^{2/3}\sqrt{\log(4T/\delta)}.
    \end{align*}
    The argument holds for every $t\le T$ on the same good event, whose probability is at least $1-\delta$. This concludes the proof.
\end{proof}

\section{LOB Access}
\label{sec:LOB}
In this section, we focus on the case the market maker has access to the \emph{limit order book}. The results provided in this section are provided in the following order:
\begin{itemize}
    \item In Section~\ref{sec:conf_bound_LOB}, we study how to proper estimate the cumulative density function of buyers' and sellers' valuations. Specifically, we first focus on the case full-feedback is available, defining two different kind of concentrations which will lead to different theoretical guarantees attained by our algorithm. Then, we define the good event associated to the aforementioned concentrations and we show how to extend those results to the LOB feedback case. 
    \item In Section~\ref{sec:algorithm_LOB}, we describe our algorithm for the LOB case and then we study the associated computational complexity. 
    \item In Section~\ref{sec:regret_LOB}, we provide the theoretical guarantees attained by our algorithm in terms of surrogate regret and per-round inventory constraints. We then provide an instance dependent bound which improves the aforementioned result in favorable instances.
    \item In Section~\ref{sec:lower_LOB}, we provide the lower bound associated to our setting.
\end{itemize}

\subsection{Estimating the CDFs and The Good Event}
\label{sec:conf_bound_LOB}

We first provide results and intuition on how to estimate the CDFs of buyers' and sellers' valuations.

\subsubsection{Estimating the CDFs under Full Feedback}

For the purposes of this section, we define the estimators $\widehat F_t$ and $\widehat S_t$ that would be used in case the full-feedback was available. The estimators are given by:
\begin{equation}\label{eq:ecdf}
\widehat F_t(b)=\frac{1}{t}\sum_{\tau=1}^t \mathbbm{I}(\Vsell_\tau\le b), \qquad
\widehat S_t(a)=\frac{1}{t}\sum_{\tau=1}^t \mathbbm{I}(\Vbuy_\tau\ge a).
\end{equation}
These estimators enjoy strong concentration properties around the true c.d.f. of the bids $F$ and the survival function of the asks $S$, which are provided in the following foundational result.

\begin{theorem}[DKW inequality \cite{dvoretzky1956asymptotic}, \cite{massart1990tight} for the tighter version] 
    \label{thm:ecdf}
    The following inequality holds:
    $$\Prob\left(\sup_{b\in [0,1]}|\widehat F_t(b)-F(b)|>  \varepsilon\right)\le 2e^{-2t\varepsilon^2}.$$
\end{theorem}

The former result implies, employing a union bound over all the time-steps of the problem, that with probability at least $1-\delta$, it holds, for all $b\in[0,1], t\in[T]$:
\begin{equation} \widehat F_t(b)-\psi_t(\delta)\le F(b)\le \widehat F_t(b)+\psi_t(\delta) \quad \text{  for  } \quad  \psi_t(\delta):=\min\left\{1,\sqrt{\frac{\log(4T/\delta)}{2t}}\right\}.\label{eq:hoeffding1}\end{equation}
This bound holds uniformly over the space of bids---an analogous result also holds for $\widehat S_t$ over the space of asks.\footnote{A similar reasoning holds for all the concentration results provided in this section, which, for the ease of presentation, are described for the bids only.} 
Although this bound is extremely tight---in the sense that it cannot be improved even at the level of constants---the resulting $\psi$ term in Equation~\eqref{eq:hoeffding1} does not depend on $b$ and therefore cannot be used to obtain instance-dependent bounds.
A confidence bound that depends on $b$ can instead be derived using Bernstein's inequality.
\begin{theorem}[Bernstein's inequality \cite{boucheron2003concentration}]\label{thm:bernstein}
    Let $(X_t)_{t=1}^n$ be a sequence of independent zero-mean random variables bounded in $[-1,1]$. Let $\sigma^2:=\sum_{t=1}^n\text{Var}(X_t)$. Then, with probability at least $1-\delta$:
    $$\left|\sum_{t=1}^nX_t \right|\le \sqrt{2\sigma^2\log(2/\delta)}+\frac{2}{3}\log(2/\delta).$$
\end{theorem}
In our setting, $\mathbbm{I}(\Vsell_t\le b)$ is a Bernoulli random variable with mean $F(b)$. Therefore, $\mathbbm{I}(\Vsell_t\le b)-F(b)$ is zero-mean and bounded in $[-1,1]$. Due to the independence of the samples, the total variance writes as $\sigma^2=\sum_{t=1}^n\text{Var}(X_t)=nF(b)(1-F(b))$. Thus, an application of Theorem~\ref{thm:bernstein} guarantees, for a specific $b\in [0,1]$ and time-step $t\in [T]$, with probability at least $1-\delta$:
\begin{equation}\widehat F_t(b)-\psi_t(b,\delta)\le F(b)\le \widehat F_t(b)+\psi_t(b,\delta),\ \psi_t(b,\delta):=\sqrt{\frac{2F(b)(1-F(b))\log(2/\delta)}{t}}+\frac{2\log(2/\delta)}{3t}.\label{eq:bernstein1}\end{equation}
The former result holds only for a fixed pair of bid $b$ and time-step $t$. Therefore, to obtain a bound that holds uniformly over all time steps, one must pay an additional $\log(T)$ factor due to a union bound, as well as introduce a covering of the bid space.
The main limitation of this bound is that the resulting confidence interval depends on a quantity that is unknown to the learner. This issue is overcome by the following empirical version (see \cite{audibert2009exploration,maurer2009empirical}).
\begin{theorem}[Empirical Bernstein's inequality for Bernoulli r.v.s~{\cite[Theorem~1]{audibert2009exploration}}]\label{thm:bernstein2}
    Let $(X_t)_{t=1}^n$ be a sequence of independent Bernoulli random variables of parameter $p\in [0,1]$. Let $\widehat p_n:=n^{-1}\sum_{t=1}^nX_t$. Then, with probability at least $1-\delta$, it holds:
    \[
    |\widehat p_n-p|
    \le
    \sqrt{\frac{2\widehat p_n(1-\widehat p_n)\log(3/\delta)}{n}}
    +
    \frac{3\log(3/\delta)}{n}.
    \]
\end{theorem}
Applying Theorem~\ref{thm:bernstein2}, we get a confidence bound similar to the one of Equation~\eqref{eq:bernstein1}, where $\psi$ is:
\begin{equation}
    \psi_t(b,\delta):=\sqrt{\frac{2\widehat F_t(b)\left(1-\widehat F_t(b)\right)\log(3/\delta)}{t}}+\frac{3\log(3/\delta)}{t}.
    \label{eq:bernstein2}
\end{equation}
As before, in contrast with the bound in Equation~\eqref{eq:hoeffding1}, this result holds only for a fixed pair $(b,t)$.

\subsubsection{The Good Event}
\label{sec:good_event_LOB}
We introduce the \emph{good event} that will be used throughout the following theorems. It is defined as the event in which the confidence bounds on both $F$ and $S$ hold simultaneously for all time steps.
\begin{definition}[\emph{Good event}]\label{def:bounds_goodevent}
    Let $\Bs^\circ \subset [0,1], \As^\circ\subset [0,1]$ be finite subsets of the bid space and the ask space, respectively.
    Let $\{\psi_t(b,\delta)\}_{t=1}^T, \{\psi_t(a,\delta)\}_{t=1}^T$ be sequences of functions that depend only on quantities generated by the end of round $t$. We define the \textit{good event} $\goodevent$ as the following set:
    \begin{equation}
        \goodevent :=
        \bigcap_{t=1}^T\left[
        \bigcap_{b\in\Bs^\circ}\left\{|F(b)-\widehat F_t(b)|\le \psi_t(b,\delta)\right\}
        \cap\bigcap_{a\in\As^\circ}\left\{|S(a)-\widehat S_t(a)|\le \psi_t(a,\delta)\right\}\right].
        \label{eq:good}
    \end{equation}
\end{definition}
{\emergencystretch=3em As already proved in the previous section for Equation~\eqref{eq:hoeffding1}, when $\psi_t(b,\delta)=\psi_t(\delta):=\min\left\{1,\sqrt{\nicefrac{\log(4T/\delta)}{2t}}\right\}$, the good event has a probability of at least $1-\delta$.\par}

Similarly, we can prove the following result.
\begin{proposition}\label{prop:goodevent1}
    Let $\Bs^\circ$ and $\As^\circ$ be finite subsets of $[0,1]$. Let:
    $$\forall b\in \Bs^\circ,\qquad \psi_t(b,\delta):=\sqrt{\frac{2\widehat F_t(b)\left(1-\widehat F_t(b)\right)\log(3|\Bs^\circ|T/\delta)}{t}}+\frac{3\log(3|\Bs^\circ|T/\delta)}{t},$$
    and:
    $$\forall a\in \As^\circ,\qquad \psi_t(a,\delta):=\sqrt{\frac{2\widehat S_t(a)\left(1-\widehat S_t(a)\right)\log(3|\As^\circ|T/\delta)}{t}}+\frac{3\log(3|\As^\circ|T/\delta)}{t}.$$
    Then, $\Prob(\goodevent)\ge 1-2\delta$.
\end{proposition}
This shows that both the DKW-type and the empirical-Bernstein-type confidence intervals are valid. The latter, being defined in terms of an empirical quantity, is harder to quantify. Fortunately, the next result shows that, under the good event, the empirical-Bernstein confidence interval is not significantly larger than the one obtained from Bernstein's inequality in Theorem~\ref{thm:bernstein}.
\begin{theorem}\label{thm:pontil}
    Let $\psi_t(b,\delta)$ be the radius in Proposition~\ref{prop:goodevent1}. Under the good event, for every $t\in[T]$ and $b\in\Bs^\circ$,
    $$\psi_t(b,\delta)\le 2\sqrt{\frac{2\log(3|\Bs^\circ|T/\delta)F(b)\left(1-F(b)\right)}{t}}+\frac{8\log(3|\Bs^\circ|T/\delta)}{t}.$$
\end{theorem}
\begin{proof}
    Let $C_1=\sqrt{2\log(3|\Bs^\circ|T/\delta)}$ and $C_2=3\log(3|\Bs^\circ|T/\delta)$, so that the $\psi$ function writes as:
    $$\psi_t(b,\delta)=C_1\sqrt{\frac{\widehat F_t(b)\left(1-\widehat F_t(b)\right)}{t}}+\frac{C_2}{t}.$$

    Let us call $\Delta F:=\widehat F_t(b)-F(b)$. The former bound writes as follows,
    \begin{align*}
        \psi_t(b,\delta)&=C_1\sqrt{\frac{\widehat F_t(b)\left(1-\widehat F_t(b)\right)}{t}}+\frac{C_2}{t}\\
        &=C_1\sqrt{\frac{(F(b)+\Delta F)\left(1-\Delta F-F(b)\right)}{t}}+\frac{C_2}{t}\\
        &=C_1\sqrt{\frac{F(b)\left(1-F(b)\right)+\Delta F\left(1-2F(b)\right)-\Delta F^2}{t}}+\frac{C_2}{t}\\
        &\le C_1\sqrt{\frac{F(b)\left(1-F(b)\right)}{t}}+C_1\sqrt{\frac{|\Delta F|}{t}}+\frac{C_2}{t}.
    \end{align*}
    Under the good event, we have $|\Delta F|\le \psi_t(b,\delta)$, which implies:
    $$\psi_t(b,\delta) \le C_1\sqrt{\frac{F(b)\left(1-F(b)\right)}{t}}+C_1\sqrt{\frac{\psi_t(b,\delta)}{t}}+\frac{C_2}{t}.$$
    Applying $C_1\sqrt{\psi_t(b,\delta)/t}\le\psi_t(b,\delta)/2+C_1^2/(2t)$ gives:
    \[
    \psi_t(b,\delta)\le2C_1\sqrt{\frac{F(b)(1-F(b))}{t}}+\frac{C_1^2+2C_2}{t},
    \]
    which completes the proof for every $t\ge1$.
\end{proof}

An analogous bound holds on the ask side by simply switching $b\gets a$, $F\gets S$ in the former proof.

\subsubsection{Estimating the CDFs under Truncated Feedback}
\label{sec:est_CDF_LOB}

In the previous section, we showed that, when full feedback on the per-round valuations is available, the estimated cumulative distribution functions converge rapidly to the true ones, and the estimation error can be upper-bounded by both DKW-type and Bernstein-type confidence intervals.

Unfortunately, under the LOB feedback model, $\widehat F_t$ cannot always be built, since at round $t\in[T]$ the agent only observes $\text{clip}(V_t^\text{sell}; B_t)$---notice that an analogous limitation holds on the ask side. Thus, if the spread $A_t-B_t$ is too narrow, the agent collects no information that can be used to improve the estimates of $\widehat F_t$ and $\widehat S_t$. Conversely, whenever information is revealed---namely, when $B_t < V_t^\text{sell}$ or $V_t^\text{buy} < A_t$---no trade takes place on the corresponding side, and hence the agent receives no reward. This highlights the fundamental trade-off induced by the LOB feedback structure.

The proposed solution to this challenge is to employ an \textit{elimination}-based approach, where the pair $B_t,A_t$ gets narrowed down monotonically over time. This is equivalent to defining two sets $\Bs_t=\cap_{\tau=1}^t[B_\tau,1],\As_t=\cap_{\tau=1}^t[0,A_\tau]$ of active bids and asks, which shrink during the learning process. 
In this way, for any $b\in \Bs_t$, we get:
\begin{align*}
    \widehat F_t(b)&\coloneq \frac{\sum_{\tau=1}^t \mathbbm{I}(V_\tau^\text{sell}\le b)}{t}\\
    &= \frac{\sum_{\tau=1}^t \mathbbm{I}(V_\tau^\text{sell}\le b, B_\tau \le b)}{t}\\
    &=\frac{\sum_{\tau=1}^t \mathbbm{I}(\max\{V_\tau^\text{sell},B_\tau\}\le b)}{t}=:\frac{1}{t}\sum_{\tau=1}^t\chi_\tau^\text{sell}(b),
\end{align*}
where the second equality incorporates the additional event $B_\tau \le b$, which holds almost surely due to the fact that $b \in \Bs_t$. Importantly, the final quantity, $\chi_\tau^\text{sell}(b)$, is observable by the agent, and can therefore be used to estimate the CDF over all active arms. A completely analogous argument applies to the ask side. We formalize this observation in the following proposition.
\begin{proposition}\label{prop:ecdfvar}
Define, for $b,a\in [0,1]$ and $t\in [T]$,
\begin{equation}
    \chi_\tau^\textnormal{sell}(b):=\mathbbm{I}(\max\{V_\tau^\textnormal{sell},B_\tau\}\le b)\qquad \chi_\tau^\textnormal{buy}(a):=\mathbbm{I}(\min\{V_\tau^\textnormal{buy},A_\tau\}\ge a)\label{eq:pseudo-indicator}.
\end{equation}
Then, if $b\in \Bs_t$, it holds:
\begin{equation}
    \frac{1}{t}\sum_{\tau=1}^t\chi_\tau^\textnormal{sell}(b) = \widehat F_t(b).\label{eq:pseudobid}
\end{equation}
and, if $a\in \As_t$, it holds:
\begin{equation}
    \frac{1}{t}\sum_{\tau=1}^t\chi_\tau^\textnormal{buy}(a) = \widehat S_t(a).\label{eq:pseudoask}
\end{equation}
\end{proposition}
The meaning of the former proposition is what follows. At any $t\in[T]$, the agent can choose between:
\begin{enumerate}
    \item Selecting $B_t,A_t$ inside $\Bs_t, \As_t$. In this case, the agent gets less information but can build good estimates of $F(B_t),S(A_t)$.
    \item Selecting $B_t,A_t$ outside $\Bs_t, \As_t$. In this case, the agent can potentially get more information but the resulting estimates of $F(B_t)$ and $S(A_t)$ are not as accurate as $\widehat F_t$ and $\widehat S_t$.
\end{enumerate}
In the next sections, we show how to leverage the previous results to build a no-regret algorithm for the LOB feedback setting.

\subsection{Our Algorithm}
\label{sec:algorithm_LOB}

\begin{algorithm}[t]
\caption{\textsc{Optimistic Double Pointer Market Maker}}
\label{alg:ODPM}
\begin{algorithmic}[1]
\Require Time horizon $T\in\mathbb{N}_{>0}$, discretized sets $\Bs^\circ,\As^\circ$, functions $\psi_t(b),\psi_t(a)$
\State Set $\Bs_{0}^\circ\gets \Bs^\circ, \As^\circ_{0}\gets\As^\circ$ 
\State Set $\widehat F_0\gets0$, $\widehat S_0\gets0$, $\psi_0\gets1$
\For{$t=1,\dots,T$}
    \State $B_{t,0},A_{t,0}\gets \text{Algorithm } \ref{alg:DDP}\ (\Bs_{t-1}^\circ, \As^\circ_{t-1}, S_{t-1}^\text{up},F_{t-1}^\text{up},S_{t-1}^\text{low},F_{t-1}^\text{low})$\label{algline:call}
    \State $\Bs_{t}^\circ\gets \Bs^\circ \cap [B_{t,0},1]$
    \State $\As_{t}^\circ\gets\As^\circ \cap [0,A_{t,0}]$
    \State $B_t\gets B_{t,0}$
    \State $A_t\gets A_{t,0}$
    \State $X\gets1$, $Y\gets1$
    \If{$S_{t-1}^\text{up}(A_t)< F_{t-1}^\text{low}(B_t)$}
        \State Sample $X\sim \text{Bern}(S_{t-1}^\text{up}(A_t)/F_{t-1}^\text{low}(B_t))$
        \If{$X=0$}\label{algline:ber1}
            \State $B_t\gets 0$
        \EndIf
    \ElsIf{$F_{t-1}^\text{up}(B_t)<S_{t-1}^\text{low}(A_t)$}
        \State Sample $Y\sim \text{Bern}(F_{t-1}^\text{up}(B_t)/S_{t-1}^\text{low}(A_t))$
        \If{$Y=0$}\label{algline:ber2}
            \State $A_t\gets 1$
        \EndIf
    \EndIf
    \State Pull $(B_t,A_t)$
    \State Observe LOB feedback as prescribed in Protocol~\ref{alg: Learner-Environment Interaction}
    \State Update all parameters on $\Bs_t^\circ,\As_t^\circ$ \label{algline:bigupdate}
\EndFor
\end{algorithmic}
\end{algorithm}

\begin{algorithm}[t]
\caption{\textsc{Double Double Pointer}}
\label{alg:DDP}
\begin{algorithmic}[1]
\Require Discrete sets of the form $[B,1]^\circ,[0,A]^\circ$, upper bound for cumulative and survival functions $ S^\text{up}, F^\text{up}$, lower bound for cumulative and survival functions $ S^\text{low}, F^\text{low}$.
\State $J^\text{low}=0$
\State Let $\{b_0,\dots,b_{N-1}\}:=[B,1]^\circ$
\State Let $\{a_0,\dots,a_{M-1}\}:=[0,A]^\circ$
\State $m\gets M-1$
\For{$n=0,\dots ,N-1$}\label{algline:lowbound1}
    \While{$m\ge0$ and $b_n<a_m$}
        \State $J^\text{low}=\max\{J^\text{low}, (a_m-b_n)\min(F^\text{low}(b_n),S^\text{low}(a_m))\}$\label{algline:jlow}
        \If{$F^\text{low}(b_n)< S^\text{low}(a_m)$}
            \State \textbf{break}
        \EndIf
        \State $m\gets m-1$
    \EndWhile
\EndFor\label{algline:lowbound2}
\State $n\gets 0$
\State $m\gets M-1$
\While{$n<N$, $m\ge0$, and $b_n<a_m$}\label{algline:upbound1}
    \State $J^\text{up}\gets (a_m-b_n) \min\{F^\text{up}(b_n), S^\text{up}(a_m)\}$\label{algline:jup}
    \If{$J^\text{up}\ge J^\text{low}$}
        \State \textbf{Return} $b_n,a_m$
    \EndIf
    \If{$F^\text{up}(b_n)<S^\text{up}(a_m)$}
        \State $n\gets n+1$
    \Else 
        \State $m\gets m-1$
    \EndIf
\EndWhile\label{algline:upbound2}
\end{algorithmic}
\end{algorithm}

We are now ready to describe our algorithm for LOB setting. The pseudocode is provided in Algorithm~\ref{alg:ODPM}. Throughout this section, we work with arbitrary discretizations $\Bs^\circ,\As^\circ$ of the set $[0,1]$ containing the sets of plausible bids and asks, respectively. The correct way to discretize the decision space will be specified later.
\begin{notation}\label{not:ball}
From this point on, for any subset of bids or asks, we use the superscript $^\circ$ to denote its intersection with $\Bs^\circ$ or $\As^\circ$, respectively. In particular, we will use:
$$\Bs_t^\circ := \Bs_t \cap \Bs^\circ \qquad \As_t^\circ := \As_t \cap \As^\circ.$$
\end{notation}
The algorithm makes use of confidence-bound parameters, defined such as Equation~\eqref{eq:good} holds with high probability, which we do not specify here in order to maintain full generality.
\begin{notation}
For the rest of this section, we fix sequences of functions $\{\psi_t(b,\delta)\}_{t=1}^T$ and $\{\psi_t(a,\delta)\}_{t=1}^T$ such that Equation~\eqref{eq:good} holds with probability at least $1-\delta$. We also omit the dependence on $\delta$, and simply write $\psi_t(b)$ and $\psi_t(a)$.
\end{notation}
Given these confidence intervals, we define upper and lower confidence bounds for $F$ and $S$ as:
\begin{equation}
    F_t^{\mathrm{up}}(b) = \min\{1,\ \widehat F_t(b) + \psi_t(b)\},
    \qquad
    F_t^{\mathrm{low}}(b) = \max\{0,\ \widehat F_t(b) - \psi_t(b)\},
    \label{eq:boundsbid}
\end{equation}
\begin{equation}
    S_t^{\mathrm{up}}(a) = \min\{1,\ \widehat S_t(a) + \psi_t(a)\},
    \qquad
    S_t^{\mathrm{low}}(a) = \max\{0,\ \widehat S_t(a) - \psi_t(a)\}.
    \label{eq:boundsask}
\end{equation}
As proved in the previous section, under the good event~\eqref{eq:good}, the true functions $F$ and $S$ belong to the confidence sets $[F_t^{\mathrm{low}},F_t^{\mathrm{up}}]$ and $[S_t^{\mathrm{low}},S_t^{\mathrm{up}}]$ at every $t\in[T]$ and for every $b \in \Bs^\circ$ and $a \in \As^\circ$.

\subsubsection{Algorithm Description}
Algorithm~\ref{alg:ODPM} maintains a pair of posted prices $(B_t,A_t)$ over time. At the beginning of round $t$, the algorithm calls the subroutine \textsc{Double Double Pointer} (Algorithm~\ref{alg:DDP}) at Line~\ref{algline:call}, passing as input the current discretized active sets $\Bs_{t-1}^\circ,\As_{t-1}^\circ$ together with the upper and lower confidence bounds based on the first $t-1$ observations. The subroutine returns a candidate pair $(B_{t,0},A_{t,0})$, which determines the new active sets and is used to update current prices by setting $B_t\gets B_{t,0}$ and $A_t\gets A_{t,0}$.

The algorithm then possibly modifies this candidate pair through two symmetric randomized corrections, applying at most one of them. If the optimistic upper bound on the survival function at the current ask is smaller than the pessimistic lower bound on the cumulative function at the current bid, \ie, if:
\[
S_{{t-1}}^{\mathrm{up}}(A_t)<F_{{t-1}}^{\mathrm{low}}(B_t),
\]
then the algorithm samples a Bernoulli random variable with parameter $S_{t-1}^{\mathrm{up}}(A_t)/F_{t-1}^{\mathrm{low}}(B_t)$. If this sample is equal to zero, then the bid is reset to $0$ at Line~\ref{algline:ber1}. Otherwise, symmetrically, if:
\[
F_{{t-1}}^{\mathrm{up}}(B_t)<S_{{t-1}}^{\mathrm{low}}(A_t),
\]
then the algorithm samples a Bernoulli random variable with parameter $F_{t-1}^{\mathrm{up}}(B_t)/S_{t-1}^{\mathrm{low}}(A_t)$, and if the sample is zero, then the ask is reset to $1$ at Line~\ref{algline:ber2}. After these possible corrections, the algorithm pulls $(B_t,A_t)$, observes the LOB feedback, and updates all estimates and confidence bounds before moving to the next round. For the uniform grids in Equation~\eqref{eq:disc}, the initial call returns $(0,1)$, so the first round reveals both valuations.

The role of Algorithm~\ref{alg:DDP} is to compute the candidate pair returned back to Algorithm~\ref{alg:ODPM}. Given active discretized sets of the form $\Bs_{t-1}^\circ$ and $\As_{t-1}^\circ$, together with upper and lower confidence bounds for $F$ and $S$, it performs two consecutive double-pointer scans. First, it computes a pessimistic benchmark value $J^{\mathrm{low}}$. It initializes $J^{\mathrm{low}}=0$, enumerates the active bids and asks as $\{b_0,\dots,b_{N-1}\}$ and $\{a_0,\dots,a_{M-1}\}$, and starts with the ask pointer at the largest ask. Then, in the loop of Lines~\ref{algline:lowbound1}--\ref{algline:lowbound2}, it scans the bid grid from left to right. For each bid $b_n$, it visits the feasible asks from the current ask pointer downward. At every visited pair, Line~\ref{algline:jlow} updates $J^{\mathrm{low}}$ with the value:
\[
(a_m-b_n)\min(F^{\mathrm{low}}(b_n),S^{\mathrm{low}}(a_m)),
\]
that is, with the pessimistic value associated with the current pair $(b_n,a_m)$.

The pair is evaluated before either pointer moves. If $F^\text{low}(b_n)<S^\text{low}(a_m)$, the inner loop ends and the bid pointer advances; otherwise, the ask pointer decreases. 

Once this lower benchmark has been computed, the subroutine resets the two pointers and starts a second scan, now using the optimistic confidence bounds. In the loop of Lines~\ref{algline:upbound1}--\ref{algline:upbound2}, it evaluates at the current pair $(b_n,a_m)$ the optimistic quantity:
\[
J^{\mathrm{up}}=(a_m-b_n)\min\left(F^{\mathrm{up}}(b_n),S^{\mathrm{up}}(a_m)\right),
\]
as in Line~\ref{algline:jup}. The scan returns the current pair immediately if $J^{\mathrm{up}}\ge J^{\mathrm{low}}$. If $F^{\mathrm{up}}(b_n)<S^{\mathrm{up}}(a_m)$, the bid pointer is increased; otherwise, the ask pointer is decreased.  In Figure~\ref{fig:ddp}, we provide a graphical description of Algorithm~\ref{alg:DDP}. 

In this way, Algorithm~\ref{alg:ODPM} employs Algorithm~\ref{alg:DDP} to identify, at each round, a candidate bid-ask pair whose optimistic value is large enough compared with the best pessimistic value currently computed given the confidence bounds. The two possible randomized corrections performed afterwards address the cases in which the selected pair $(B_{t,0},A_{t,0})$ returned by \textsc{Double Double Pointer} is not balanced with respect to the confidence bounds. If  $S_{t-1}^{\mathrm{up}}(A_t)<F_{t-1}^{\mathrm{low}}(B_t)$, then the bid side is too large relative to the ask side, and the algorithm resets $B_t$ to $0$ with a suitable probability. Symmetrically, if $F_{t-1}^{\mathrm{up}}(B_t)<S_{t-1}^{\mathrm{low}}(A_t)$, then the ask side is too large relative to the bid side, and the algorithm resets $A_t$ to $1$ with a suitable probability.

\begin{figure}[t]
\centering
\begin{tikzpicture}[>=Latex, font=\small, scale=0.95]

\definecolor{lowbg}{RGB}{232,242,255}
\definecolor{lowpath}{RGB}{45,125,210}
\definecolor{lowmark}{RGB}{0,170,140}

\definecolor{upbg}{RGB}{255,239,224}
\definecolor{uppath}{RGB}{230,120,40}
\definecolor{retmark}{RGB}{70,170,90}

\begin{scope}[xshift=0cm]

    \fill[lowbg, rounded corners=6pt] (-0.45,-0.55) rectangle (5.45,5.75);

    \draw[step=1cm, white, very thick] (0,0) grid (5,5);
    \draw[rounded corners=4pt, thick] (0,0) rectangle (5,5);

    \node[font=\bfseries] at (2.5,5.45) {$J^{\mathrm{low}}$};

    \node[below] at (0.5,0) {$b_0$};
    \node[below] at (4.5,0) {$b_{N-1}$};
    \node[left] at (0,4.5) {$a_{M-1}$};
    \node[left] at (0,0.5) {$a_0$};

    \draw[line width=2.5pt, lowpath, -{Latex[length=3mm]}]
        (0.5,4.5) -- (0.5,3.5) -- (1.5,3.5) -- (1.5,2.5) --
        (1.5,1.5) -- (2.5,1.5) -- (3.5,1.5) -- (3.5,0.5) -- (4.5,0.5);

    \foreach \x/\y in {0/4,0/3,1/3,1/2,1/1,2/1,3/1,3/0,4/0}
        \fill[lowpath] (\x+0.5,\y+0.5) circle (2.2pt);

    \fill[red!80] (0.5,4.5) circle (3pt);
    \node[above right, red!80!black] at (0.5,4.5) {\scriptsize start};

    \fill[lowmark] (3.5,1.5) circle (4.5pt);
    \draw[lowmark, line width=1.2pt] (3,1) rectangle (4,2);

\end{scope}

\begin{scope}[xshift=7.5cm]

    \fill[upbg, rounded corners=6pt] (-0.45,-0.55) rectangle (5.45,5.75);

    \draw[step=1cm, white, very thick] (0,0) grid (5,5);
    \draw[rounded corners=4pt, thick] (0,0) rectangle (5,5);

    \node[font=\bfseries] at (2.5,5.45) {$J^{\mathrm{up}}$};

    \node[below] at (0.5,0) {$b_0$};
    \node[below] at (4.5,0) {$b_{N-1}$};
    \node[left] at (0,4.5) {$a_{M-1}$};
    \node[left] at (0,0.5) {$a_0$};

    \draw[line width=2.5pt, uppath, -{Latex[length=3mm]}]
        (0.5,4.5) -- (1.5,4.5) -- (1.5,3.5) -- (2.5,3.5) -- (2.5,2.5);

    \foreach \x/\y in {0/4,1/4,1/3,2/3,2/2}
        \fill[uppath] (\x+0.5,\y+0.5) circle (2.2pt);

    \fill[red!80] (0.5,4.5) circle (3pt);

    \fill[retmark] (2.5,2.5) circle (4.5pt);
    \draw[retmark, line width=1.2pt] (2,2) rectangle (3,3);

\end{scope}

\draw[very thick, gray!70, -{Latex[length=4mm]}] (5.8,2.5) -- (6.7,2.5);

\end{tikzpicture}
\caption{Illustration of \textsc{Double Double Pointer}. Both scans start from $(b_0,a_{M-1})$ and move right or down, evaluating each visited feasible pair before moving. Left: the scan computes the benchmark value $J^{\mathrm{low}}$ using the lower confidence bounds; the green box marks a maximizing pair. It continues until a pointer exits its grid or no feasible pair remains. Right: the algorithm restarts and scans using the upper confidence bounds, stopping at the first pair whose optimistic value reaches the lower benchmark.}
\label{fig:ddp}
\end{figure}
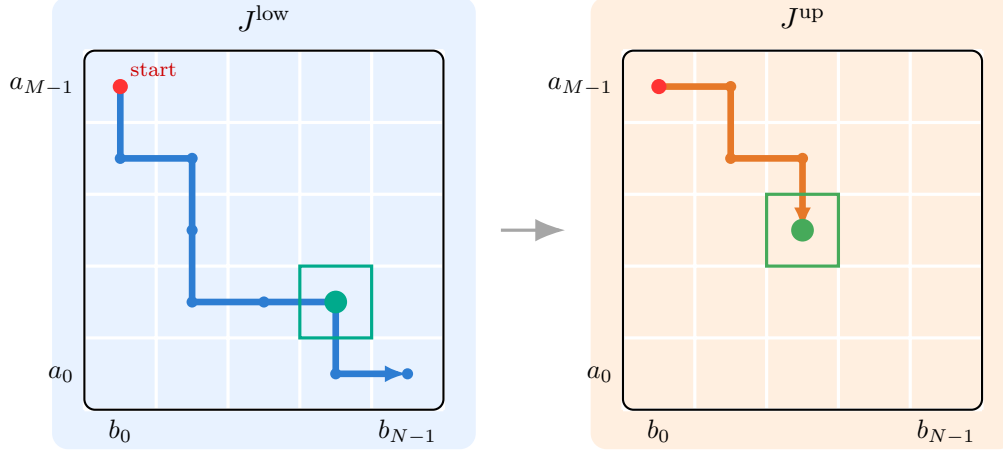

\subsubsection{Computational Complexity}

In this section, we study the computational complexity of Algorithm~\ref{alg:ODPM}.

Similarly to Section~\ref{sec:2B}, we refer as $\gamma_t$ to the per-round complexity of choosing $B_t, A_t$ at each time-step $t$. Thus, our goal reduces to bound:
$$\Gamma:=\max_{t=1,\dots T}\gamma_t.$$

\begin{theorem}\label{prop:compcomp}
    The (per-round) run time computational complexity of Algorithm~\ref{alg:ODPM} is bounded as:
    $$\Gamma \le \Os(|\Bs^\circ|+|\As^\circ|).$$
\end{theorem}
\begin{proof}
    The computational complexity $\gamma_t$ of a single step of Algorithm~\ref{alg:ODPM} boils down to the one of Lines \ref{algline:bigupdate} and \ref{algline:call}, and to the one of Algorithm~\ref{alg:DDP}, since all the other lines require constant time. Specifically, Line~\ref{algline:bigupdate}
    requires to update all parameter as in Equations~\eqref{eq:pseudo-indicator},~\eqref{eq:boundsbid} and~\eqref{eq:boundsask}. Computing, for each surviving bid and each surviving ask, the sample mean $\widehat F_t(b)$ or $\widehat S_t(a)$ requires constant time, employing an incremental update. Computing the bonuses---even in the most complex case of Bernstein bonuses~\eqref{eq:bernstein2}---also requires constant time. Therefore, it holds:
    $$\gamma_t^\text{update}\le \Os(|\Bs^\circ|+|\As^\circ|).$$
    The complexity of the call to Algorithm~\ref{alg:DDP} can be bounded as follows. Consider Lines~\ref{algline:lowbound1}--\ref{algline:lowbound2}. At every step of the inner cycle, either $n$ increases (when the while loop ends) or $m$ decreases. Each comparison and update requires constant time. Therefore, the total number of elementary steps does not exceed:
    $$\gamma_t^\text{first loop}\le \Os(N+M)\le \Os(|\Bs^\circ|+|\As^\circ|).$$
    A similar reasoning holds for the second loop at Lines~\ref{algline:upbound1}--\ref{algline:upbound2}. In an analogous way, it holds:
    $$\gamma_t^\text{second loop}\le \Os(N+M)\le \Os(|\Bs^\circ|+|\As^\circ|).$$
    Summing the three terms $\gamma_t^\text{update},\gamma_t^\text{first loop},\gamma_t^\text{second loop}$ completes the proof.
\end{proof}

Theorem~\ref{prop:compcomp} shows that, although the optimization problem is bivariate, its computational complexity scales as $\mathcal{O}(|\Bs^\circ| + |\As^\circ|)$ rather than $\mathcal{O}(|\Bs^\circ| |\As^\circ|)$. This efficiency is due to our two-pointer algorithmic structure. 

This result can be further specialized based on the choice of discretization which will be specified in the following section.
Specifically, when DKW confidence bounds are employed, following Theorem~\ref{thm:mainDKW} and setting $N=\lceil\sqrt T\rceil$, so that $|\Bs^\circ|=|\As^\circ|=N+1$ yields a complexity of order $\Gamma \leq \mathcal{O}(\sqrt{T}).$ Under the Bernstein-type bounds of Theorem \ref{thm:mainBER}, a finer grid $N=T$ is required, leading to $\Gamma \leq \mathcal{O}(T).$
This difference highlights an important trade-off. To obtain an instance-dependent bound, one needs a finer discretization grid. By contrast, if a worst-case regret guarantee is enough, a discretization step of size $T^{-1/2}$ is sufficient, leading to a lower computational complexity.

\subsection{Theoretical Guarantees of Algorithm~\ref{alg:ODPM}}
\label{sec:regret_LOB}

In the following section, we provide the theoretical guarantees attained by Algorithm~\ref{alg:ODPM}.

\subsubsection{Preliminaries Results}
We start proving one necessary condition, that is, the algorithm selects only feasible actions.
\begin{proposition}\label{prop:ddp-correct}
    Algorithm~\ref{alg:ODPM} always selects a bid-ask pair such that $B_t<A_t$.
\end{proposition}
\begin{proof}
    By design of the algorithm, two cases may occur:
    \begin{enumerate}
        \item If one of the conditions at Lines~\ref{algline:ber1},\ref{algline:ber2} is triggered, then either $B_t=0$ or $A_t=1$. These corrections preserve the feasibility of the pair returned by Algorithm~\ref{alg:DDP}, which we verify in the next point.
        \item In the opposite case, the pair of arms to pull is determined by Algorithm~\ref{alg:DDP}. In either scan, whenever the bid pointer advances, every remaining feasible pair using that bid is dominated by the current pair: its spread is no larger, and its minimum bound is at most the bid's bound. The analogous statement holds when the ask pointer decreases. When $b_n\ge a_m$, no feasible pair remains. Thus, the routine in Lines~\ref{algline:lowbound1}--\ref{algline:lowbound2} finds the maximum lower value:
        \[
        J^\text{low}=(a_m'-b_n')\min\{F^\text{low}(b_n'),S^\text{low}(a_m')\},
        \]
        for some feasible pair $(b_n',a_m')$. At this point, either $J^\text{low}=0$, in which case the second routine immediately returns the feasible outer pair, or $J^\text{low}>0$.

        In the latter case, the second routine stops when it finds a pair such that:
        \[
        (a_m-b_n)\min\{F^\text{up}(b_n),S^\text{up}(a_m)\}\ge J^\text{low}.
        \]
        Since $F^\text{up}(b_n')\ge F^\text{low}(b_n')$ and $S^\text{up}(a_m')\ge S^\text{low}(a_m')$, the pair $(b_n',a_m')$ has upper value at least $J^\text{low}$. By the same dominance argument, it cannot be discarded before the routine returns a pair meeting this threshold. The loop condition ensures that the returned pair satisfies $b_n<a_m$.
    \end{enumerate}
    This concludes the proof.
\end{proof}

In the next result, we prove a relation between the random imbalance at time $t$ and the sum of its conditional expected increments.
\needspace{15\baselineskip}
\begin{proposition}\label{prop:imbalance}
    Let $\rho_\tau:=\E[\mathbbm{I}(\Vsell_\tau\le B_\tau)-\mathbbm{I}(\Vbuy_\tau\ge A_\tau)|\Fs_{\tau-1}]$. The following relations hold:
    \begin{itemize}
        \item With probability at least $1-\delta$, simultaneously for every $t=1,\dots,T$,
        \[
        \left|I_t-\sum_{\tau=1}^t\rho_\tau\right|\le\sqrt{2T\log(2/\delta)}.
        \]
        \item Calling:
        \[
        \sigma_t^2:=\text{Var}[\mathbbm{I}(\Vsell_t\le B_t)-\mathbbm{I}(\Vbuy_t\ge A_t)|\Fs_{t-1}],
        \]
        for every deterministic $v\ge0$, with probability at least $1-\delta$, simultaneously for every $t\le T$ such that $\sum_{\tau=1}^t\sigma_\tau^2\le v$, it holds:
        \[
        \left|I_t-\sum_{\tau=1}^t\rho_\tau\right|\le\sqrt{2v\log(2/\delta)}+\frac43\log(2/\delta).
        \]
    \end{itemize}
    The same conclusions hold for the cumulative reward, with $\rho_t$ and $\sigma_t^2$ replaced by the conditional mean and variance of $r_t$.
\end{proposition}
\begin{proof}
    Let $\Fs_t$ contain all valuations and algorithmic random draws through round $t$. By definition,
    \[
    I_t-\sum_{\tau=1}^t\rho_\tau
    =\sum_{\tau=1}^t\left(\mathbbm{I}(\Vsell_\tau\le B_\tau)-\mathbbm{I}(\Vbuy_\tau\ge A_\tau)-\rho_\tau\right).
    \]
    Each summand has conditional mean zero and conditional range of length at most $2$. The maximal Hoeffding--Azuma inequality~\cite{lattimore2020bandit}, applied to both signs, proves the first part.

    For the second part, we apply Freedman's inequality~\cite{boucheron2003concentration}. Since the centered increments have absolute value at most $2$, for every $\varepsilon>0$ and $v\ge0$, it holds:
    \[
    \Prob\left(\exists t\le T:\ I_t-\sum_{\tau=1}^t\rho_\tau\ge\varepsilon,\ \sum_{\tau=1}^t\sigma_\tau^2\le v\right)
    \le\exp\left(-\frac{\varepsilon^2}{2(v+2\varepsilon/3)}\right).
    \]
    Applying this bound to both signs with $\varepsilon=\sqrt{2v\log(2/\delta)}+\frac43\log(2/\delta)$ gives the second part. The reward $r_t$ also lies in $[-1,1]$, so the same proof applies to its centered sum. 
    
    This concludes the proof.
\end{proof}

The previous result relates the true imbalance to the sum of the variables $\rho_\tau$. We next bound these conditional expected increments under the good event.

\begin{theorem}\label{thm:balance}
    Under the good event~\eqref{eq:good}, it holds:
    \[
    \left|\sum_{\tau=1}^t\rho_\tau\right|\le4\sum_{\tau=1}^t\max\{\psi_{\tau-1}(A_{\tau,0}),\psi_{\tau-1}(B_{\tau,0})\}.
    \]
\end{theorem}
\begin{proof}
    Fix a round $\tau$. In this proof, let $\Fs_{\tau-1}$ contain all valuations and algorithmic random draws through round $\tau-1$, before the fresh correction draws, and let $\rho_\tau:=\E[I_\tau-I_{\tau-1}\mid\Fs_{\tau-1}]$. By Lines~\ref{algline:ber1} and~\ref{algline:ber2}, $\rho_\tau$ writes as:
    \begin{align*}
        \rho_\tau &=\E[\mathbbm{I}(\Vsell_\tau\le B_\tau)-\mathbbm{I}(\Vbuy_\tau\ge A_\tau)|\Fs_{\tau-1}]\\*
        &=F(B_{\tau,0})\Prob(X=1|\Fs_{\tau-1})+F(0)\Prob(X=0|\Fs_{\tau-1})\\*
        &\qquad-S(A_{\tau,0})\Prob(Y=1|\Fs_{\tau-1})-S(1)\Prob(Y=0|\Fs_{\tau-1})\\*
        &=F(B_{\tau,0})\Prob(X=1|\Fs_{\tau-1})-S(A_{\tau,0})\Prob(Y=1|\Fs_{\tau-1}).
    \end{align*}
    We prove the result in case $F(B_{\tau,0})\ge S(A_{\tau,0})$. The other case follows analogously switching asks and bids. Under the good event, the ask correction cannot be triggered, since:
    \[
    F_{\tau-1}^\text{up}(B_{\tau,0})\ge F(B_{\tau,0})\ge S(A_{\tau,0})\ge S_{\tau-1}^\text{low}(A_{\tau,0}).
    \]
    Therefore $\Prob(Y=1|\Fs_{\tau-1})=1$. Now, two cases may occur:
    \begin{enumerate}
        \item If the bid correction is not triggered, $S_{\tau-1}^\text{up}(A_{\tau,0})\ge F_{\tau-1}^\text{low}(B_{\tau,0})$. By Lemma~\ref{lem:double}, it holds:
        \begin{align*}
            0\le\rho_\tau&=F(B_{\tau,0})-S(A_{\tau,0})\\*
            &\le F(B_{\tau,0})-F_{\tau-1}^\text{low}(B_{\tau,0})+S_{\tau-1}^\text{up}(A_{\tau,0})-S(A_{\tau,0})\\*
            &\le2\psi_{\tau-1}(B_{\tau,0})+2\psi_{\tau-1}(A_{\tau,0}).
        \end{align*}
        \item In the opposite case, $F_{\tau-1}^\text{low}(B_{\tau,0})>S_{\tau-1}^\text{up}(A_{\tau,0})\ge0$, so the denominator in the correction probability is strictly positive. Since $F_{\tau-1}^\text{low}(B_{\tau,0})\Prob(X=1|\Fs_{\tau-1})=S_{\tau-1}^\text{up}(A_{\tau,0})$, it holds:
        \begin{align*}
            0\le\rho_\tau&=\bigl(F(B_{\tau,0})-F_{\tau-1}^\text{low}(B_{\tau,0})\bigr)\Prob(X=1|\Fs_{\tau-1})\\*
            &\qquad+S_{\tau-1}^\text{up}(A_{\tau,0})-S(A_{\tau,0})\\*
            &\le2\psi_{\tau-1}(B_{\tau,0})+2\psi_{\tau-1}(A_{\tau,0}),
        \end{align*}
        where the last step is due to Lemma~\ref{lem:double} and $\Prob(X=1|\Fs_{\tau-1})\le1$.
    \end{enumerate}
    Putting everything together and using the symmetric argument when $F(B_{\tau,0})<S(A_{\tau,0})$, we have shown that:
    $$|\rho_\tau|\le4\max\{\psi_{\tau-1}(A_{\tau,0}),\psi_{\tau-1}(B_{\tau,0})\}.$$
    Summing this bound over the first $t$ rounds completes the proof.
\end{proof}

This result allows to control the imbalance. In the following, we turn our attention to the regret.

\subsubsection{Bounding The (Virtual) Regret}

The first step towards proving the regret bound is to show that, under the good event, the virtual optimum is never eliminated. First, we need to show that the routine at Lines~\ref{algline:jlow} and~\ref{algline:jup} used to decide which arm to eliminate works as expected.
\begin{lemma}\label{lem:optj}
    Let:
    $$J_t^\textnormal{up}(b,a):= (a-b) \min\{F_t^\textnormal{up}(b), S_t^\textnormal{up}(a)\},\qquad J_t^\textnormal{low}(b,a):= (a-b) \min\{F_t^\textnormal{low}(b), S_t^\textnormal{low}(a)\}.$$
    Then, under the good event, for every pair $b\in \Bs_t^\circ, a\in \As_t^\circ$ with $b<a$ it holds:
    $$J_t^\text{low}(b,a)\le J(b,a)\le J_t^\text{up}(b,a).$$
\end{lemma}
\begin{proof}
    By Proposition~\ref{prop:ecdfvar} and the definition of good event, $F_t^\text{low}(b)\le F(b)\le F_t^\text{up}(b)$ and $S_t^\text{low}(a)\le S(a)\le S_t^\text{up}(a)$. Therefore,
    \begin{align*}
        J_t^\text{low}(b,a)&= (a-b) \min\{F_t^\text{low}(b), S_t^\text{low}(a)\}\\
        &\le (a-b) \min\{F(b), S(a)\}\qquad (=J(b,a))\\
        &\le (a-b) \min\{F_t^\text{up}(b), S_t^\text{up}(a)\}\\&=J_t^\text{up}(b,a).
    \end{align*}
    This concludes the proof.
\end{proof}
We are now ready to prove the result on the virtual optimum.
\begin{theorem}\label{thm:resilience}
    Under the good event, the pair $b^\circ,a^\circ$ in Definition~\ref{def:virtual} never gets eliminated.
\end{theorem}
\begin{proof}
    Initially both optimal prices are active. Suppose they are active before round $t$. By design of Algorithm~\ref{alg:ODPM}, in order for $a^\circ,b^\circ$ to be eliminated, one of the following conditions must occur:
    \begin{enumerate}
        \item \textbf{$a^\circ$ is eliminated before $b^\circ$, at some round $t\in[T]$}. 
        
        If this happens, there must be $b\le b^\circ$ such that $F_{t-1}^\text{up}(b)\ge S_{t-1}^\text{up}(a^\circ)$ and $J_{t-1}^\text{up}(b,a^\circ) < J_{t-1}^\text{low}(b',a')$ for some feasible pair $b'\in\Bs_{t-1}^\circ,a'\in\As_{t-1}^\circ$. Let $b\le b^\circ$ such that $F_{t-1}^\text{up}(b)\ge S_{t-1}^\text{up}(a^\circ)$. Then, it holds:
        \begin{align*}
            J_{t-1}^\text{low}(b',a') &\le J(b^\circ,a^\circ) \tag*{(Lemma~\ref{lem:optj})}\\
            &\le (a^\circ-b^\circ)\min\{S_{t-1}^\text{up}(a^\circ),F_{t-1}^\text{up}(b^\circ)\}\\
            &\le (a^\circ-b^\circ)S_{t-1}^\text{up}(a^\circ)\\
            &= (a^\circ-b^\circ) \min \{S_{t-1}^\text{up}(a^\circ), F_{t-1}^\text{up}(b)\}\\
            &\le (a^\circ-b) \min \{S_{t-1}^\text{up}(a^\circ), F_{t-1}^\text{up}(b)\}\\
            &= J_{t-1}^\text{up}(b,a^\circ).
        \end{align*}
        This shows that $J_{t-1}^\text{up}(b,a^\circ) < J_{t-1}^\text{low}(b',a')$ is false, which means that this case cannot occur.

        \item \textbf{$b^\circ$ is eliminated before $a^\circ$, at some $t\in[T]$}. 
        
        If this happens, there must be $a\ge a^\circ$ such that $S_{t-1}^\text{up}(a)>F_{t-1}^\text{up}(b^\circ)$ and $J_{t-1}^\text{up}(b^\circ,a) < J_{t-1}^\text{low}(b',a')$ for some feasible pair $b'\in\Bs_{t-1}^\circ,a'\in\As_{t-1}^\circ$. However,
        \begin{align*}
            J_{t-1}^\text{low}(b',a') &\le J(b^\circ,a^\circ) \tag*{(Lemma~\ref{lem:optj})}\\
            &{\le} (a^\circ-b^\circ)\min\{S_{t-1}^\text{up}(a^\circ),F_{t-1}^\text{up}(b^\circ)\}\\
            &\le (a^\circ-b^\circ)F_{t-1}^\text{up}(b^\circ)\\
            &= (a^\circ-b^\circ) \min \{S_{t-1}^\text{up}(a), F_{t-1}^\text{up}(b^\circ)\}\\
            &\le (a-b^\circ) \min \{S_{t-1}^\text{up}(a), F_{t-1}^\text{up}(b^\circ)\}\\
            &= J_{t-1}^\text{up}(b^\circ,a).
        \end{align*}
        So, once again $J_{t-1}^\text{up}(b^\circ,a) < J_{t-1}^\text{low}(b',a')$ cannot be true.
    \end{enumerate}
    This completes the induction and the proof.
\end{proof}

The previous theorem is the main ingredient behind the regret bound. We first bound the sum of the conditional expected reward gaps. This is done in the following theorem.
\begin{theorem}\label{thm:regmain}
    Under the good event, for every $t\in[T]$, it holds:
    \[
    \sum_{\tau=1}^t\bigl(J(b^\circ,a^\circ)-\E[r_\tau|\Fs_{\tau-1}]\bigr)
    \le\sum_{\tau=1}^t\Bigl(6\max\{\psi_{\tau-1}(A_{\tau,0}),\psi_{\tau-1}(B_{\tau,0})\}
    +2\max\{\psi_{\tau-1}(b^\circ),\psi_{\tau-1}(a^\circ)\}\Bigr).
    \]
\end{theorem}
\begin{proof}
    In this proof, let $\Fs_{t-1}$ contain all valuations and algorithmic random draws through round $t-1$, and let $\rho_t:=\E[I_t-I_{t-1}|\Fs_{t-1}]$.
    First, we show that at any step $t\in[T]$, the pair of arms $(B_{t,0},A_{t,0})$ is not far from the optimum of $J$. Indeed, by Proposition~\ref{prop:ecdfvar} and the proof of Proposition~\ref{prop:ddp-correct}, it holds:
    \begin{align*}
        J(B_{t,0},A_{t,0}) &\ge(A_{t,0}-B_{t,0})\min\{F_{t-1}^\text{up}(B_{t,0})-2\psi_{t-1}(B_{t,0}),\\
        &\qquad\qquad\qquad S_{t-1}^\text{up}(A_{t,0})-2\psi_{t-1}(A_{t,0})\}\\
        &\ge J_{t-1}^\text{up}(B_{t,0},A_{t,0})-2\max\{\psi_{t-1}(B_{t,0}),\psi_{t-1}(A_{t,0})\}\\
        &\ge J_{t-1}^\text{low}(b^\circ,a^\circ)-2\max\{\psi_{t-1}(B_{t,0}),\psi_{t-1}(A_{t,0})\}\\
        &\ge J(b^\circ,a^\circ)-2\max\{\psi_{t-1}(b^\circ),\psi_{t-1}(a^\circ)\}-2\max\{\psi_{t-1}(B_{t,0}),\psi_{t-1}(A_{t,0})\},
    \end{align*}
    where the first inequality holds by Lemma~\ref{lem:double} and
    the third inequality holds since, by Theorem~\ref{thm:resilience}, the pair $b^\circ,a^\circ$ is always active, so that $J^\text{low}$ in Line~\ref{algline:jlow} does not fall below $J_{t-1}^\text{low}(b^\circ,a^\circ)$.

    The second step of the proof compares $\E[r_t|\Fs_{t-1}]$ with $J(B_{t,0},A_{t,0})$. By design of the algorithm,
    \begin{align*}
        \E[r_t|\Fs_{t-1}]&=\E[A_t\mathbbm{I}(\Vbuy_t\ge A_t)-B_t\mathbbm{I}(\Vsell_t\le B_t)|\Fs_{t-1}]\\
        &= A_{t,0}\cdot S(A_{t,0})\Prob(Y=1|\Fs_{t-1})+1\cdot S(1)\Prob(Y=0|\Fs_{t-1})\\
        &\qquad -B_{t,0}\cdot F(B_{t,0})\Prob(X=1|\Fs_{t-1})+0\cdot F(0)\Prob(X=0|\Fs_{t-1})\\
        &=A_{t,0}\cdot S(A_{t,0})\Prob(Y=1|\Fs_{t-1})-B_{t,0}\cdot F(B_{t,0})\Prob(X=1|\Fs_{t-1})\\
        &\ge (A_{t,0}-B_{t,0})\min\{S(A_{t,0})\Prob(Y=1|\Fs_{t-1}),F(B_{t,0})\Prob(X=1|\Fs_{t-1})\}\\
        &\qquad-\left|\E[\mathbbm{I}(\Vsell_t\le B_t)-\mathbbm{I}(\Vbuy_t\ge A_t)|\Fs_{t-1}]\right|\\
        &\ge (A_{t,0}-B_{t,0})\underbrace{\min\{S(A_{t,0})\Prob(Y=1|\Fs_{t-1}),F(B_{t,0})\Prob(X=1|\Fs_{t-1})\}}_{T1}-|\rho_t|,
    \end{align*}
    where the proof of Theorem~\ref{thm:balance} bounds $|\rho_t|$ by $4\max\{\psi_{t-1}(A_{t,0}),\psi_{t-1}(B_{t,0})\}$.

    Now, let's consider the remaining term. If neither correction is triggered, $T1=\min\{F(B_{t,0}),S(A_{t,0})\}$. Otherwise, under the good event, if $S(A_{t,0})>F(B_{t,0})$, only the ask correction can be triggered, and its denominator is strictly positive. In this case, it holds:
    \begin{align*}
        T1&=\min\left\{S(A_{t,0})\min\left\{\frac{F_{t-1}^\text{up}(B_{t,0})}{S_{t-1}^\text{low}(A_{t,0})},1\right\},F(B_{t,0})\right\}\\
        &\ge \min\left\{S(A_{t,0})\min\left\{\frac{F(B_{t,0})}{S(A_{t,0})},1\right\},F(B_{t,0})\right\}\\
        &= F(B_{t,0}).
    \end{align*}
    If instead $S(A_{t,0})\le F(B_{t,0})$, only the bid correction can be triggered, again with a strictly positive denominator, and it holds:
    \begin{align*}
        T1&=\min\left\{S(A_{t,0}),F(B_{t,0})\min\left\{\frac{S_{t-1}^\text{up}(A_{t,0})}{F_{t-1}^\text{low}(B_{t,0})},1\right\}\right\}\\
        &\ge \min\left\{S(A_{t,0}),F(B_{t,0})\min\left\{\frac{S(A_{t,0})}{F(B_{t,0})},1\right\}\right\}\\
        &= S(A_{t,0}).
    \end{align*}
    Therefore, in both cases,
    $$T1\ge \min\{F(B_{t,0}),S(A_{t,0})\},$$
    which implies that:
    $$\E[r_t|\Fs_{t-1}]\ge J(B_{t,0},A_{t,0})-4\max\{\psi_{t-1}(A_{t,0}),\psi_{t-1}(B_{t,0})\}.$$
    Combining the previous parts, the sum of the conditional expected reward gaps is bounded as:
    \begin{align*}
        &\sum_{\tau=1}^t\bigl(J(b^\circ,a^\circ)-\E[r_\tau|\Fs_{\tau-1}]\bigr)\\
        &\quad\le\sum_{\tau=1}^t\Bigl(J(b^\circ,a^\circ)-J(B_{\tau,0},A_{\tau,0})
        +4\max\{\psi_{\tau-1}(A_{\tau,0}),\psi_{\tau-1}(B_{\tau,0})\}\Bigr)\\
        &\quad\le\sum_{\tau=1}^t\Bigl(6\max\{\psi_{\tau-1}(A_{\tau,0}),\psi_{\tau-1}(B_{\tau,0})\}
        +2\max\{\psi_{\tau-1}(b^\circ),\psi_{\tau-1}(a^\circ)\}\Bigr).
    \end{align*}
    This concludes the proof.
\end{proof}

\subsubsection{Choosing $\psi$ and The Final Bounds}

In this section, we combine all the previous result in order to get the final theoretical guarantees attained by 
Algorithm~\ref{alg:ODPM}.

We first specify the discretization used for bids and asks. From this point on, we take the discretized bid and ask as follows:
\begin{equation}\label{eq:disc}
\Bs^\circ := \left\{ \frac{n}{N} : n = 0,1,\dots,N \right\},
\qquad
\As^\circ := \left\{ \frac{n}{N} : n = 0,1,\dots,N \right\}.
\end{equation}

This choice allows to prove the following relation between the surrogate regret and the virtual one.
\begin{proposition}\label{virtual_to_surrogate}
    For any algorithm working on the discretized sets $\Bs^\circ,\As^\circ$ defined in Equation~\eqref{eq:disc}, at every $t\le T$, the following inequality holds:
    $$R^J_t\le R^\circ_t+2(t/N).$$
\end{proposition}
\begin{proof}
    By Proposition~\ref{prop:unigrid}, $J^*$ cannot be larger than $J(b^\circ,a^\circ)+2/N$. This immediately gives the result.
\end{proof}

Given the previous result, we are now able to show the final theoretical guarantees attained by Algorithm~\ref{alg:ODPM}. This is done in the following theorem.
\begin{theorem}\label{thm:mainDKW}
    Let $\delta\in(0,1)$. Use $\psi_0=1$. Let $\psi_t(b,\delta)=\psi_t(a,\delta):=\min\left\{1,\sqrt{\nicefrac{\log(12T/\delta)}{2t}}\right\}$ and $N=\lceil \sqrt T \rceil$ in Equation~\eqref{eq:disc}. Then, Algorithm~\ref{alg:ODPM} enjoys, with probability at least $1-\delta$ simultaneously for all the $t$ time-steps, the following guarantees:
    $$|I_t|\le 10\sqrt{T\log(12T/\delta)},\qquad R^J_t\le 15\sqrt{T\log(12T/\delta)}.$$
\end{theorem}
\begin{proof}
    Let $\Fs_t$ contain all valuations and algorithmic random draws through round $t$, and let $\rho_t:=\E[I_t-I_{t-1}|\Fs_{t-1}]$. By Equation~\eqref{eq:hoeffding1} with failure parameter $\delta/3$, the good event has probability at least $1-\delta/3$. Applying the first part of Proposition~\ref{prop:imbalance} with parameter $\delta/3$ to inventory and cumulative reward, and taking a union bound, gives an event of probability at least $1-\delta$ on which the good event and both concentration conclusions hold at every time-step. On this event, the following statements hold:
    \begin{enumerate}
        \item By the first part of Proposition~\ref{prop:imbalance} and Theorem~\ref{thm:balance}, it holds:
        \begin{align*}
            |I_t|&\le\left|\sum_{\tau=1}^t\rho_\tau\right|+\sqrt{2T\log(6/\delta)}\\
            &\le4\sum_{\tau=1}^t\max\{\psi_{\tau-1}(A_{\tau,0}),\psi_{\tau-1}(B_{\tau,0})\}\\
            &\qquad+\sqrt{2T\log(6/\delta)}.
        \end{align*}
        Additionally, by definition of $\psi_t$, it holds:
        \begin{align*}
            &\sum_{\tau=1}^t\max\{\psi_{\tau-1}(A_{\tau,0}),\psi_{\tau-1}(B_{\tau,0})\}\\
            &\qquad\le\sqrt{\frac{\log(12T/\delta)}{2}}
            \left(1+\sum_{s=1}^{T-1}\frac{1}{\sqrt{s}}\right)\\
            &\qquad\le\sqrt{2T\log(12T/\delta)}.
        \end{align*}
        Here we used $\psi_0=1$, $\log(12T/\delta)\ge2$, and $1+\sum_{s=1}^{T-1}s^{-1/2}\le2\sqrt T$, which also holds for $T=1$. This shows that:
        \[
        |I_t|\le4\sqrt{2T\log(12T/\delta)}+\sqrt{2T\log(6/\delta)}
        \le10\sqrt{T\log(12T/\delta)}.
        \]
        \item By Proposition~\ref{virtual_to_surrogate} and Theorem~\ref{thm:regmain}, adding the reward concentration bound gives:
        \begin{align*}
            R_t^J&\le R_t^\circ+2\sqrt T\\
            &\le\sum_{\tau=1}^t\bigl(J(b^\circ,a^\circ)-\E[r_\tau|\Fs_{\tau-1}]\bigr)\\
            &\qquad+\sqrt{2T\log(6/\delta)}+2\sqrt T\\
            &\le8\sqrt{2T\log(12T/\delta)}+\sqrt{2T\log(6/\delta)}+2\sqrt T\\
            &\le15\sqrt{T\log(12T/\delta)},
        \end{align*}
        where we have employed the bonus sum bound from the previous point, since the bonus does not depend on the price. The last step uses $\log(6/\delta)\le\log(12T/\delta)$ and $\log(12T/\delta)\ge1$.
    \end{enumerate}
    This concludes the proof, simultaneously for every $t\le T$.
\end{proof}

\Needspace{14\baselineskip}
We conclude this part by further improving the previous bound, providing an instance dependent result which, while technically more involved, captures the performance of the optimal bid-ask pair.
\begin{theorem}\label{thm:mainBER}
    Let $\delta\in(0,1)$. Use $\psi_0=1$, the two radii in Proposition~\ref{prop:goodevent1} with failure parameter $\delta/6$, and $N=T$ in Equation~\eqref{eq:disc}. Then, Algorithm~\ref{alg:ODPM} enjoys, with probability at least $1-\delta$ simultaneously over all the $t$ time-steps, the following guarantees:
    \[
        |I_t|\le 256\sqrt{\max\{F(b^\circ),S(a^\circ)\}T\log(12T/\delta)}+1728(1+\log T)\log(12T/\delta),
    \]
    and:
    \[R^J_t \le 320\sqrt{\max\{F(b^\circ),S(a^\circ)\}T\log(12T/\delta)}+2160(1+\log T)\log(12T/\delta).\]
\end{theorem}
\begin{proof}
    We start the proof by bounding the conditional variances. Let $\Fs_t$ contain all valuations and algorithmic random draws through round $t$, and let $\rho_t:=\E[I_t-I_{t-1}|\Fs_{t-1}]$. Specifically, the following bound holds:
    \begin{align*}
        \sigma_t^2&:=\text{Var}[\mathbbm{I}(\Vsell_t\le B_t)-\mathbbm{I}(\Vbuy_t\ge A_t)|\Fs_{t-1}]\\
        &\le2\text{Var}[\mathbbm{I}(\Vsell_t\le B_t)|\Fs_{t-1}]
        +2\text{Var}[\mathbbm{I}(\Vbuy_t\ge A_t)|\Fs_{t-1}].
    \end{align*}
    By the law of total variance, it holds:
    \begin{align*}
        \text{Var}[\mathbbm{I}(\Vsell_t\le B_t)|\Fs_{t-1}]
        &=\text{Var}[\E[\mathbbm{I}(\Vsell_t\le B_t)|\Fs_{t-1},X]|\Fs_{t-1}]\\
        &\qquad+\E[\text{Var}[\mathbbm{I}(\Vsell_t\le B_t)|\Fs_{t-1},X]|\Fs_{t-1}]\\
        &=F(B_{t,0})^2\text{Var}[X|\Fs_{t-1}]\\
        &\qquad+F(B_{t,0})(1-F(B_{t,0}))\E[X|\Fs_{t-1}]\\
        &\le F(B_{t,0})\E[X|\Fs_{t-1}]\le F(B_{t,0}).
    \end{align*}
    Indeed, by design of the algorithm, conditional on $\Fs_{t-1}$ and $X$, the trade indicator is Bernoulli with parameter $F(B_{t,0})$ if $X=1$ and zero otherwise. An analogous result holds for the ask side, so that $\sigma_t^2\le2F(B_{t,0})+2S(A_{t,0})$.
    The same bound holds for the conditional variance of $r_t$. Indeed, deactivated sides make no trades, so $r_t=A_{t,0}\mathbbm{I}(\Vbuy_t\ge A_t)-B_{t,0}\mathbbm{I}(\Vsell_t\le B_t)$, and the candidate prices are $\Fs_{t-1}$-measurable and belong to $[0,1]$.

    This bound can be refined using the structure of the algorithm. Specifically, under the good event, Theorem~\ref{thm:resilience} gives $B_{t,0}\le b^\circ<a^\circ\le A_{t,0}$. Therefore,
    \begin{equation}\label{eq:fsaa}
    \begin{aligned}
        F(B_{t,0})&\le F(b^\circ),\qquad S(A_{t,0})\le S(a^\circ),\\
        \sum_{\tau=1}^t\sigma_\tau^2&\le2T(F(b^\circ)+S(a^\circ))=:v.
    \end{aligned}
    \end{equation}
    The sum of the conditional reward variances is bounded by the same deterministic $v$ on the good event. Proposition~\ref{prop:goodevent1} with parameter $\delta/6$ gives $\Prob(\goodevent)\ge1-\delta/3$. Applying the second part of Proposition~\ref{prop:imbalance} with parameter $\delta/3$ to inventory and cumulative reward, and using the variance bounds on the good event, gives a common event of probability at least $1-\delta$ on which the good event and both concentration conclusions hold at every time-step.

    Now, we can proceed as follows on this common event:
    \begin{enumerate}
        \item By the second part of Proposition~\ref{prop:imbalance} and Theorem~\ref{thm:balance}, it holds:
        \begin{align*}
            |I_t|&\le\left|\sum_{\tau=1}^t\rho_\tau\right|
            +\sqrt{2v\log(6/\delta)}+\frac43\log(6/\delta)\\*
            &\le4\sum_{\tau=1}^t\max\{\psi_{\tau-1}(A_{\tau,0}),\psi_{\tau-1}(B_{\tau,0})\}\\*
            &\qquad+\sqrt{2v\log(6/\delta)}+\frac43\log(6/\delta).
        \end{align*}
        The grids have $T+1$ points each, so $\log(18T(T+1)/\delta)\le2\log(12T/\delta)$. By definition of $\psi_t$, Theorem~\ref{thm:pontil}, and Equation~\eqref{eq:fsaa}, we have:
        \begin{align*}
            &\sum_{\tau=1}^t\max\{\psi_{\tau-1}(A_{\tau,0}),\psi_{\tau-1}(B_{\tau,0})\}\\
            &\quad\le1+\sum_{s=1}^{T-1}\left(
            4\sqrt{\frac{\max\{F(b^\circ),S(a^\circ)\}\log(12T/\delta)}{s}}
            +\frac{16\log(12T/\delta)}{s}\right)\\
            &\quad\le1+8\sqrt{T\log(12T/\delta)\max\{F(b^\circ),S(a^\circ)\}}\\
            &\qquad+16(1+\log T)\log(12T/\delta).
        \end{align*}
        The initial radius contributes $1$. The last step uses $\sum_{s=1}^{T-1}s^{-1/2}\le2\sqrt T$ and $\sum_{s=1}^{T-1}s^{-1}\le1+\log T$; for $T=1$, the sums are empty. This finally proves that:
        \[
        |I_t|\le256\sqrt{T\log(12T/\delta)\max\{F(b^\circ),S(a^\circ)\}}
        +1728(1+\log T)\log(12T/\delta).
        \]
        Here we also used $v\le4T\max\{F(b^\circ),S(a^\circ)\}$, $\log(6/\delta)\le\log(12T/\delta)$, and $\log(12T/\delta)\ge2$.
        \item By Proposition~\ref{virtual_to_surrogate} and Theorem~\ref{thm:regmain}, adding the reward concentration bound gives:
        \begin{align*}
            R_t^J&\le R_t^\circ+2t/T\\
            &\le6\sum_{\tau=1}^t\max\{\psi_{\tau-1}(A_{\tau,0}),\psi_{\tau-1}(B_{\tau,0})\}\\
            &\qquad+2\sum_{\tau=1}^t\max\{\psi_{\tau-1}(b^\circ),\psi_{\tau-1}(a^\circ)\}\\
            &\qquad+\sqrt{2v\log(6/\delta)}+\frac43\log(6/\delta)+2t/T.
        \end{align*}
        Finally, employing the bound on the bonus sum provided in the previous point, which also holds with $B_{\tau,0},A_{\tau,0}$ replaced by $b^\circ,a^\circ$, we obtain:
        \[
        R_t^J\le320\sqrt{T\log(12T/\delta)\max\{F(b^\circ),S(a^\circ)\}}
        +2160(1+\log T)\log(12T/\delta).
        \]
    \end{enumerate}
    This concludes the proof, simultaneously for every $t\le T$.
\end{proof}
Theorem~\ref{thm:mainBER} provides an instance-dependent bound for the online market making problem under LOB feedback. In particular, when the discretization parameter in Equation~\eqref{eq:disc} is set to $N=T$ and the confidence radius is chosen according to Proposition~\ref{prop:goodevent1} with parameter $\delta/6$, Algorithm~\ref{alg:ODPM} guarantees, with high probability and uniformly over all time steps, a per-round inventory imbalance $|I_t|$ and a surrogate regret $R_t^J$ of order $\widetilde{\mathcal O}(\sqrt{\max\{F(b^\circ),S(a^\circ)\}T})$.
The benchmark appearing in the bound is the pair $(b^\circ,a^\circ)$ maximizing the surrogate objective over the grid with $\nicefrac{1}{T}$-discretization. In particular, since $\max\{F(b^\circ),S(a^\circ)\}\le 1$, Theorem~\ref{thm:mainBER} immediately implies worst-case bounds of order $\widetilde{\mathcal{O}}(\sqrt{T})$ for both regret and per-round inventory.

\subsection{Lower bound}
\label{sec:lower_LOB}

In this section, we provide a lower bound showing that the previously derived worst-case regret upper bound is optimal, up to logarithmic factors. We remark that the following lower bound is not restricted to the LOB setting studied here. Indeed, we prove it under the much stronger assumption that the agent knows the functions $S$ and $F$ \emph{exactly}. 
This assumption is stronger even than full feedback, and highlights that the impossibility of simultaneously achieving small regret and small imbalance is driven by the intrinsic nature of the problem, rather than by the feedback structure.

In the rest of this section, we are going to refer as \textit{hard instance} to the one characterized by this choice of the c.d.f. and survival function for the bids and asks, respectively:
\begin{equation}
F(b)=
\begin{cases}
    0&\quad b< 1/10\\
    1/2&\quad 1/10\le b<9/10\\
    1&\quad b\geq9/10 
\end{cases}\qquad 
S(a)=
\begin{cases}
    1&\quad a\le 1/10\\
    1/2&\quad 1/10<a\le 9/10\\
    0&\quad a>9/10
\end{cases}\label{eq:hard}
\end{equation}
In such a case, Program~\eqref{progr:offline_opt} admits a unique solution characterized by:
$$b=\frac{1}{10},\qquad a=\frac{9}{10},\qquad \frac{\text{OPT}}{T}=\frac{2}{5}.$$
Due to the structure of the problem, at each time step every action generates the same trades as one of the following representatives, whose reward is at least as large:
\begin{enumerate}
    \item $s_0:=(b,a)=(1/10,9/10)$, which incurs zero expected regret. The imbalance generated follows the following distribution:
    $$X^{(0)}\sim \begin{cases}
        1&\qquad p=1/4\\
        0&\qquad p=1/2\\
        -1&\qquad p=1/4
    \end{cases}.$$
    \item $s_1:=(b,a)=(0,1/10)$, which guarantees a sale without a purchase. Therefore, the associated imbalance is:
    $$X^{(1)}=-1.$$
    Its expected reward is $\frac{1}{10}$; thus, the associated expected regret is $\boxed{\nicefrac{3}{10}}$.
    \item $s_2:=(b,a)=(9/10,1)$, which guarantees a purchase without a sale. For this reason, its imbalance is:
    $$X^{(2)}=1.$$
    Its expected reward is $-\frac{9}{10}$; thus, the associated expected regret is $\boxed{\nicefrac{13}{10}}$.
    \item $s_3:=(b,a)=(0,9/10)$. The associated imbalance is the following:
    $$X^{(3)}\sim\begin{cases}
        0&\qquad p=1/2\\
        -1&\qquad p=1/2
    \end{cases}.$$
    Its expected reward is $\frac{9}{20}$; thus, the associated expected regret corresponds to $\boxed{-\nicefrac{1}{20}}$.
    \item $s_4:=(b,a)=(1/10,1)$. The associated imbalance is the following:
    $$X^{(4)}\sim\begin{cases}
        1&\qquad p=1/2\\
        0&\qquad p=1/2
    \end{cases}.$$
    Its expected reward is $-\frac{1}{20}$; thus, the associated expected regret corresponds to $\boxed{\nicefrac{9}{20}}$.
    \item $s_5:=(b,a)=(0,1)$, which prevents both trades. Its imbalance is:
    $$X^{(5)}=0.$$
    Its expected reward is zero; thus, the associated expected regret is $\boxed{\nicefrac{2}{5}}$.
\end{enumerate}
These six trade patterns exhaust all the possibilities under $b<a$. We refer as $T_i$ to the number of rounds in which the selected action has the same trade pattern as $s_i$, for $i=1,\ldots,5$.
Combining everything we have that the expected regret can be bounded as:
\begin{equation}
\E[R_T]\ge\frac{3}{10}\mathbb{E}[T_1]+\frac{13}{10}\mathbb{E}[T_2]
-\frac{1}{20}\mathbb{E}[T_3]+\frac{9}{20}\mathbb{E}[T_4]+\frac25\mathbb{E}[T_5].\label{eq:lower:countregret}
\end{equation}
Equality holds when the algorithm selects only the representatives. Similarly, by the definition of the inventory, any algorithm must satisfy:
\begin{equation}
    -\mathbb{E}[T_1]+\mathbb{E}[T_2]-\frac{\E[T_3]}{2}+\frac{\E[T_4]}{2}=\E[I_T].\label{eq:lower:const}
\end{equation}
\needspace{12\baselineskip}
Crucially, Equations~\eqref{eq:lower:const} and \eqref{eq:lower:countregret} lead to the following minimization problem that links the number of pulls performed by any algorithm to the associated regret.
The imbalance identity allows us to lower bound expected regret by the value of the following linear program, parametrized over $\E[T_\text{sub}]$, where $T_\text{sub}:=T_1+T_2+T_3+T_4+T_5$, and over $\E[I_T]$:
$$\begin{cases} \min_{\bold v \in \mathbb{R}^5} \quad &\frac{3}{10}v_1 + \frac{13}{10}v_2 - \frac{1}{20}v_3 + \frac{9}{20}v_4+\frac25v_5\\
\text{s.t}.\qquad &- v_1 + v_2 - \frac{1}{2}v_3 + \frac{1}{2}v_4 = \E[I_T] \\ 
&v_1 + v_2 + v_3 + v_4+v_5 = \E[T_\text{sub}]\\ 
&v_1, v_2, v_3, v_4,v_5 \ge 0 
\end{cases},$$
whose value is at least $\E[T_\text{sub}]/5+\E[I_T]/2$. Indeed, subtracting $1/5$ times the second constraint and $1/2$ times the first from the objective leaves the nonnegative quantity $3(v_1+v_2)/5+v_5/5$. Therefore,
\begin{equation}
    \E[R_T]\ge\frac{\E[T_\text{sub}]}{5}+\frac{\E[I_T]}{2}.\label{eq:lower:countbound}
\end{equation}
The terminal inventory term accounts for the negative regret of $s_3$; the proof below refines this comparison by measuring the total correction relative to $s_0$.

Given the aforementioned reasoning, we are now able to provide our lower bound.

\begin{figure}[!t]
    \centering
    \includegraphics[width=1.0\linewidth]{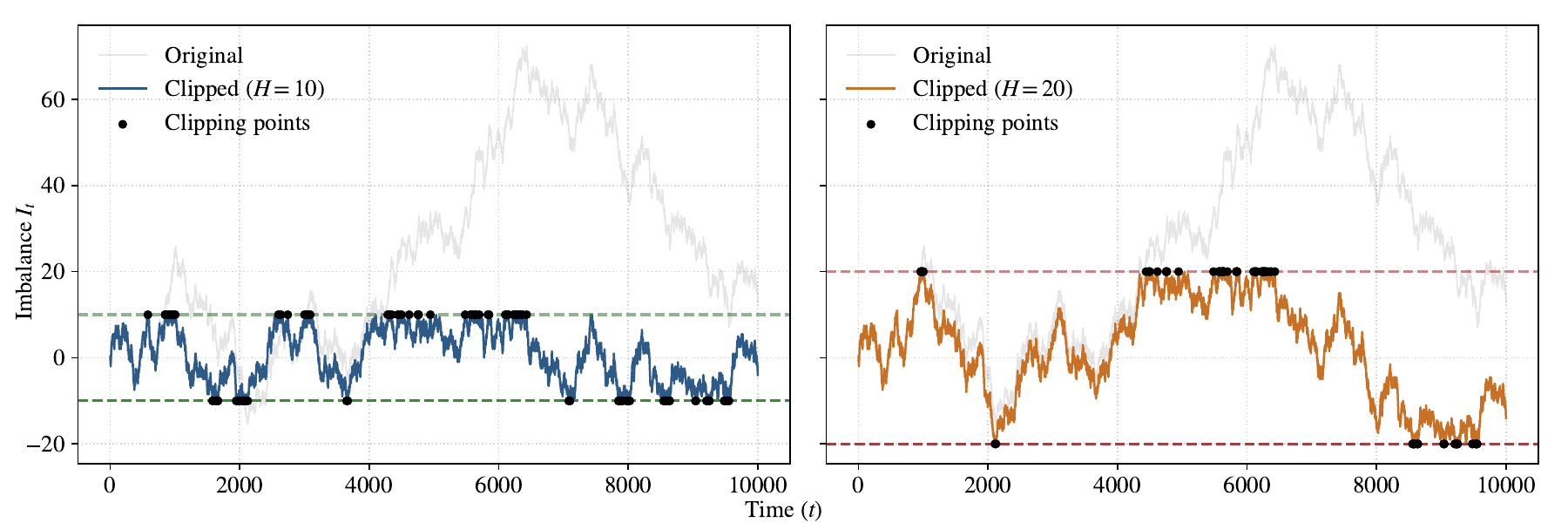}
    \caption{   
    This figure provides heuristic intuition for the lower-bound construction developed in Section~\ref{sec:lower_LOB}. Suppose that the agent aims to guarantee an imbalance bound of the form $|I_t|\le H=\mathcal O(T^\beta)$ for some $\beta<1/2$. If the agent always selects the $J$-maximizing actions $b^\circ$ and $a^\circ$, then, since $F(b^\circ)=S(a^\circ)=\nicefrac{1}{2}$, the resulting imbalance process $I_t$ evolves as a driftless random walk; this behavior is illustrated by the shaded curve. To prevent the walk from leaving the interval $[-H,H]$ (delimited by the dashed lines), the agent must occasionally intervene by selecting suboptimal actions, represented by the black points. As the figure suggests, the smaller $H$ is, the more frequently such corrective interventions are needed, leading to a fundamental trade-off between the strength of the imbalance constraint and the resulting cumulative regret. The core of the proof is to show that expected regret grows at least linearly with the expected total amount of correction, up to an additive $\mathcal O(H)$ term. This correction is of order $\Omega(\nicefrac{T}{H})=\Omega(T^{1-\beta})$.
}
    \label{fig:lowerbound}
\end{figure}

\needspace{8\baselineskip}
\begin{theorem}\label{prop:lowerbound}
    Let $T\ge144H^2$.
    Then, every algorithm that satisfies the per-round inventory constraint for $\alpha\le H$ suffers expected regret $\E[R_T]\ge \frac{T}{720H}$.
\end{theorem}
\begin{proof}
Let $I_0=0$ and $Z_t:=I_t-I_{t-1}$. Define the imbalance generated by $s_0$ on the valuations at round $t$ as:
\[
X_t^{(0)}:=\mathbbm{I}(\Vsell_t\le1/10)-\mathbbm{I}(\Vbuy_t\ge9/10).
\]
Let $\mathcal H_t$ denote the information available after selecting $(B_t,A_t)$ and before observing the valuations at round $t$. Conditional on this information, $\E[Z_t\mid\mathcal H_t]=F(B_t)-S(A_t)$ and $\E[r_t\mid\mathcal H_t]=A_tS(A_t)-B_tF(B_t)$. Therefore,
\begin{equation}
\begin{aligned}
\E\left[\frac25-r_t-\frac12 Z_t\,\middle|\,\mathcal H_t\right]
&=\frac25+\left(B_t-\frac12\right)F(B_t)+\left(\frac12-A_t\right)S(A_t)\\
&\ge \frac25\E\left[\left|Z_t-X_t^{(0)}\right|\,\middle|\,\mathcal H_t\right].
\end{aligned}\label{eq:lower:regret}
\end{equation}
Indeed, when $(F(B_t),S(A_t))=(1/2,1/2)$, the expression equals $2/5-(A_t-B_t)/2\ge0$ and $Z_t=X_t^{(0)}$. The other possible pairs are $(0,1)$, $(0,1/2)$, $(0,0)$, $(1/2,0)$, and $(1,0)$, for which the expression is at least $4/5$, $1/5$, $2/5$, $1/5$, and $4/5$, respectively, while the conditional expectations of $|Z_t-X_t^{(0)}|$ are $1$, $1/2$, $1/2$, $1/2$, and $1$. Combining everything, the regret can be bounded as:
\begin{equation}
\E[R_T]\ge \frac25\E\left[\sum_{t=1}^T\left|Z_t-X_t^{(0)}\right|\right]+\frac12\E[I_T].\label{eq:lower:regret2}
\end{equation}

For some deterministic $0\le j<k\le T$ to be defined, consider the imbalances $I_j$ and $I_k$, the total imbalances at time $j$ and $k$, respectively. Indeed, it holds:
\begin{align*}
    I_k-I_j&=\sum_{\tau=j+1}^{k} \left[\mathbbm{I}(\Vsell_\tau\le B_\tau)-\mathbbm{I}(\Vbuy_\tau\ge A_\tau)\right]\\
    &=\sum_{\tau=j+1}^{k}X_\tau^{(0)}+\Xi_j^k.
\end{align*}
Using this decomposition, $\Xi_j^k:=\sum_{\tau=j+1}^k(Z_\tau-X_\tau^{(0)})$ corresponds to the correction relative to always playing $s_0$.
Now, by definition of $s_0$ and independence of the valuations, it holds:
\[
X_\tau^{(0)}\sim\frac{\xi_\tau^{(1)}+\xi_\tau^{(2)}}{2},
\qquad \xi_\tau^{(1)},\xi_\tau^{(2)}\overset{\indi}{\sim}\text{Rademacher},
\]
with independence also across rounds. Thus, by the Khintchine's anti-concentration inequality \cite{young1976best}, it holds:
\begin{align*}
    \E\left[\left|\sum_{\tau=j+1}^{k}X_\tau^{(0)}\right|\right]
    &=\frac12\E\left[\left|\sum_{\tau=j+1}^{k}\left(\xi_\tau^{(1)}+\xi_\tau^{(2)}\right)\right|\right]\\*
    &\ge\frac{\sqrt{2(k-j)}}{2\sqrt2}\\*
    &=\frac{\sqrt{k-j}}{2}.
\end{align*}
In order for the assumption of the theorem to be satisfied, we have $\E[|I_k-I_j|]\le2H$. Therefore, it follows from the previous equations that:
\[
\E[|\Xi_j^k|]\ge\frac{\sqrt{k-j}}{2}-2H.
\]
Notice that $|\Xi_j^k|\le\sum_{\tau=j+1}^{k}|Z_\tau-X_\tau^{(0)}|$.

Now, we can rewrite the regret lower bound as follows by defining a bucket size $\text{q}:=\lceil36H^2\rceil$. We use the complete deterministic blocks and discard the nonnegative contribution of the remaining rounds. By Equation~\eqref{eq:lower:regret2} and $I_T\ge-H$, it holds:
\begin{align*}
    \E[R_T]&\ge\frac25\sum_{\ell=0}^{\lfloor T/\text{q}\rfloor-1}\E\left[\left|\Xi_{\text{q}\ell}^{\text{q}\ell+\text{q}}\right|\right]-\frac H2\\*
    &\ge\frac25\sum_{\ell=0}^{\lfloor T/\text{q}\rfloor-1}\left(\frac{\sqrt{\text{q}}}{2}-2H\right)-\frac H2\\*
    &\ge\frac25\left\lfloor\frac{T}{\text{q}}\right\rfloor H-\frac H2\\*
    &\ge\frac25\left(\frac{T}{37H^2}-1\right)H-\frac H2\\*
    &\ge\frac{T}{720H},
\end{align*}
where the penultimate step uses $\text{q}\le37H^2$ for $H\ge1$, and the last step comes from the assumption $T\ge144H^2$. This concludes the proof.
\end{proof}

The previous theorem implies, in particular, that any algorithm attaining $o(\sqrt{T})$ per-round inventory imbalance must also incur
$\omega(\sqrt{T}) $ expected regret on the hard instance. 

In Figure~\ref{fig:lowerbound}, we provide a graphical intuition behind the lower bound of our setting.

\section{Attaining Global Inventory Balance}
\label{sec:global}
In this section, we provide a wrapper procedure which, given as input any algorithm with similar guarantees to the ones of Section~\ref{sec:2B}--\ref{sec:LOB}, obtains a global inventory constraint satisfaction, that is, at the end of the process, the number of sold items is equal to the number of purchased ones.

\subsection{Algorithm}

In Algorithm~\ref{alg: global_imbalance}, we provide a black-box metaprocedure that converts any algorithm with a sublinear per-round imbalance guarantee into one that additionally satisfies the global inventory constraint. Specifically, given any algorithm $\mathscr{A}$ whose per-round inventory imbalance is bounded by $I_{\mathscr{A}}$,\footnote{For simplicity, we assume $I_{\mathscr{A}}$ to be an integer. Otherwise it sufficient to use $\lfloor I_{\mathscr{A}}\rfloor $ in place of $I_{\mathscr{A}}$.} we execute $\mathscr{A}$ for the first $T-I_{\mathscr{A}}$ rounds and reserve the final $I_{\mathscr{A}}$ rounds for a balancing phase. By the guarantee on $\mathscr{A}$, the inventory at time $T-I_{\mathscr{A}}$ has magnitude at most $I_{\mathscr{A}}$, and can therefore be fully balanced over the remaining horizon. During this phase, the algorithm repeatedly selects the most inventory-increasing bid-ask pair whenever the agent is short, and the most inventory-decreasing bid-ask pair whenever the agent is long, until the total imbalance reaches zero. After that, it selects the no-trade bid-ask pair $(0,1)$ for all remaining rounds. In this way, the resulting algorithm satisfies the global inventory constraint.

The careful reader will notice that, in this section, we allow our algorithm to play the extreme bid-ask pairs $(1,1)$ and $(0,0)$, which do not satisfy the strict inequality $B_t < A_t$. Nonetheless, this is only a notational convenience. In standard market-making models, one typically has access to prices that effectively implement the decision of not selling or not buying. Thus, these extreme actions should be interpreted as shorthand for reserve prices that deterministically prevent trade on one side while forcing it on the other.


\subsection{Theoretical Guarantees}

\begin{algorithm}[!t]
		\caption{\textsc{Global Inventory Balance Metaprocedure}}
		\label{alg: global_imbalance}
        \begin{algorithmic}[1]
        \Require $T\in\mathbb{N}_{>0}$, Algorithm $\mathscr{A}$ with per-round imbalance upper bound of $I_\mathscr{A}$
        \State Initialize $\mathscr{A}$ for horizon $T$ and set $I_0\gets0$
		\For{$t\in[T-I_{\mathscr{A}}]$}			
			\State Select $(B_t,A_t)\gets\mathscr{A}$
            \State Update $I_t$ from the observed trades
         \EndFor
        \While{$|I_{t-1}|>0$}
        \If{$\sum_{\tau=1}^{t-1} \mathbbm{I}(\Vbuy_\tau\ge A_\tau) > \sum_{\tau=1}^{t-1}\mathbbm{I}(\Vsell_\tau\le B_\tau)$}
        \State Select $(B_t,A_t)\gets(1,1)$
        \Else
        \State Select $(B_t,A_t)\gets(0,0)$
        \EndIf
        \State Update $I_t$ from the observed trades
		\EndWhile
        \State Select $(B_t,A_t)\gets(0,1)$ until the end of the learning dynamic        
        \end{algorithmic}
\end{algorithm}

We can now prove the final result of the paper.

\begin{theorem}
\label{thm:global}
    Let $\delta\in(0,1)$ and let $\mathscr{A}$ be an algorithm for the online market making problem whose (surrogate) regret and per-round imbalance are upper bounded by $R^J_{\mathscr{A}},I_{\mathscr{A}}$ with probability at least $1-\delta$, respectively. Then, Algorithm~\ref{alg: global_imbalance} instantiated with $\mathscr{A}$ attains with probability at least $1-\delta$:
    \[R^J_T\leq R^J_{\mathscr{A}} + 2I_\mathscr{A} , \quad |I_T|=0,\quad |I_t |\leq I_\mathscr{A} \quad \forall t\in[T].\]
\end{theorem}
\begin{proof}
The final surrogate regret bound follows from the fact that, up to round $T-I_{\mathscr{A}}$, the surrogate regret is bounded by $R^J_{\mathscr{A}}$, while over the remaining rounds it is at most $2I_{\mathscr{A}}$.

The per-round imbalance guarantee follows from the fact that, in the second phase of the algorithm, namely after round $T-I_{\mathscr{A}}$, the quantity $|I_t|$ strictly decreases at each round.

The satisfaction of the global inventory constraint follows by observing that, at round $T-I_{\mathscr{A}}$, the imbalance is at most $I_{\mathscr{A}}$, and then decreases strictly over the next $I_{\mathscr{A}}$ rounds until it reaches $0$, by construction of the algorithm.

Noticing that the above argument holds with probability at least $1-\delta$ concludes the proof.
\end{proof}

As a final remark, the assumption that the per-round imbalance upper bound $I_{\mathscr{A}}$ achieved by algorithm $\mathscr{A}$ is known in advance---notice that, this assumption does not hold for our instance-dependent bound---can be \emph{removed}. It is sufficient to start balancing the inventory as soon as $|I_t| = T-t$, and the same guarantee continues to hold.

\section{Final Discussion}

In this paper, we study the sequential market making problem through the lens of \emph{regret minimization}. Specifically, we formulate sequential market making as an online learning problem under peculiar inventory constraints. In this setting, the goal of the agent (the market maker) is two-fold. On the one hand, the agent iteratively posts bid and ask prices so as to \emph{maximize the revenue} collected during the learning process. On the other hand, the agent must handle the inventory by $i)$ guaranteeing that the number of sold items matches the number of purchased ones at the end of the learning process, and $ii)$ keeping the inventory as small as possible throughout learning.

We study the setting in which buyers' and sellers' valuations are sampled at each round from fixed distributions, without making any additional assumption on the shape of the corresponding c.d.f.'s. Moreover, we consider two feedback models, which determine what the market maker observes at the end of each round. Our results can be \emph{shortly} summarized as follows:
\begin{itemize}
\item We first study the well-known two-bit feedback model, in which the market maker only observes whether a trade with the seller and/or with the buyer occurred. In this setting, we develop an algorithm attaining $\widetilde{\mathcal{O}}(T^{\nicefrac{2}{3}})$ regret and $\widetilde{\mathcal{O}}(T^{\nicefrac{2}{3}})$ per-round inventory bounds, while satisfying the overall inventory constraint. 

\item We then turn to the LOB feedback model, where the agent has access to the \emph{limit order book}, and therefore observes the realized valuation of the buyer or seller whenever the corresponding trade does not occur. This feedback model is particularly relevant for several reasons. First, it is realistic in practice, since larger market players often have access to limit order book information. Second, as we show in this work, it allows one to achieve $\widetilde{\mathcal{O}}(\sqrt{T})$ regret and per-round inventory bounds---which we prove to be tight up to logarithmic factors---thus yielding significantly improved guarantees while still relying on a realistic feedback structure. Finally, our results show that in sequential market making, observing the complete per-round valuations at the end of the round is not necessary to get $\widetilde{\mathcal{O}}(\sqrt{T})$ guarantees.
\end{itemize}

Overall, our results show that sequential market making is a technically challenging online learning problem, where limited feedback and inventory constraints interact in a non-trivial way and require substantially new ideas.

\subsection{Open Problems}

An intriguing direction for future work is to study sequential market making under adversarial buyer and seller valuations, namely, no statistical assumption is made on the per-round valuations of buyers and sellers. To fix ideas, let us consider the natural offline benchmark in hindsight:
\begin{equation}\label{progr:offline_opt_adv}
	\text{OPT}_{\mathrm{Adv}} \coloneqq \begin{cases}
		\max_{b,a\in [0,1]^2}  & \sum_{t=1}^T\left[a\cdot \mathbbm{I}(\Vbuy_t\geq a) - b \cdot \mathbbm{I}(\Vsell_t\leq b) \right] \\[0.2cm]
        \text{s.t.}
		&\sum_{t=1}^T \mathbbm{I}(\Vbuy_t\ge a)=\sum_{t=1}^T\mathbbm{I}(\Vsell_t\le b)\\
		& b<a.
	\end{cases}
\end{equation}
This is the standard type of comparator considered in adversarial constrained online learning~\citep{Unifying_Framework}. At the same time, it is well known that, in bandit settings with adversarial rewards and adversarial constraints, one cannot in general hope for sublinear regret and sublinear violation with respect to a benchmark satisfying the constraints on average (namely, the constraints in-hindsight, as specified in $\text{OPT}_{\text{Adv}}$). The associated impossibility result is provided in Proposition 4 of~\citep{Mannor} for the simple case of two player zero-sum games with constraints. 

However, it is far from clear whether the existing impossibility results fully capture the adversarial sequential market making problem. Indeed, our setting has a very specific structure: the action is a bid-ask pair, the reward and the feasibility conditions are tightly coupled through the inventory, and the benchmark in Program~\eqref{progr:offline_opt_adv} is computed in terms of a single bid-ask pair rather than a distribution over actions (bid-ask), which instead is the case of~\citep{Mannor}. For this reason, a direct reduction from the classical impossibility instance is not immediate, and understanding the adversarial regime appears particularly interesting in our setting.

This leads to several natural open questions:
\begin{itemize}
    \item Is it possible to attain sublinear regret with respect to $\text{OPT}_{\mathrm{Adv}}$ while satisfying the inventory constraints?
    \item If the answer is negative, what is the right benchmark to study? A very conservative alternative is the optimum satisfying the inventory constraint at every round, namely
    \begin{equation*}\label{progr:offline_opt_adv_weak}
	\text{OPT}^{\mathrm{Weak}}_{\mathrm{Adv}} \coloneqq \begin{cases}
		\max_{b,a\in [0,1]^2}  & \sum_{t=1}^T\left[a\cdot \mathbbm{I}(\Vbuy_t\geq a) - b \cdot \mathbbm{I}(\Vsell_t\leq b) \right] \\[0.2cm]
        \text{s.t.}
		& \mathbbm{I}(\Vbuy_t\ge a)=\mathbbm{I}(\Vsell_t\le b)\qquad \forall t\in[T]\\
		& b<a,
	\end{cases}
\end{equation*}
    but this benchmark appears overly restrictive, and arguably not very natural, in market making settings.
    \item More broadly, are there other meaningful notions of regret or competitive guarantees that still retain $\text{OPT}_{\mathrm{Adv}}$ as the reference benchmark and are learnable under inventory constraints?
\end{itemize}
These questions naturally suggest several directions for future work.

\newpage
\appendix

\section*{AI Use Statement}
Generative AI tools (OpenAI Codex) were used exclusively to $(i)$ refine the writing, $(ii)$ generate the figures and $(iii)$ identify potential typos in the mathematical proofs.
\section{Notation Summary}

\begin{table}[H]
\centering
\label{tab:notation}
\small
\renewcommand{\arraystretch}{1.08}
\setlength{\tabcolsep}{5pt}
\begin{tabularx}{\textwidth}{@{}p{1.3in}X@{}}
\toprule
\textbf{Symbol} & \textbf{Meaning} \\
\midrule
$T$ & Time horizon. \\
$\Vbuy_t,\Vsell_t$ & Buyer’s and seller’s private valuations at time $t$. \\
$B_t,A_t$ & Bid--ask pair selected at time $t$. \\
$I_t$ & Inventory at time $t$. \\
$\alpha$ & Per-round inventory threshold. \\
$\mathrm{OPT}$ & Optimal reward value. \\
$R_T$ & Regret incurred by the algorithm. \\
$F(b)$ & Seller c.d.f. at $b$: $F(b)\coloneqq \E[\mathbbm{I}(\Vsell_t\le b)]$. \\
$S(a)$ & Buyer survival function at $a$: $S(a)\coloneqq \E[\mathbbm{I}(\Vbuy_t\ge a)]$. \\
$J(b,a)$ & Surrogate objective: $J(b,a)\coloneqq (a-b)\min\{F(b),S(a)\}$. \\
$J^*$ & Supremum of $J$. \\
$R_T^J$ & Surrogate regret, defined in \eqref{eq:surr_regret}. \\
$\mathscr{A}$ & An algorithm. \\
$\indi$ & Independent random variables. \\
$\Vs$ & Joint distribution of seller and buyer valuations. \\
$\Bs^\circ,\As^\circ$ & Discretized sets of bids and asks. \\
$\goodevent$ & Good event; see \eqref{eq:good_two} (2-bit case) and \eqref{eq:good} (LOB case). \\
$\ell$ & Phase index in Algorithm~\ref{alg:twobit}. \\
$\psi_\ell$ & Confidence bonus from Definition~\ref{def:bounds_two}. \\
$F_\ell^\mathrm{low}(b),F_\ell^\mathrm{up}(b)$ & Lower and upper confidence bounds for $F(b)$ at phase $\ell$ in Algorithm~\ref{alg:twobit}. \\
$S_\ell^\mathrm{low}(a),S_\ell^\mathrm{up}(a)$ & Lower and upper confidence bounds for $S(a)$ at phase $\ell$ in Algorithm~\ref{alg:twobit}. \\
$J_\ell^\mathrm{low}(b,a)$ & $(a-b)\min\{F_\ell^\mathrm{low}(b),S_\ell^\mathrm{low}(a)\}$ in Algorithm~\ref{alg:twobit}. \\
$J_\ell^\mathrm{up}(b,a)$ & $(a-b)\min\{F_\ell^\mathrm{up}(b),S_\ell^\mathrm{up}(a)\}$ in Algorithm~\ref{alg:twobit}. \\
$\Bs_\ell,\As_\ell$ & Sets of bids and asks active at phase $\ell$ in Algorithm~\ref{alg:twobit}. \\
$\Cs_\ell(b),\Cs_\ell(a)$ & Sets of asks and bids compatible with $b$ and $a$ at phase $\ell$ in Algorithm~\ref{alg:twobit}; see \eqref{eq:coset}. \\
$\gamma_t$ & Time complexity to perform one step of the algorithm.\\
$\Gamma$ & $\coloneqq\max_{t\in[T]}\gamma_t$.\\
$\Fs_t$ & Natural filtration of the process.\\
$\psi_t(b)$ & Confidence bonus of Equation~\eqref{eq:good}. \\
$F_t^\mathrm{low}(b),F_t^\mathrm{up}(b)$ & Lower and upper confidence bounds for $F(b)$ at step $t$ for the LOB case (Eq. \eqref{eq:boundsbid}). \\
$S_t^\mathrm{low}(a),S_t^\mathrm{up}(a)$ & Lower and upper confidence bounds for $S(a)$ at step $t$  for the LOB case (Eq. \eqref{eq:boundsbid}). \\
$J_t^\mathrm{low}(b,a)$ & $(a-b)\min\{F_t^\mathrm{low}(b),S_t^\mathrm{low}(a)\}$ in Algorithm~\ref{alg:ODPM}. \\
$J_t^\mathrm{up}(b,a)$ & $(a-b)\min\{F_t^\mathrm{up}(b),S_t^\mathrm{up}(a)\}$ in Algorithm~\ref{alg:ODPM}. \\
$\Bs_t,\As_t$ & Sets of bids and asks active at step $t$ in Algorithm~\ref{alg:ODPM}. \\
\bottomrule
\end{tabularx}
\end{table}

\section{Technical Lemmas}
\gooddeltatwo*
\begin{proof}
    By Theorem~\ref{thm:ecdf}, for fixed $t,\ell$ as in Definition~\ref{def:bounds_two} and $\varepsilon>0$, it holds:
    \[
    \Prob\left(\max\Big\{\sup_b\Delta_F(b,t,\ell),\sup_a\Delta_S(a,t,\ell)\Big\}>\varepsilon\right)
    \le 4e^{-2(2^\ell)\varepsilon^2}.
    \]
    For $S$, the theorem is applied to the reflected buyers' valuations, since $S(a)=\Prob(-\Vbuy_t\le-a)$. Taking $\varepsilon\uparrow\psi_\ell$ and employing a union bound on all the possible values of $t$ and $\ell$ in Definition~\ref{def:bounds_two}, we get:
    \[
    \Prob(\goodevent^c)
    \le\sum_{\ell=0}^{\lfloor\log_2(T^{2/3})\rfloor}\sum_{t=1}^{T-2^\ell+1}4e^{-2(2^\ell)\psi_\ell^2}
    \le T^2\frac{\delta^2}{4T^2}=\frac{\delta^2}{4}\le\delta,
    \]
    where the last inequalities hold as $\lfloor\log_2(T^{2/3})\rfloor+1\le T$ and $0<\delta<1$.
\end{proof}


\begin{lemma}\label{lem:double}
    Under the good event~\eqref{eq:good}, for every active $(b,a)$ pair, it holds at any time $t\in[T]:$
     \[
    F(b)-2\psi_t(b)\le F_t^\text{low}(b)\le F(b)
    \le F_t^\text{up}(b)\le F(b)+2\psi_t(b),
    \]
    and:
    \[
    S(a)-2\psi_t(a)\le S_t^\text{low}(a)\le S(a)
    \le S_t^\text{up}(a)\le S(a)+2\psi_t(a).
    \]
\end{lemma}
\begin{proof}
    By definition of the good event, it holds:
    \begin{align*}
        F_t^\text{up}(b)&\leq \widehat F_t(b)+\psi_t(b,\delta)\\
        &\le F(b)+|F(b)-\widehat F_t(b)|+\psi_t(b,\delta)\\
        &\le F(b)+2\psi_t(b,\delta).
    \end{align*}
    The other points follow in an analogous manner.
\end{proof}

\bibliographystyle{alpha}
\bibliography{example_paper}

\end{document}